\documentclass{article}
\usepackage[letterpaper,margin=1.0in]{geometry}
\usepackage{amsmath,amsfonts,bm}

\def\eqref#1{equation~\ref{#1}}
\def\1{\bm{1}}

\DeclareMathAlphabet{\mathsfit}{\encodingdefault}{\sfdefault}{m}{sl}
\SetMathAlphabet{\mathsfit}{bold}{\encodingdefault}{\sfdefault}{bx}{n}

\newcommand{\pdatacot}{p_{\rm{CoT}}}
\newcommand{\pdatanocot}{p_{\rm{NoCoT}}}
\newcommand{\R}{\mathbb{R}}

\newcommand{\softmax}{\mathrm{softmax}}

\newcommand{\bsequence}{\textbf{b}}
\newcommand{\transformer}{\mathcal{T}}
\newcommand{\embed}{\mathbf{E}}
\newcommand{\attention}{\mathcal{A}}
\newcommand{\FFN}{\mathrm{FFN}}
\newcommand{\qk}{A}
\newcommand{\weight}{W}
\newcommand{\head}{h}
\newcommand{\relu}{\mathrm{ReLU}}
\newcommand{\hiddenstate}{\mathrm{X}}
\newcommand{\loss}{\mathcal{L}}
\newcommand{\sgn}{\mathrm{sgn}}
\newcommand{\event}{\mathcal{E}}
\usepackage{hyperref}
\usepackage{enumitem}
\usepackage[utf8]{inputenc} % allow utf-8 input
\usepackage[T1]{fontenc}    % use 8-bit T1 fonts
\usepackage{hyperref}       % hyperlinks
\usepackage{url}            % simple URL typesetting
\usepackage{booktabs}       % professional-quality tables
\usepackage{amsfonts}       % blackboard math symbols
\usepackage{nicefrac}       % compact symbols for 1/2, etc.
\usepackage{microtype}      % microtypography
\usepackage{graphicx}       % graphics
\usepackage[numbers,sort]{natbib}
\usepackage{multirow}
\usepackage{bbding}
\graphicspath{{media/}}     % organize your images and

\usepackage{thmtools}
\usepackage{thm-restate}

\usepackage{hyperref}
\usepackage{amsmath}
\usepackage{amsthm}
\usepackage{amssymb}
\usepackage{url}
\usepackage{cleveref}
\usepackage{enumitem}
\newtheorem{definition}{Definition}
\newtheorem{theorem}{Theorem}
\newtheorem{assumption}{Assumption}
\newtheorem{lemma}{Lemma}
\usepackage{enumitem}

\usepackage{graphicx}
\usepackage{caption}  % 管理标题的宏包
\usepackage{subcaption}  % 管理子图的宏包
\usepackage{adjustbox}  % 调整图片大小和位置的宏包

\title{From Sparse Dependence to Sparse Attention: Unveiling \\How Chain-of-Thought Enhances Transformer Sample Efficiency}

\author{
  Kaiyue Wen\thanks{These authors contributed equally.}\ \thanks{A large part of this work was done while Kaiyue was at Tsinghua University.} \\
  Stanford University \\
  \texttt{kaiyuew@stanford.edu}
  \and
  Huaqing Zhang\textsuperscript{*} \\
  IIIS, Tsinghua University \\
  \texttt{zhanghq22@mails.tsinghua.edu.cn}
  \and
  Hongzhou Lin \thanks{This work is independent of and outside of the work at Amazon.} \\
  Amazon \\
  \\\texttt{hongzhou.lin89@gmail.com}
  \and
  Jingzhao Zhang \\
  IIIS, Tsinghua University \\
  Shanghai AI Lab \\
  Shanghai Qizhi Institute \\
  \texttt{jingzhaoz@mail.tsinghua.edu.cn}
}

\date{}  % 如果你不想显示日期，可以将其留空

\begin{document}
\maketitle
\begin{abstract}

Chain-of-thought (CoT)  significantly enhances the reasoning performance of large language models (LLM). While current theoretical studies often attribute this improvement to increased expressiveness and computational capacity, we argue that expressiveness is not the primary limitation in the LLM regime, as current large models will fail on simple tasks. Using a parity-learning setup, we demonstrate that CoT can substantially improve sample efficiency even when the representation power is sufficient. Specifically, with CoT, a transformer can learn the function within polynomial samples, whereas without CoT, the required sample size is exponential. Additionally, we show that CoT simplifies the learning process by introducing sparse sequential dependencies among input tokens, and leads to a sparse and interpretable attention. We validate our theoretical analysis with both synthetic and real-world experiments, confirming that sparsity in attention layers is a key factor of the improvement induced by CoT.
\footnote{Our code is available at \url{https://github.com/zhqwqwq/Learning-Parity-with-CoT}.}
% Chain-of-thought(CoT) can improve the performance of large language models. Existing theory studies suggest that the improvement results from that CoT enables stronger expressiveness of language models by allowing for more computation. 
% We consider a simple parity-learning setup and show that expressiveness may not be the bottleneck for language models. Our analysis proves that CoT can reduce the sample complexity of learning parity. Specifically, CoT makes the problem easier by enabling sparse sequential dependence in the input tokens.
% \jz{Summary}

\end{abstract}

\section{Introduction}
Chain-of-thought (CoT) has proven to be a powerful technique for enhancing reasoning in large language models \cite{wei2022chain, kojima2022large}.  By instructing the model to break complex problems into smaller, manageable steps, CoT facilitates more efficient reasoning and better generalization, particularly in algorithmic and logical tasks \cite{nye2022show, lewkowycz2022solving, wang2023selfconsistency}. 
Building on this, performance can be further improved through multi-step prompting and multi-path sampling techniques \cite{JMLR:v24:22-1144, wang2022rationale, zhou2023leasttomost, zhang2023automatic, fu2023complexitybased}.
% such as incorporating self-consistency checks \cite{JMLR:v24:22-1144, wang2022rationale}, increasing problem difficulty within the chain \cite{zhou2023leasttomost}, and employing complexity-based prompting \cite{zhang2023automatic, fu2023complexitybased}. 

This focus on CoT within in-context learning has since expanded to more structured learning approaches \cite{yao2024tree, besta2024graph}. By adding reasoning examples of CoT style to the instruction-tuning dataset, models enhance their problem-solving abilities more effectively than relying solely on CoT during prompting \cite{zelikman2022star, JMLR:v25:23-0870}. As a result, CoT is now shaping a new paradigm in language model development, marking a shift from simply scaling data \cite{kaplan2020scaling, hoffmann2022an} to focusing on advanced reasoning strategies \cite{lightman2024lets}, which leads to more effective learning outcomes.

% Chain-of-thought (CoT) has emerged as a powerful technique for improving reasoning in large language models \cite{wei2022chain, kojima2022large, nyeshow, lightman2024lets}. By breaking complex problems into smaller steps, CoT enables more efficient reasoning and better generalization, especially in algorithmic and logical tasks. This marks a shift from merely scaling data \cite{kaplan2020scaling, hoffmann2022an} to focusing on reasoning strategies like CoT, which is shaping a new paradigm in language model development \cite{wang2023selfconsistency, JMLR:v25:23-0870, yao2024tree}.

% In parallel, significant efforts have been made to improve in-context learning through CoT prompting \cite{JMLR:v24:22-1144, wang2023towards}. For instance, self-consistency methods \cite{JMLR:v24:22-1144} encourage models to generate multiple reasoning paths and select the most consistent one, improving accuracy on a wide range of tasks. Other strategies, such as breaking down complex tasks into simpler sub-problems, as demonstrated in Least-to-Most prompting \cite{zhou2023leasttomost}, also enhance performance. These techniques highlight that structuring the reasoning process, rather than solely increasing model size, can lead to better learning outcomes.

While CoT’s success is well-established, understanding why it works is still a hotly debated topic \cite{saparov2023language, prystawski2024think}. Recent theoretical studies suggest that CoT enhances a model’s expressiveness, increasing its representational capacity when the sequence is long enough \cite{feng2023towards, li2024chain}. However, expressivity alone does not guarantee success. Large language models often struggle with simple tasks—like counting the number of 'r's in "strawberry"—when not using CoT. Given the increasing model sizes, it seems unlikely that such tasks are inherently inexpressible. This discrepancy calls for a deeper study of generalization, hinting the true power of CoT may lie beyond the expressiveness.

% Classical universal approximation theory suggests that even a two-layer network can approximate any function \cite{cybenko1989approximation, HORNIK1989359}, yet in practice, deep networks are the ones that consistently excel. 

In this paper, we study the benefit of CoT from the sample efficiency perspective.  We provide concrete examples where expressiveness is not the limiting factor; that is, the function can be expressed both with and without CoT by the transformers. 
We demonstrate, both in theory and in practice, that without CoT, the learning requires exponentially more samples comparing to with CoT. 
% In other words, {{\bf CoT improves sample efficiency of transformers}}. 
{Further, we show that CoT sequences introduce sparse sequential dependence,  thereby enhancing the sparsity in the attention layers. We then show that transformers can efficiently optimize and generalize on such sequences}. We summarize our contributions as follows.

\begin{enumerate}[leftmargin=*]
\item %Theoretically, we show that while the parity problem can be expressed by a 1-layer transformer, learning requires exponential to $k$ samples without CoT when the number of parameters is limited (Theorem~\ref{thm:rep} and Theorem~\ref{thm:nocot}). Meanwhile, for CoT data with sparse sequential dependency, a 1-layer transformer can learn the parity function and faithfully represent the sparse dependency in its attention pattern with almost linear samples with respect to $k$ and $n$ (Theorem~\ref{thm:main}).
Theoretically, we show that while the parity problem can be expressed by a 1-layer transformer, learning requires exponentially many samples without CoT when the number of parameters is limited (Theorem~\ref{thm:rep} and ~\ref{thm:nocot}). Meanwhile, for CoT data with sparse sequential dependence, a 1-layer transformer can learn the parity function and faithfully represent the sparse dependence in its attention pattern with almost linear samples with respect to sequence length(Theorem~\ref{thm:main}).
\item Empirically, we verify our analysis that training on parity function with CoT data requires only polynomial samples and will induce sparse and interpretable attention (\Cref{fig: sample complexity with CoT and attention}). We further show that evaluating and training on CoT data will also induce sparser attention on real-world dataset GSM-8k on pretrained language models.
\end{enumerate}

% Without CoT, the model attempts to map inputs directly to outputs in one step, which can be a highly complex function to approximate, especially for tasks requiring reasoning or multi-step calculations. 
% By introducing CoT, the model learns to generate intermediate reasoning steps. Each step involves simpler functions that are easier for the model to learn and generalize from, compared to learning the entire complex function at once.
% Enhanced Generalization: Breaking down the reasoning process allows the model to apply learned patterns to new, unseen problems more effectively. It can reuse reasoning strategies across different tasks, improving performance.
% Error Reduction: Intermediate steps provide opportunities for error correction within the reasoning process. This reduces the likelihood of compounding errors that can occur when approximating the entire function in a single step.
% In summary, adding CoT transforms a complex function approximation task into a series of simpler tasks. This not only makes the learning process more manageable for the model but also leads to better performance on tasks that require reasoning and multi-step problem-solving.
% \vspace{-0.1in}
\section{Motivating Examples: Learning Parity Functions}

\begin{figure}[t]
% \vspace{-0.5in}
    \centering
    \includegraphics[width=\linewidth]{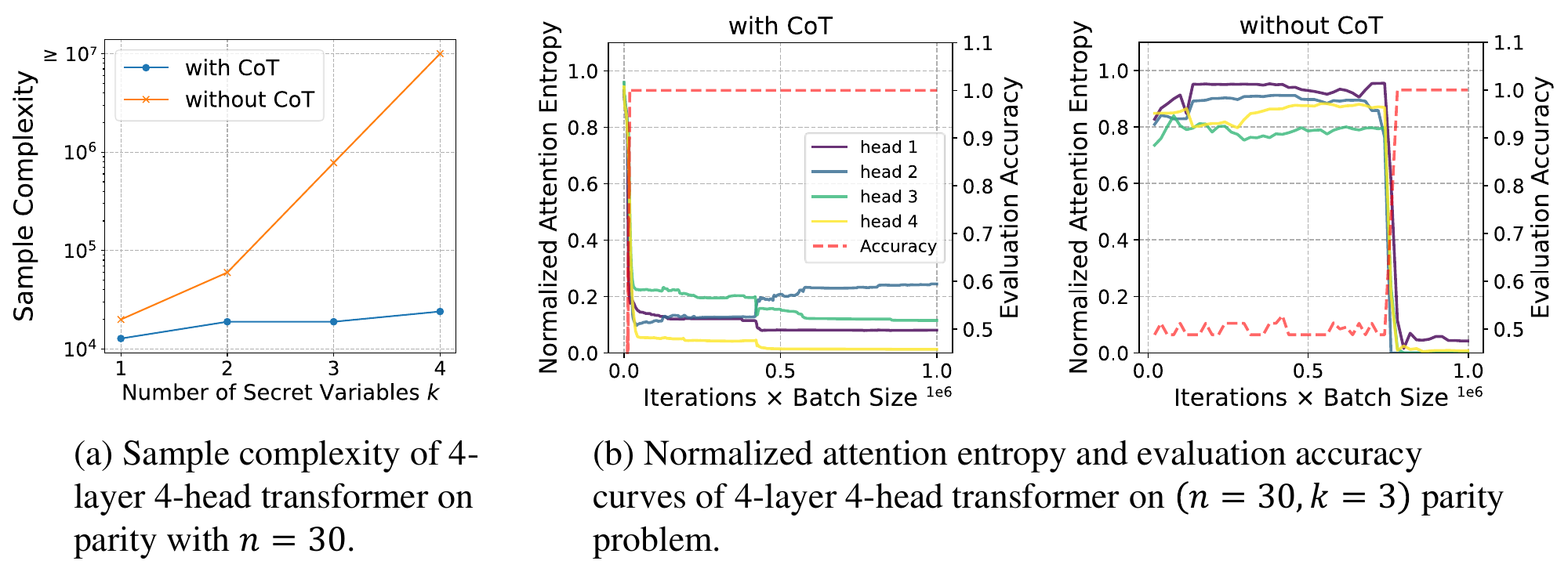}
    \caption{(a) We show that, without Chain-of-Thought (CoT), the sample complexity for training transformers to learn the parity function grows exponentially with the hardness parameter $k$. In contrast, utilizing CoT significantly improves sample efficiency.  (b) We also show that the sparsity of attention layers, measured by normalized entropy (\ref{eq: normalized attention entropy}), is crucial in the parity learning experiment. In both CoT and non-CoT scenarios, as the attention layers become sparser—indicated by a rapid decrease in normalized entropy—a corresponding jump in evaluation accuracy occurs. }
    % Here, the sample complexity is measured as the amount of data seen when the first time a model achieves perfect evaluation accuracy. 
    \label{fig: motivating example}
% \vspace{-0.2in}
\end{figure}

% \caption{(a) We show the sample complexity of training transformers (4 layers, 4 heads) to learn parity function with dimension $n = 30$. We observe that the sample complexity is growing exponentially without CoT when the hardness of the problem increases, from $k =1$ to $4$. In contrast, with CoT, the training is much more sample efficient.  (b) The normalized attention entropy (as defined in (\ref{eq: normalized attention entropy})) curve for the heads of the first layer when learning $(n=30, k=3)$ parity problem. Before the evaluation accuracy approaches $1$, the normalized attention entropy rapidly decreases to a low level.}

We start with empirically exploring the sample efficiency of training transformers to learn the class of parity functions \citep{blum2003} with and without CoT. The parity functions exhibits simple structure but yet hard to learn by traditional networks \cite{abbe2018provable, malach2022hardness}. Specifically, given a set of \(n\) binary variables \(b_1, b_2, \dots, b_n\), a parity function takes $k$ secret variables \( b_{i_1}, \dots b_{i_k} \) and outputs~$1$ if the sum of these $k$ variables is odd, and $0$ if it is even:
\[ 
f(b_1, ... , b_n) = b_{i_1} \oplus b_{i_2} \oplus \cdots \oplus b_{i_k},      
\]
where $b_i \in \{0,1\}$ and $\oplus$ is the XOR operator.
For example, $f(b_1, b_2, b_3, b_4, b_5) = b_1 \oplus b_{2} \oplus b_{4} $ is a $5$-variable parity function with $k=3$. The function \(f\) returns 1 if the sum of \(b_1\), \(b_{2}\) and \(b_{4}\) is odd, and $0$ otherwise, independently of the value of $b_3$ and $b_5$. Intuitively, the parameter $k$ controls the hardness of the problem. When \(k\) increases, the number of possible subsets grows exponentially, making the identification of the correct secret set more challenging. 

Given a $n$-variable parity function~$f$, we generate sequences of length $n+2$ as auto-regressive manner: $b_1, \cdots b_{n}, b_{n+1}:=0$ and $b_{n+2}=f(b_{[1:n]})$. To incorporate CoT, we break the sum of XOR down into $k$ steps and add all the intermediate steps into the sequence, i.e. including  $b_{i_1}$, $b_{i_1} \oplus b_{i_2}$, $\cdots$, $b_{i_1} \oplus b_{i_2} \oplus \cdots \oplus b_{i_k}$. As a result, the CoT data has length $n+k+1$. 
With the example function $f(b_1, b_2, b_3, b_4, b_5) = b_1 \oplus b_{2} \oplus b_{4} $, one sampled sequence would be  
\begin{align*}
\text{No CoT } \quad & \underbrace{0, 1, 0, 1, 0}_{\mathrm{input}}, \underbrace{0}_{\mathrm{[EOS]}}, \underbrace{0}_{\mathrm{answer}} \in \{0,1\}^7, \text{ as } b_1 \oplus b_{2} \oplus b_{4} = 0 \oplus 1 \oplus 1 = 0. \\
\text{With CoT } \quad & \underbrace{0, 1, 0, 1, 0}_{\mathrm{input}}, \underbrace{0}_{\mathrm{[EOS]}} \underbrace{0, 1}_{\mathrm{CoT}}, \underbrace{0}_{\mathrm{answer}} \in \{0,1\}^{9}, \text{as } b_1 = 0, b_1 \oplus b_{2} = 1. 
\end{align*}
% \[ \underbrace{0, 1, 0, 1, 0}_{query}, 0 \in \{0,1\}^6, \text{ as } f(0,1,0,1,0) = 0 \oplus 1 \oplus 1 = 0. \] 
% , where the first $n$ tokens are independent Rademacher variables and the last token is the corresponding value of $f$. Since the last token depends on $k$-of the previous tokens, the generated sequence has an intrinsic sequential dependency of degree~$k$. 

% Under the same example, the sequence \emph{with CoT} would look like 
% \[ \underbrace{0, 1, 0, 1, 0}_{query}, \underbrace{0, 1}_{CoT}, 0 \in \{0,1\}^{8} \]  where the last three tokens are computed by $b_1$, $b_1 \oplus b_{2}$ and $b_1 \oplus b_{2} \oplus b_{4}$. 

% With the introduction of CoT, we incrementally inject the hidden steps, adding one at a time. In this case, the last three digits has sparser sequential dependency: it depends on the previous token and one hidden token, respectively $b_1, b_{2}, b_{4}$. Hence we say the sequential dependency of the CoT sequence has degree $1$. \emph{In this sense, we say CoT enables sparse sequential dependence. }

Next, we train transformer networks with and without CoT respectively, with a common held-out test set, and compare their sample complexities. The sample complexity is defined as the amount of data a model sees at the first time it achieves perfect validation accuracy. More precisely, we halt the training process when the model reaches 100\% validation accuracy and record the total number of training examples used up to this point as the empirical sample complexity.
% at training time, the model is presented with batches of training data and optimized iteratively. 

% A natural question arises: 
    % {\it How does the sample complexity scale with and without CoT, as the parity function becomes more complex?}

% \jz{Transition}
In Figure~\ref{fig: motivating example}, we compare the empirical results of learning parity function with and without CoT. The results show that the sample complexity without CoT (orange curve) grows exponentially, while with CoT (blue curve), it increases linearly. In other words, incorporating CoT allows us to learn the task with exponentially fewer samples. Additionally, it suggests that sparsity in the attention layers plays a crucial role in the learning process. To explore why CoT enables transformers to learn more efficiently, we formally define our problem setup below.

% This example also shows that expressiveness is not the bottleneck for learning parities.

% \vspace{-0.1in}
\section{Theoretical Analysis} \label{sec: theory}
% \vspace{-0.1in}
In this section, we provide a formal analysis of the training dynamics of Transformers on the parity problem, both with and without Chain of Thought (CoT). We select the parity problem as our testbed because identifying a set of key variables amidst various confounding ones is a fundamental aspect of many reasoning tasks, and parity serves as an abstraction of this process.

% \vspace{-0.1in}
\subsection{Notations and Definitions}

% \textbf{Notation.} We use $\oplus$ to indicate the XOR operator on binary variables. 
% We use $e_i$ to denote the $i$-th one-hot vector and specify the dimension in the context. We will use $[j]$ to denote variables corresponding to the $j-$th token dimension without otherwise specified. We will use $p(x, y)$ to indicates $2(x - 1) + y$. 

\textbf{Parity Problem.} Each token in our setting is either $0$ or $1$. We represent a sequence of length $T$ as $\bsequence[1], \ldots, \bsequence[T]$. A parity function is a function of form $\mathrm{parity}_S (\bsequence) = \oplus_{j \in \mathcal{S}} \bsequence[j]$, where $\mathcal{S}$ is the set of secret indices that is fixed during training and testing, and $\oplus$ is the XOR operator on binary variables. The cardinal of the secret set $k = |S|$  controls the hardness of the problem.

% \textbf{Data Distribution.} We will study language modeling on the well-known \emph{parity} problem. In our setting, every token takes the value of 0 or 1 and we will denote a sequence of binary variables with length $T$ as $\bsequence[1], \ldots, \bsequence[T]$. The parity problem requires the model to calculate the the sum of bits on a set of secret indices. We will denote the set of secret indices as $\mathcal{S}$. 

% For standard sequences without CoT, the parity of the sum of bits on the secret indices is appended at the end of the sequence directly.
\begin{definition}[Parity Problem $(n,k)$ without CoT]
\label{def:nocot}
Given a secret set $\mathcal{S}$ of cardinal $k$, we define the parity problem without CoT as the following distribution of sequences $\pdatanocot$:
    \begin{align*} 
        \bsequence[i] \sim U(\{0,1\}), \forall i \in [1,n],
        \bsequence[n + 1] = 0, 
        \bsequence[n + 2] = \mathrm{parity}_S (\bsequence).
    \end{align*}
where $\bsequence[1], ..., \bsequence[n]$ are uniformly sampled from $0$ and $1$.  
\end{definition}

\begin{definition}[Parity Problem $(n,k)$  with CoT]
\label{def:cot} Given a secret set $\mathcal{S}$ of cardinal $k$, 
    we define the parity problem with CoT as the following distribution of augmented sequences $\pdatacot$:
    \begin{align*} 
        \bsequence[i] \sim U(\{0,1\}), \forall i \in [1,n],
        \bsequence[n + 1] = 0, 
        % \bsequence[n + 1+ i] = \oplus_{j=1}^i \bsequence[S[j]], \forall i \in \{ 1 , \ldots,  k \}.
        \bsequence[i + 1] = \bsequence[i] \oplus \bsequence[S[i]], \forall i \in \{ n + 1 , \ldots,  n + k \}.
    \end{align*}
    where $S$ is any permutation of $\mathcal{S}$, i.e. $\{S[i] \mid i = n + 1 , \ldots, n + k \} = \mathcal{S}$. 
\end{definition}
When CoT is provided, the data includes a step-by-step computation of the desired parity function, adding one variable at a time. Note that CoT is not unique for a given secret set since we can arbitrarily permute $\mathcal{S}$, given the commutative property of the XOR operator.

\textbf{Transformer Architecture.} To conduct the theoretical analysis, we simplify the Transformer architecture similar to prior works (see e.g. \cite{wang2024transformersprovablylearnsparse,nichani2024transformerslearncausalstructure,pmlr-v202-li23p}). More precisely, we drop all the layer norms or batch normalization; simplify the positional embedding; use a square matrix to represent the attention layer and concatenate the residual branch in a Densenet fashion~\cite{huang2017densely}. After simplification, the network still has strong expressive power like the standard Transformers.

% The major deference is that the  attention is directly parameterized by a square matrix, and that we concatenate, instead of adding, the residual branch. We first define the embedding for a binary input sequence.

In a standard Transformer, an embedding layer matches each vocabulary token into a dense vector of size $d$ and then adds a positional embedding. In our case, we only have boolean tokens and the sequence has constant length~$T$. Hence, we simply freeze $2T$ unit vectors as the embedding vectors: 
\begin{definition}[Embedding Module]
\label{def:embed}
For any position $i \in [T]$, we sample two embedding vectors $e_{i,0}$ and $e_{i,1}$ uniformly at random from unit hypercube $\mathrm{U}\left(\{\frac{1}{\sqrt{d}}, -\frac{1}{\sqrt{d}}\}^d \right)$. These embedding vectors are frozen during training.
% a set of near orthogonal vectors with unit norm.
% \begin{align*}
%     e_{p(j, l)} \sim \mathrm{U}\left(\{\frac{1}{\sqrt{d}}, -\frac{1}{\sqrt{d}}\}^d \right).
% \end{align*}
Then for any binary sequence $\bsequence \in \{0, 1\}^T$, the embedding is   defined as 
    \begin{align*}
        \embed \left( \bsequence \right) =[e_{(i, \bsequence[i])}]_{i \in [T]} \in \R^{d \times T}.
    \end{align*}
Due to the properties of the hypercube, with high probability, the embedding vectors are near-orthogonal with each other.
\end{definition}

% \begin{definition}[Embedding Module]
% \label{def:embed}
% For any position $i \in [T]$, $l \in \{0, 1\}$, we will sample vector $e_{p(j, l)}$ is sampled uniformly at random from a set of near orthogonal vectors with unit norm.
% \begin{align*}
%     e_{p(j, l)} \sim \mathrm{U}\left(\{\frac{1}{\sqrt{d}}, -\frac{1}{\sqrt{d}}\}^d \right).
% \end{align*}
% The embedding function over a binary sequence $\bsequence \in \{0, 1\}^T$ is then defined as,
%     \begin{align*}
%         \embed \left( \bsequence \right) =[e_{p(i, \bsequence[i])}]_{i \in [T]} \in \R^{d \times T}.
%     \end{align*}
% This embedding is then fixed during training.
% \end{definition}

Next, we define the attention layer. In standard architecture, the attention layer is derived from the product of query and key matrices, i.e. $QK^T = \hiddenstate^T W^Q (W^K)^T \hiddenstate$. In our case, we directly use a full matrix $\hiddenstate^T A \hiddenstate$ to parameterize it, which has the same representation power. 
\begin{definition}[Attention Module]
\label{def:attn}
Given an input matrix $\hiddenstate \in \R^{d \times T}$ and attention weight $A$, we define the attention module as 
    \begin{align*}
        &\attention(\hiddenstate) = \hiddenstate \  \softmax(C + \hiddenstate^T\qk \hiddenstate),
    \end{align*}
    where $C \in \R^{T \times T}$ is the autoregressive mask with value $-\infty$ on the lower triangular matrix.     
    % value $0$ on the upper right triangle (including the diagonal line) and $-\infty$ at other values. The $\softmax$ is applied column-wise. 
\end{definition}

% \begin{definition}[Attention Module]
% \label{def:attn}
%     We define the attention function over a matrix $\hiddenstate \in \R^{d \times T}$ as 
%     \begin{align*}
%         &\attention(\hiddenstate) = \hiddenstate \  \softmax(C + \hiddenstate^T\qk \hiddenstate).
%     \end{align*}
%     Here $C \in \R^{T \times T}$
%     is the autoregressive mask with value $0$ on the upper right triangle (including the diagonal line) and $-\infty$ at other values.The $\softmax$ is applied column-wise. 
% \end{definition}

% Compared with standard attention, the attention module defined in~\Cref{def:attn} fixes the value matrix as identity and merges the $Q$ and $K$ matrices as there is only one head here.

The attention module is then followed by a fully connected layer with ReLU activation:
\begin{definition}
\label{def:ffn}  
 We define the FFN function with width $2m$ over a matrix $\hiddenstate \in \R^{2d \times T}$ as 
 % \vspace{-1mm}
 \begin{align*}
     \FFN(\hiddenstate) = \head^T \relu( \weight X).
 \end{align*}
 % \vspace{-1mm}
Here $\weight \in \R^{2m \times 2d}, \head \in \R^{2m \times o}$ with $o$ being the output dimension and $\relu$ is element-wise. 
\end{definition}

Finally, we concatenate the modules into a simplified Transformer block.
\begin{definition}[Simplified Transformer Block]
\label{def:trans}
We define a simplified Transformer block as,
% \vspace{-1mm}
\begin{align*}
    \transformer(\bsequence) = \FFN\left(\begin{bmatrix} 
    \embed\left( \bsequence \right) \\
    \attention\left(\embed\left( \bsequence \right) \right) \end{bmatrix} \right),
\end{align*}
% \vspace{-1mm}
where $\embed$ is the embedding module, $\attention$ is the attention module and $\FFN$ the fully connected layer. An $L$-layer Transformer consists of a composition of $L$ such blocks, with the embedding module $\embed$ appearing only in the initial layer. The intermediate dimension is set to $d$, while the output dimension of the final layer is 1. 
\end{definition}

Compared with the standard residual structure, we use a Densenet structure to concatenate the residual branch with identity branch. Again, this transformation does not affect the representation power of the network, which is the standard practice in previous theoretical analysis (see e.g. \citet{wang2024transformersprovablylearnsparse,nichani2024transformerslearncausalstructure}).

\textbf{Loss function.} To simplify our analysis, we use hinge loss $\ell(\hat y, y) = \max\{ (-1)^y \hat y + 1, 0\}$ as the loss function. We define the next token prediction loss of a boolean sequence $\bsequence$ as  $ L(w) =   \ell( \transformer^{(L)}(\bsequence)[n+1], \bsequence[n + 2])$ and
$ L(w) =  \sum_{i = n + 1}^{n + k} \ell( \transformer^{(L)}(\bsequence)[i], \bsequence[i + 1])$ for with and without CoT setup respectively,
where $w$ denotes all the trainable parameters.% $[A, W]$ in the attention module and the FFN map across all layers. 
% Further, $T[w]$ indicates the parameterization defined in this section. In the following text, the notation $[w]$ is omitted when its meaning is clear from the context.

\subsection{Exponential Sample Complexity without CoT}

We now present analysis in the no CoT setup. First, we show that the parity problem is easy to represent with the simplified Transformer architecture:

% \begin{theorem}[Easy to Represent]
% \label{thm:rep}
%      Consider the Transformer model defined in~\Cref{def:trans}, for any $\delta < 0.1$ and large enough $n$, with probability at least $1 - \delta$ over the randomness of embedding $e$, there exists a weight configuration of the Transformer with dimension $d = \tilde \Omega(k)$ and width $2m = \Omega(k)$, such that it can achieve perfect accuracy on the parity problem $(n,k)$ without CoT, i.e.
%     \begin{align*}
%         \forall \bsequence \sim \pdatanocot, \sgn\left(\transformer(\bsequence)[n + 1] \right) = \bsequence[n + 2].
%     \end{align*}
%     % Further, when the number of secret indices $k$ is in $[n / \log^5 (n / \delta), n / \log^4(n/\delta)]$, there exists a weight configuration that can be stored in $o(nk)$ memory. \kaiyue{Change $d$ depends on $k$.}
% \end{theorem}

\begin{theorem}[Easy to Represent]
\label{thm:rep}
     % Consider the Transformer model defined in~\Cref{def:trans}, for any $\delta < 0.1$ and large enough $n$, with probability at least $1 - \delta$ over the randomness of embedding $e$, there exists a weight configuration of the Transformer with dimension $d = \tilde \Omega(k)$ and width $2m = \Omega(k)$, such that it can achieve perfect accuracy on the parity problem $(n,k)$ without CoT, i.e.
     Consider the Transformer model defined in~\Cref{def:trans}, for any $\delta<0.1$ and large enough $n$, when the number of secret indices $k$ is in $[n / \log^5 (n / \delta), n / \log^4(n/\delta)]$, with probability at least $1 - \delta$ over the randomness of embedding $e$, there exists a weight configuration of the Transformer with dimension $d = \Theta(k\log (n/\delta))$ and width $2m =O(k)$ with $\Theta(\log n)$ precision of the weights and activations, such that $d^2+dm=o(nk/\log n)$, and it achieves perfect accuracy on the parity problem $(n,k)$ without CoT, i.e.
     % \vspace{-2mm}
    \begin{align*}
        \forall \bsequence \sim \pdatanocot, \sgn\left(\transformer(\bsequence)[n + 1] \right) = (-1)^{\bsequence[n + 2]+1}.
    \end{align*}
    % \vspace{-2mm}
    % Further, when the number of secret indices $k$ is in $[n / \log^5 (n / \delta), n / \log^4(n/\delta)]$, there exists a weight configuration that can be stored in $o(nk)$ memory. \kaiyue{Change $d$ depends on $k$.}
\end{theorem}

In other words, the model possesses sufficient expressive power to represent any parity function, even in the absence of the chain of thought. Therefore, we are in the \emph{representational-sufficient} regime where expressiveness is not the bottleneck. The proof of this statement is based on random matrix theory and concentration inequalities, which we defer to~\Cref{app:thm:rep}. 
% This theorem shows that the problem we are analyzing falls in the \emph{representational-sufficient} regime, where the model has the representational power to represent solution to the reasoning problem, even without the presence of the chain of thought.

While the function is expressible, this does not guarantee that the solution can be easily found. In fact, we show that achieving a perfectly accurate solution using a gradient-based optimization method requires an exponential number of samples in $k$ assuming the memory, i.e. space complexity, to perform optimization is bounded throughout the training.

\begin{theorem}[Hard to Learn]
\label{thm:nocot}
For any randomly initialized simplified Transformer model with constant layers, when the embedding dimension $d$ and width $2m$ satisfies $d^2 + dm = o(nk / \log(n))$, for any constant number of passes $q$, when the model is trained with $q-$pass stochastic gradient descent with $O((d^2 + dm)\log n)$ memory, the sample complexity required to learn the parity problem $(n,k)$ without CoT to any nontrivial accuracy $a > 50\%$ with nontrivial probability $p > 50\%$ is $2^{\Omega(k)}$.
\end{theorem}

The proof is deferred to~\Cref{app:thm:nocot}. We use results from the classical online learning communities~\citep{lyu2023tighttimespacelowerbounds} to show that in the regime where the memory required to perform the training is less than $nk$, the model can't effectively store information about enough samples inside the parameters during training, and hence can't infer the secret indices effectively. Here the parameters $k$ denotes the size of the secret set $\mathcal{S}$ and the sample complexity grows exponentially when $k$ grows.

\subsection{Polynomial Sample Complexity with CoT}

The result in~\Cref{thm:nocot} presents a seemingly negative outcome for learning algorithmic reasoning, as the model requires exponentially many samples relative to $k$. However, in our main result,~\Cref{thm:main}, we will demonstrate that the sample complexity can be significantly reduced if a step-by-step derivation is provided. In other words, CoT is much more sample efficient, where model can effectively learn complex relationships as long as each token depends on only a few previous tokens.

% The above~\Cref{thm:nocot} seems a negative message for learning algorithmic reasoning because the model fails to learn the key variables unless exponentially many samples with respect to $k$ are provided. However, we will show that the sample complexity can be drastically reduced if a step-by-step derivation is provided in our main~\Cref{thm:main}. This suggests that the model can learn complex relationships effectively as long as each token only depends on a few previous tokens. 

We will use the following initialization of the 1-layer Transformer model. The attention is initialized to be uniform ($A = 0$) and the contribution of attention output is initialized to be zero ($W_{r, d+1:2d} = 0$). The parameter $\head$ in the FFN and the word embedding $\embed$ is fixed during training.

\begin{assumption}[Initialization]
\label{assum:init}
At initialization, 
\begin{align*}
    \forall r \in [2m], A &= 0, W_{r, d + 1: 2d} = 0, 
    W_{r, 1:d} = \sum_{i = n + 1}^{n + k} \sum_{b = 0}^1 \nu_{r, i, b} e_{i, b}, h_{1:m} = \frac{1}{2m}, h_{m+1:2m} = -\frac{1}{2m}.
\end{align*}
Here $\nu_{r, i, b}$ is independent random variable sampled uniformly from $\{-\epsilon, \epsilon\}$. 
\end{assumption}

We can then train the model with stochastic gradient descent (SGD) and get the following theorem.

\begin{theorem}[Easy to Learn with CoT]
    \label{thm:main}
    For any constant $\delta \in (0,1)$,  when the number of secret indices $k$ is in $[n / \log^5 (n / \delta), n / \log^4(n/\delta)]$, with probability $1 - \delta$, a randomly initialized simplified Transformer (see~\Cref{assum:init}) with $o(nk)$ parameters trained for constant steps using mini-batch SGD with $\Tilde O(n)$ samples using appropriate hyperparameters (see~\Cref{assum:order}) can reach perfect accuracy on a parity problem $(n,k)$ with CoT,
    \begin{align*}
        \forall \bsequence \sim \pdatacot, i \in \{n + 1, \ldots, n + k\},\sgn\left(\transformer(\bsequence)[i] \right) = \bsequence[i + 1].
    \end{align*}
    Furthermore, after training, the attention pattern is interpretable and one-hot in the sense that, for any $\bsequence \sim \pdatacot, i \in \{n + 1,\cdots,  n + k\}, j\leq i$,
    \begin{align*}
        \big| \softmax(C +  \embed(\bsequence)A\embed(\bsequence) )[j, i] - \1(j = S[i])  \big| < 1/n^8.
    \end{align*}
    Moreover, the result still holds even if all the weights and activations are  in $\Theta(\log n)$ precision.
\end{theorem}

This theorem indicates that the attention module can successfully extract the sparse sequential dependencies in the CoT data and faithfully represent it in the attention pattern, which is validated in our experiments (see \Cref{fig: sample complexity with CoT and attention}). It is also amongst the first optimization dynamics analysis of Transformers using finite-sample gradients rather than population gradients. We delay the full proof of the theorem to~\Cref{sec:proofofthm3}.

\emph{Proof Sketch.} In our analysis, the dynamics of the model includes three key phases. In the first phase, the weight of the FFN layers become correlated with the embeddings associated to the secret indices. In the second phase, this correlation caused the attention module to receive a strong signal to amplify the attention weight on the corresponding secret indices at each position. The attention pattern will become one-hot after this step. Finally, in the third phase, the FFN layers learn the correct mapping to the output, utilizing both the embedding at the current token and the retrieved embedding from the secret indices as indicated by the attention.

\textbf{Phase 1. Configuring the FFN} At initialization, as the FFN weight corresponds to the attention output ($W_{r, d+1:2d}$) is initialized to be zero. The set of neurons activated at each position $i \in [n+1, n+k]$ is solely determined by the word embedding $e_{(i,\bsequence[i])}$ at the position. Because $e_{i, b}$ are nearly orthogonal, it holds that $\langle e_{i,\mathbf b[i]},W_{r,1:d}\rangle \approx \nu_{r,i,b}$ for all $i$ and $r$ with high probability. As a result, for a fixed $i$ and $\mathbf{b}[i]$, the set of activated neurons at position $i$ is $\{r \mid \nu_{r, i, \bsequence[i]} > 0 \}$. Therefore, at initialization, the MLP can be viewed as an ensemble of multiple linear functions specialized to each position and boolean value. We can show that when the learning rate is small, the set of activated neurons at each position remains the same through the training process. Hence, we can conceptually view the FFN function as a set of different linear functions applied independently at each position and value. We will denote this set of linear weight as $\kappa_{t,i,b}$ (formally defined in~\Cref{lem:attnsignal}).

As the attention weight matrix $\qk$ is initialized as zero, the attention pattern will be uniform at initialization. Notice that in the parity data with CoT, the only position whose value correlates with $\bsequence[i + 1]$ when conditioned on $\bsequence[i]$ is $S[i]$. This suggests that the linear weight $\kappa_{1,i,\bsequence[i]}$ will have a stronger correlation with the embedding $e_{(S[i], \bsequence[j])}$ than other embeddings. 

This step crucially relies on the sparse dependency in the CoT data. The strong linear correlation will not be present in the data without CoT when $k > 1$, as every token $\bsequence[i]$ for $i \in [n+1]$ will be uncorrelated with the desired output $\bsequence[n + 2]$ in such case.

\textbf{Phase 2. Learning the Sparse Attention} 
At the second step, as  the FFN weight corresponds to the attention output is no longer zero, the attention weight matrix $\qk$ will receive a non-zero gradient. By the chain rule, the gradient corresponding to how the attention from $i-$th token attends to the $j-$th token  $\softmax(C + \hiddenstate^T\qk \hiddenstate)[j,i]$ will be larger when the approximate contribution of the embedding of the $j-$th token to the output at the $i-$th token $\langle \kappa_{1,i,\bsequence[i]}, e_{(j, \bsequence[j])} \rangle$ is larger. This suggests that the attention will be amplified on the index that has a strong correlation with the output at the current token, which is the secret index $S[i]$. This step will make the attention pattern approximately one-hot as in the theorem statement.

\textbf{Phase 3. Learning the Output Mapping} 
At the final step, as the attention pattern is one-hot, the FFN layer only needs to learn a mapping from the embedding at the current token and the embedding at the secret index to the output. As the mapping is linear conditioned on the embedding at the current token, the FFN layer can learn the correct mapping within a single step. This step will make the model reach perfect accuracy on the parity problem with CoT.

% \vspace{-0.1in}

\begin{figure}[t]
% \vspace{-0.5in}
    \centering
    \includegraphics[width=0.95\linewidth]{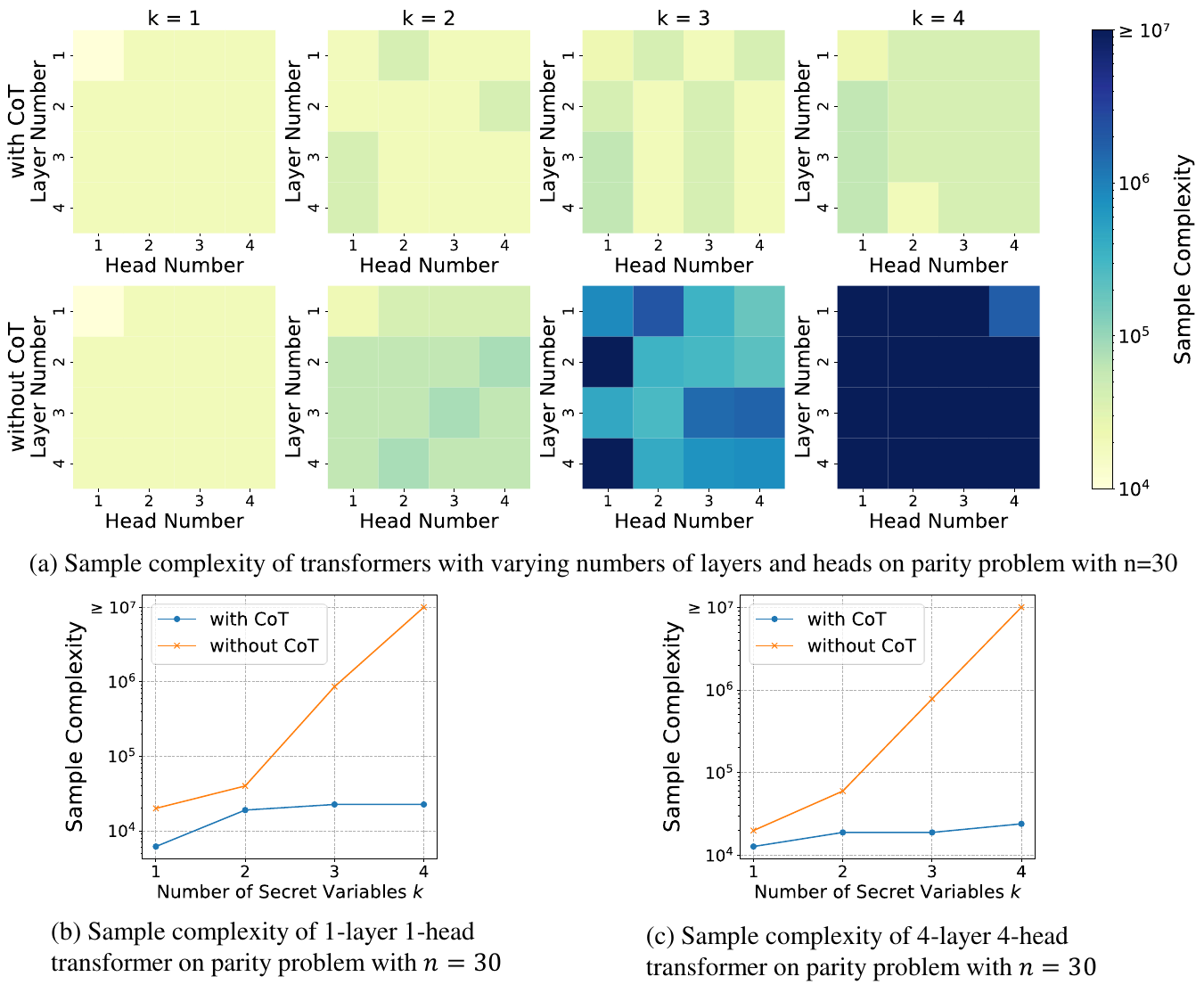}
    \caption{The sample complexity for learning parity without CoT increases exponentially with $k$. CoT significantly reduces the sample complexity, demonstrating exponential improvement across varying numbers of heads and layers.}
    \label{fig: sample complexity}
% \vspace{-0.3in}
\end{figure}

\begin{figure}[t]
% \vspace{-0.5in}
    \centering
\includegraphics[width = \linewidth]{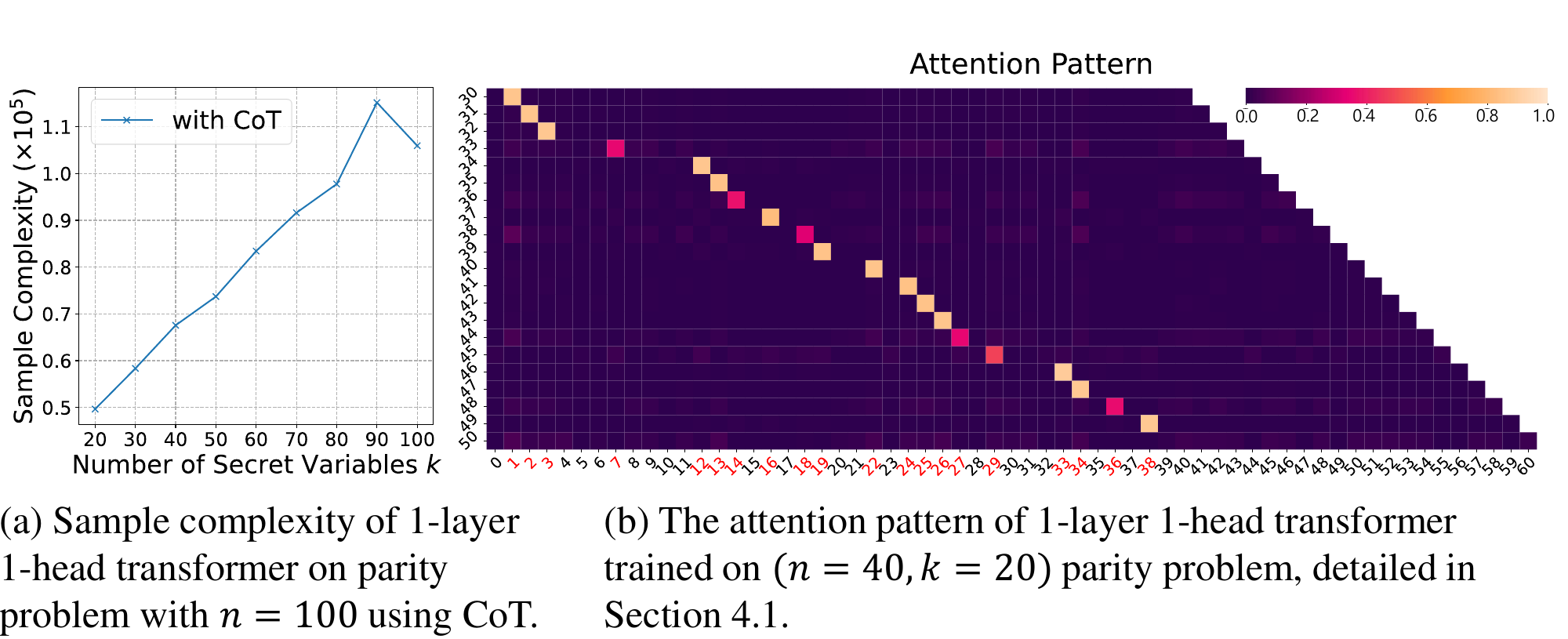}
    \caption{(a) When $n$ is fixed, the sample complexity of learning parity with CoT grows approximately linearly with $k$. (b) The attention pattern learned by the transformer with CoT is \textbf{interpretable}, as the $i$-th output token of CoT predominantly attends to secret index $S[i]$.}
    % , while its attention to other tokens remains minimal.
    \label{fig: sample complexity with CoT and attention}
% \vspace{-0.25in}
\end{figure}

\section{Empirical Experiments}\label{sec:exp}

In this section, we validate and extend our theoretical findings through comprehensive experiments in the following three aspects: First, we empirically confirm that CoT reduces the sample complexity of transformers in learning the parity problem. 
%with varying depths and number of heads. 
%Second, to complement the established lower bound of sample complexity for learning parity without CoT with one-pass training, we conduct multi-pass training of transformers without CoT. The results indicate that multi-pass training indeed improves the models' ability to learn parity problems. 
Second, we conduct multi-pass training of transformers without CoT, as a complement to the established lower bound (Theorem \ref{thm:nocot}) which only applies to constant-pass training. The results indicate that multi-pass training indeed improves the models' ability to learn parity problems. 
Third, we conduct experiments on the GSM8K dataset \citep{cobbe2021training} to show that CoT introduces sparse sequential dependence on real-world training data.
Although necessary simplifications on transformer are made in Section \ref{sec: theory} to develop theoretical results, we use the standard transformer architecture in the subsequent experiments following the GPT-2 architecture \citep{radford2019language} with trainable position embeddings unless otherwise specified.
% \vspace{-0.1in}

\subsection {Parity Learning with Multi-layer Transformers}\label{sec: exp multi-layer}
% \vspace{-0.1in}

In this section, we conduct an ablation study on the sample complexity of transformers with standard GPT-2 architectures and one-pass training,  with and without CoT.

% \begin{figure}
%     \centering
%     \includegraphics[width=\linewidth]{Figures/fig3_onepass.pdf}
%     \caption{The attention pattern of a $1$-layer $1$-head transformer trained on parity problem with $n=40$ and $k=20$ on a random input. The $i$-th row represents the amount that the $i$-th token attends to previous tokens. All indices $j$ for which $\mathbf s_j=1$ are marked in red. The attentions of the input tokens are not presented.\huaqing{TODO: check the notation.}%\huaqing{Should we include the conclusions in the titles, or we just discuss in the result part?}
%     }
%     \label{fig: attention pattern}
% \end{figure}

% \huaqing{Maybe we should mention that $S[i]$ here follows ascending order.}\jz{Recall $S[i]$ def when you first use it .}

\textbf{Ablation study on sample complexity~(Figure \ref{fig: sample complexity}).}
% In Figure \ref{fig: sample complexity} and \ref{fig: sample complexity with CoT and attention}a, 
%  
% In Figure \ref{fig: sample complexity}, 
We train transformers on parity problem with $n=30$ and $k=1,2,3,4$, with and without CoT, varying layers and heads from $1$ to $4$.  We choose the best performer across learning rates from $6 \times 10^{-5}$, $8 \times 10^{-5}$, and $1 \times 10^{-4}$.
At each step, a fresh batch of training data is sampled, with a maximum budget of $10^7$ samples. We record the number of samples seen by the model before reaching an evaluation accuracy of $1$ as the sample complexity.

% Recall that sample complexity refers to the minimum number of
% samples seen by the model before reaching an evaluation accuracy of $1$.  
% In Figure \ref{fig: sample complexity}, we train transformers on the parity problem with $n=30$ and $k=1,2,3,4$, with and without CoT, varying layers and heads from $1$ to $4$.  We choose the best learning rates from $6 \times 10^{-5}$, $8 \times 10^{-5}$, and $1 \times 10^{-4}$.
% At each step, a fresh batch of training data is sampled, with a maximum budget of $10^7$ samples. Figure \ref{fig: sample complexity with CoT and attention}a shows the sample complexity of a $1$-layer $1$-head transformer on larger parity problems ($n=100$ and $20\leq k \leq 100$). Training without CoT exceeds the sample budget, so results are not shown. 
% Figure \ref{fig: sample complexity with CoT and attention}b depicts the attention pattern of a $1$-layer $1$-head transformer trained on $(n=40, k=20)$ parity problem with CoT on a random query. The CoT here processes secret variables in an ascending order. The $i$-th row represents the amount that the $i$-th token attends to previous tokens. Indices corresponding to secret variables are marked in red. The attention of the query tokens is omitted.

\textbf{Results.} Figure \ref{fig: sample complexity} shows that training with CoT [first row in Figure \ref{fig: sample complexity} (a)] consistently achieves better sample efficiency than training without CoT [second row in Figure \ref{fig: sample complexity} (a)]. Moreover, for a fixed configuration of heads and layers, the sample complexity without CoT grows exponentially with the parameter $k$. In contrast, CoT greatly reduces the sample complexity, showing an exponential improvement across different numbers of heads and layers. These findings are consistent with our theoretical analysis.

\textbf{Training with CoT induce sparse and interpretable attentions (Figure \ref{fig: sample complexity with CoT and attention}).}~~We evaluate the sample complexity of a $1$-layer, $1$-head transformer on larger parity problems ($n=100$ and $20 \leq k \leq 100$). Without CoT, training exceeds the sample budget when $k \geq 4$. However, with CoT, we can successfully train for any $k$. Figure \ref{fig: sample complexity with CoT and attention}b shows the attention pattern of a $1$-layer, $1$-head transformer trained on a parity problem with $(n=40, k=20)$ using CoT on a random query. In this case, CoT processes the secret variables in ascending order. The $i$-th row illustrates the attention weight of the $i$-th token attending to the previous ones, with indices corresponding to secret variables highlighted in red. 
% Attention from query tokens is omitted for clarity.

% We show the sample complexity of a $1$-layer $1$-head transformer on larger parity problems ($n=100$ and $20\leq k \leq 100$). Training without CoT exceeds the sample budget when $k> 4$, but with CoT we are able to train for any $k$. Figure \ref{fig: sample complexity with CoT and attention}b depicts the attention pattern of a $1$-layer $1$-head transformer trained on $(n=40, k=20)$ parity problem with CoT on a random query. The CoT here processes secret variables in an ascending order. The $i$-th row represents the amount that the $i$-th token attends to previous tokens. Indices corresponding to secret variables are marked in red. The attention of the query tokens is omitted.

\textbf{Results.} In Figure \ref{fig: sample complexity with CoT and attention}a, we show that training with CoT succeed with large $n$ and $k$. Moreover, the empirical result suggests that the sample complexity of learning parity with CoT grows approximately linearly with $k$. This shows a clear contrast to without CoT setting where the required samples explodes exponentially. To validate our theorem on learning the sparse dependence, we show in Figure \ref{fig: sample complexity with CoT and attention}b that the attention pattern learned by the transformer with CoT is interpretable. Specifically, the $i$-th output token of CoT primarily focuses on the $i$-th secret variable, with minimal attention given to other tokens.  % \huaqing{TODO: More results are presented in Appendix xx.}

% \vspace{-0.1in}
\subsection{Multi-pass Training Improves Parity Learning without CoT}\label{sec:multipass}

\begin{figure}[t!]
% \vspace{-0.5in}
\setlength{\abovecaptionskip}{-1pt}
    \centering
    \includegraphics[width=0.9\linewidth]{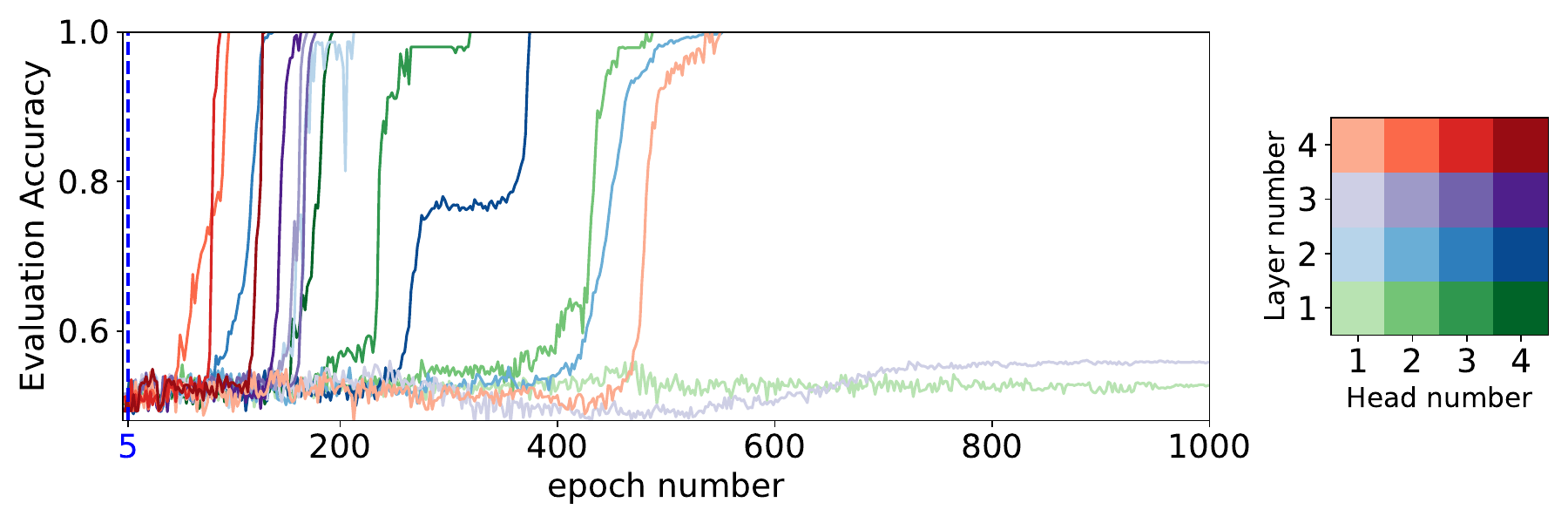}
    \caption{Evaluation accuracy of transformers on the $(n=20, k=6)$ parity problem \textbf{without} CoT. The model is trained on a dataset of $10,000$ samples for $1,000$ epochs. Almost all layer-head configuration achieve perfect evaluation accuracy. Adding more heads is more effective than adding layers. The blue dashed line marks the \textbf{with} CoT setup, which achieves perfect accuracy in 5 epochs.}
    \label{fig: multi-pass helps}
% \vspace{-0.2in}
\end{figure}

\begin{figure}[t!]
% \vspace{-0.5in}
    \centering
\includegraphics[width=\linewidth]{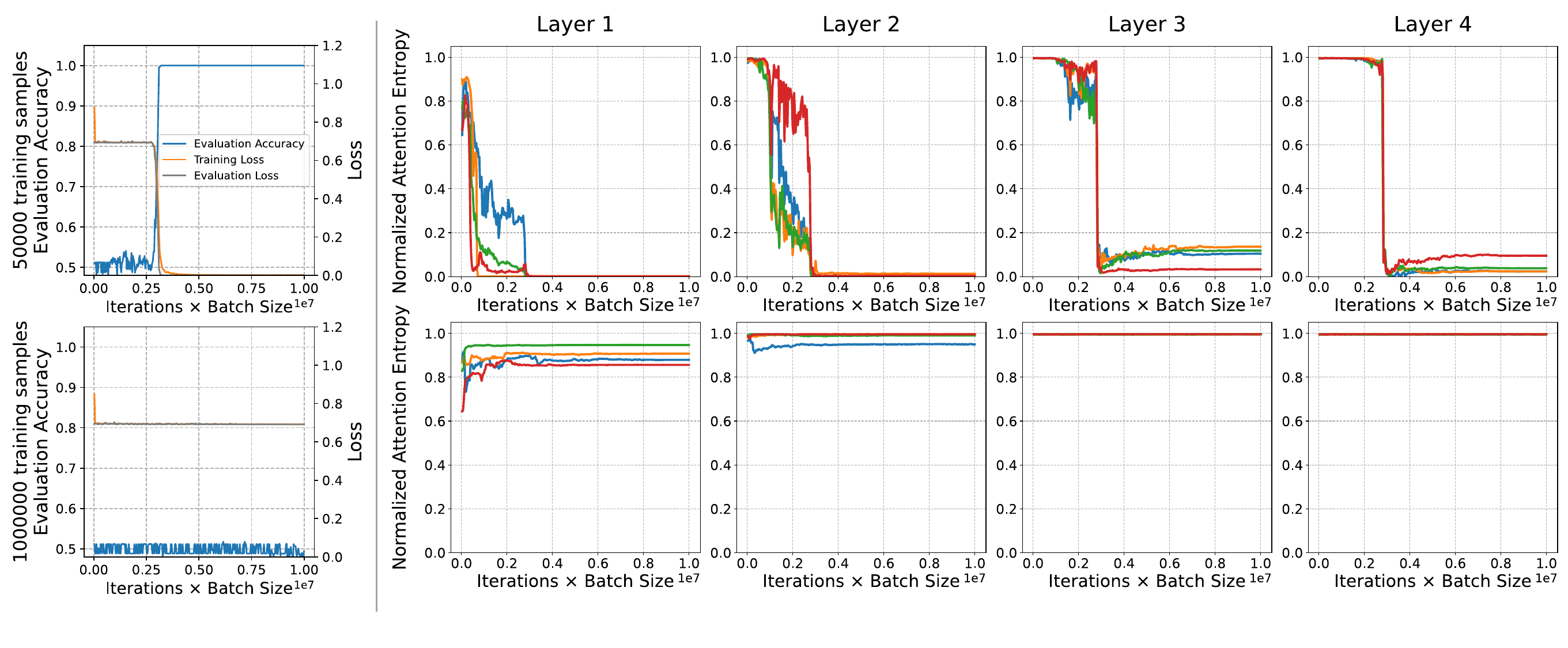}
% \vspace{-0.1in}
    \caption{$4$-layer $4$-head transformer trained on the $(n=20, k=6)$ parity problem without CoT using multi-pass training, detailed in~\Cref{sec:multipass}. When trained on a small dataset of $50000$ samples, the model \text{achieves perfect evaluation accuracy} (Top), accompanied by a significant decrease in entropy. Surprisingly, when trained on an even larger training set with $1000000$ samples, the model \text{fails to learn} (Bottom), and both the training loss and the normalized attention entropy remain elevated. This shows that the development of attention sparsity may improve optimization.}
    \label{fig: repeated-data better}
% \vspace{-0.2in}
\end{figure}

% As suggested by our lower bound in~\Cref{thm:nocot}, analyzing the sample complexity of learning parity with multi-pass training is challenging.
%due to the correlation introduced in the gradient estimator. 
% To the best of our knowledge, there are currently no established lower bounds for the sample complexity of learning parity problems beyond constant pass training. 
% To complement the established lower bounds for transformers learning parity (\Cref{thm:nocot}), we conduct empirical experiments by training transformers \textbf{without CoT} using \textbf{multi-pass training}. The results demonstrate that multi-pass training can indeed enhance the sample efficiency of transformers. We empirically discover that the key ingredient to successfully learning the parity problem is to develop sparse attention towards previous tokens (Figure \ref{fig: repeated-data better} Right), which is similar to the role of CoT data.

In our theory, we only establish lower bounds for the sample complexity of learning parity problems with constant pass SGD training. To complement the established lower bounds for transformers learning parity (\Cref{thm:nocot}), we conduct empirical experiments by training transformers {without CoT} using \textbf{multi-pass training}. We make two observations below.
First, the results demonstrate that training with repeated data can help Transformers learn the parity function, although this process still consumes significantly more computation (epochs) than trained on CoT data  (\Cref{fig: multi-pass helps}). Second, a key difference between one-pass and multi-pass training is the development of \textbf{sparse attention} (see Figure \ref{fig: repeated-data better} Right), similar to the role of CoT shown in the previous section. This shows that the development of sparse attention is crucial even in when training without CoT on the parity data.
% \huaqing{But to theoretically characterize the improvement on sample complexity (e.g. polynomial or quasi-exponential?) needs future work.}
% \huaqing{Definition of Attention Entropy, why it can measure attention sparsity}

\textbf{Experiment protocols.} On $(n=20,k=6)$ parity problem, we conduct multi-pass training without CoT. In Figure \ref{fig: multi-pass helps}, the models are trained with $10^4$ samples for $1000$ epochs, with the number of layers and heads ranging from $1$ to $4$. In Figure \ref{fig: repeated-data better}, we compare the training of $4$-layer $4$-head transformer on $5\cdot 10^4$ and $10^6$ samples respectively. The learning rate is initialized $10^{-4}$. 
% for a \textbf{fixed number of iteration steps}. 

\textbf{Normalized attention entropy.} To measure attention sparsity, we introduce the concept of \emph{normalized attention entropy} for each attention head (illustrated in~\Cref{fig: repeated-data better} Right). Let $\mathcal{P}^{\ell,h}(\mathbf{x})[i]$ denote the attention score distribution produced by the $h$-th head in the $\ell$-th layer at the $i$-th token of input $\mathbf{x}$. The {normalized attention entropy} for the input $\mathbf{x}$ is then defined as:
\begin{align}\label{eq: normalized attention entropy} \mathrm{Ent}(\mathbf{x}; \ell, h) = \min_{i\geq 2} \frac{H(\mathcal{P}^{\ell,h}(\mathbf{x})[i])}{\log i} \in [0,1], \end{align}
where $H(\mathcal{P}) = -\int \log(\mathcal{P}) d\mathcal{P}$ is the entropy of distribution $\mathcal{P}$. The normalization term $\log i$ represents the entropy of a uniform distribution over $i$ tokens, account for the varying context length. The minimum is taken over different tokens in $\mathbf{x}$ since attention heads may specialize in extracting information for specific tokens. As a result, a lower normalized attention entropy indicates a sparser attention pattern. To compute the normalized attention entropy for each head, we would average the normalized entropy across all question-answer pairs in the validation set.

\textbf{Results.} As shown in Figure \ref{fig: multi-pass helps}, when trained on $10^4$ samples, which only accounts for a small portion ($\sim 1\%$) of all possible inputs ($2^{20}\approx 10^6$), most of the transformer architectures we examined achieve perfect evaluation accuracy given sufficient epochs. While the $k=6$ problem is intractable with one-pass training without CoT, these results demonstrate that multi-pass training can indeed enhance learning in the no-CoT setup. However, CoT is by far the most effective accelerator, achieving perfect accuracy in just 5 epochs, significantly outperforming the multi-pass no-CoT training.  %and find solutions with sparse attention patterns. 
% \huaqing{Strong bias to sparse dependency?}

Although learning without CoT is less efficient, it offers a "slow-mode" trajectory. In Figure \ref{fig: repeated-data better} (Top), the no-CoT loss function initially plateaus before eventually dropping to zero. During this plateau, the entropy in the attention layers continues to decrease, indicating that feature learning is occurring gradually. When all the attention heads become sparse, a transition phase occurs in the loss and accuracy: the loss drops to zero, and evaluation accuracy jumps to 1. In contrast, in the failure case shown in Figure \ref{fig: repeated-data better} (Bottom), the entropy of the attention layers remains high throughout.

% Interestingly, with a fixed number of iterations, a model trained on a smaller dataset over more epochs achieves perfect evaluation accuracy (Figure \ref{fig: repeated-data better}, Up), while a model trained on a larger dataset with fewer epochs does not (Figure \ref{fig: repeated-data better}, Down). This contradicts the common belief that a more diverse dataset enhances models' generalization and highlights the necessity of multi-pass training for learning the parity problem without CoT. 
This experiment on no-CoT confirms that sparse attention is crucial for parity learning. As previously demonstrated in Figure \ref{fig: motivating example}, CoT accelerates learning by quickly inducing sparsity. This suggests that CoT not only improves the sample efficiency but also improve the optimization landscape by facilitating sparsity, due to the introduction of sparse dependencies added in the intermediate steps.
% As demonstrated in Figure \ref{fig: repeated-data better} Right, successful learning with a smaller dataset coincides with a significant decrease in attention entropy, indicating the development of sparse attention. In contrast, entropy remains elevated when trained on a larger dataset.
% \vspace{-0.1in}
\subsection{CoT Induces Sparsity on Real-World Data}\label{sec: GSM8K}

% \vspace{-0.12in}
\begin{figure}[t] 
% \setlength{\abovecaptionskip}{-1pt}
% \vspace{-0.5in}
    \centering
    \includegraphics[width=\linewidth]{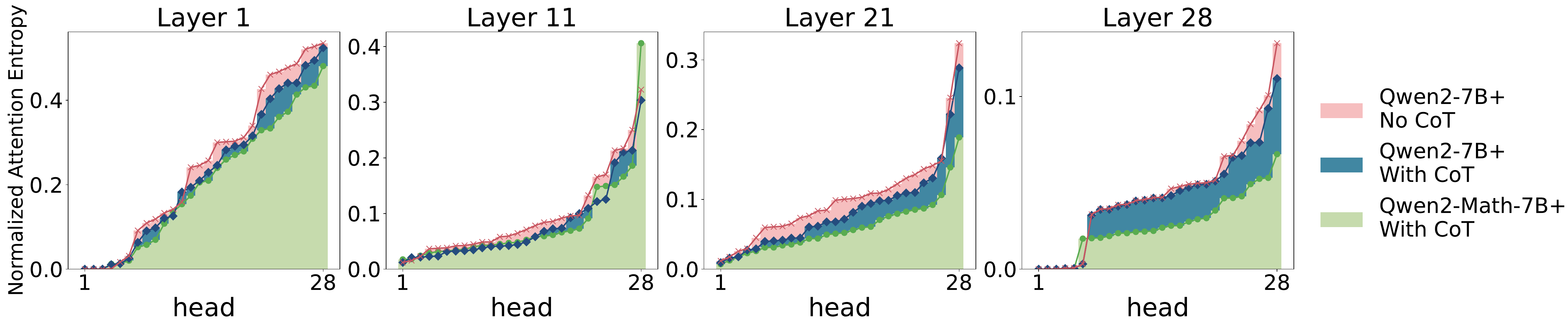}
    \caption{We compare the normalized attention entropy of the pre-trained Qwen2-7B and math-specialized Qwen2-Math-7B models on the GSM8K dataset with and without CoT prompting (Section \ref{sec: GSM8K}). Each bar represents the entropy of an attention head. The Qwen2-7B model exhibits sparser attention when processing CoT data, indicating that real-world CoT data has a sparser structure. The entropy difference between the Qwen2-7B and Qwen2-Math-7B model suggests that fine-tuning on CoT data promotes the development of sparser attention patterns of the model.}
    \label{fig: GSM8K data}
    % \vspace{-0.28in}
\end{figure}
% \vspace{-0.1in}

% \vspace{-0.1in}
% In previous sections, we theoretically and empirically demonstrate that CoT exponentially reduces the sample complexity required for transformers to learn the synthetic parity problem by reducing the sequential dependence. Furthermore, models trained on CoT data exhibit sparse attention patterns. In this section, we conduct experiments with pre-trained models on the real-world
Now we move from the synthetic parity problem to real world experiment on GSM8K dataset of grade-school math word problems \citep{cobbe2021training}. We observe that: 
% \vspace{-0.1in}
\begin{enumerate}[leftmargin=*]
    \item Real-world CoT data also exhibits sparse sequential dependence, leading to a sparser attention pattern in pre-trained models.
    \item Fine-tuning on CoT data further enhances the models' attention sparsity on the input data.
\end{enumerate}
% \vspace{-0.1in}

% As shown in Figure \ref{fig: GSM8K data}, the attention entropy i% s significantly lower for the CoT data, indicating that real-world CoT data indeed exhibits a sparser structure.
\textbf{Experiment protocols.} In Figure \ref{fig: GSM8K data}, we examine two data types: \texttt{With CoT}, where inputs from GSM8K dataset are concatenated with ground truth answers that include multiple reasoning steps, and \texttt{No CoT}, where inputs are directly concatenated with the final answer. We evaluate two language models: the pre-trained model Qwen2-7B \citep{yang2024qwen2} and the specialized mathematics model Qwen2-7B-Math \citep{qwen2024} which is fine-tuned from Qwen2-7B on a mathematics-specific corpus with CoT data. 
We plot the normalized attention entropy (\Cref{eq: normalized attention entropy}) across different heads. More details can be found in the Appendix \Cref{sec: experiment detail}. % \huaqing{TODO. normalized attention entropy. data format, average over test set. Model architecture (28 heads and layers)}.

\textbf{Results.}
Unlike the synthetic parity problem, the sequential dependency of real-world data is hard to measure directly. However, it can be inferred from the attention sparsity of pre-trained models when they process such data as input. As shown in Figure \ref{fig: GSM8K data}, comparing the normalized attention entropy of the same pre-trained model Qwen2-7B on different types of data, we can see that the entropy is lower for \texttt{With CoT} data compared to \texttt{No CoT} data, indicating that real-world CoT data indeed exhibits a sparser structure. Furthermore, on the same \texttt{With CoT} data, Qwen2-Math-7B model demonstrates lower attention entropy compared to Qwen2-7B model, suggesting that fine-tuning on CoT data promotes the development of sparser attention patterns in the model.

% We then evaluate the attention sparsity of two language models, Qwen2-7B and Qwen2-Math-7B \huaqing{cite}, using the GSM8K dataset. Qwen2-Math-7B is a specialized mathematics language model developed from Qwen2-7B and fine-tuned on a mathematics-specific corpus. As illustrated in Figure \ref{fig: GSM8K model}, Qwen2-Math-7B demonstrates sparser attention compared to Qwen2-7B when processing the same input, suggesting that fine-tuning on Chain-of-Thought (CoT) data promotes the development of sparser attention patterns in the model.%which contains 

% In the parity problem we consider above, CoT clearly reduces the dependency degree from $k$ to $1$. However, the dependency degree of real-world data is hard to measure directly. Thus 

% We use the GSM8K dataset \citep{cobbe2021training}

% \begin{figure}
%     \centering
%     \includegraphics[width=\linewidth]{Figures/fig_GSM8K_ex1.pdf}
%     \caption{CoT fine-tuning makes the model's attention pattern sparse.}
%     \label{fig: GSM8K model}
% \end{figure}

\section{Additional Related Work}
\textbf{Parity Learning.} The most relevant work to ours is~\cite{wies2023subtask}, which shows that subtask decomposition enables learning the parity problem with polynomial sample complexity in recurrent neural networks (RNNs). Their learnability results rely on techniques from~\cite{wang2022provablegeneralizationrecurrentneural}, which operate within the NTK linearization regime of RNNs. In contrast, the Transformers analyzed in our work exhibit feature learning and identify the sparse secret set within the attention module. On the positive side, a line of works that studies optimization dynamics of neural networks on parity~\citep{kou2024matchingstatisticalquerylower,barak2023hiddenprogressdeeplearning,edelman2023paretofrontiersneuralfeature,daniely2020learningparitiesneuralnetworks,abbe2023sgdlearningneuralnetworks,abbe2024mergedstaircasepropertynecessarynearly} show that $n^{\Omega(k)}$ samples is sufficient to learn parity, which is close to the statistical query lower bound \citep{kearns1998efficient}. 
%On the negative side, it has been shown that in SGD learning with \textbf{additional noise}, $n^{\Omega(k)}$ samples are indeed needed \citep{shalevshwartz2017failuresgradientbaseddeeplearning,abbe2020polytimeuniversalitylimitationsdeep}. 
{On the negative side, it is well-established that learning parity using gradient descent (GD) requires $\Omega(n^k)$ iterations due to the SQ lower bound \citep{kearns1998efficient}. For SGD learning, \cite{shalevshwartz2017failuresgradientbaseddeeplearning,abbe2020polytimeuniversalitylimitationsdeep} showed that an exponential number of samples is necessary when SGD is hindered by additional noise or when the number of weights updated in each step is constrained. } However, it remains unclear whether standard mini-batch SGD will still need exponential sample complexity. In our paper, we take the alternative approach to consider the training as an online algorithm with bounded memory (parameters) and utilize results from online learning literature to provide a rigorous exponential lower bound \citep{lyu2023tighttimespacelowerbounds,kol2017time}. \cite{hahn2024sensitivefunctionshardtransformers} shows that while parity is easy to represent using Transformers, the sensitivity structure of the function will require the representation to have large weight norms. {A recent work \citep{abbe2024far} conjectured that a distribution is weakly learnable by a Transformer if and only if it has constant globality, which is supported by empirical experiments. The work also highlights the parity problem as a notable special case.}
 
\textbf{Chain-of-Thoughts (CoT).} CoT is the technique to let the model generate a reasoning process before final answers.
CoT prompting has proven effective in enhancing language models' reasoning capabilities~\citep{wei2022chain,zhang2022automaticchainthoughtprompting,wang2023planandsolvepromptingimprovingzeroshot,zhou2023leasttomostpromptingenablescomplex,wang2024chainofthoughtreasoningprompting,creswell2023selectioninference}. Training with CoT data has further improved the model's capability in performing complex reasoning (see~\citet{qwen2024,yue2023mammothbuildingmathgeneralist,yu2024metamathbootstrapmathematicalquestions,kim2023cotcollectionimprovingzeroshot} and reference therein).
Different lines of work have explored the effect of CoT. From the representation theory perspective, works including \cite{feng2024towards,merrill2024expressivepowertransformerschain,li2024chainthoughtempowerstransformers,nowak2024representationalcapacityneurallanguage,wen2024rnnstransformersyetkey} show that CoT can provably expand the representation power of neural architectures including Transformers and RNNs. From the statistical approximation level, prompting with CoT reduces the statistical error \citep{hu2024unveilingstatisticalfoundationschainofthought,prystawski2023thinkstepstepreasoning}. \cite{li2023dissectingchainofthoughtcompositionalityincontext} studies how MLP learns with CoT data. \cite{li2024how} studies how Transformers learn to perform CoT through prompting using gradient descent. We differ from this work as (1) we focus on how Transformer captures the sparse sequential dependency in CoT, which is not reflected in their data modeling; (2) we study zero-shot CoT data directly. \citet{dutta2024thinkstepbystepmechanisticunderstanding,wang2024understandingtransformerperformmultistep} studies pre-trained Transformers' activation on CoT data. They highlight the attention head's role in moving the essential information from the context to the reasoning process, which is consistent with our theoretical insight. Concurrently with our work, \citet{juno2024transformer} theoretically demonstrated that CoT reduces the number of gradient steps required to solve the parity problem. While their analysis focuses on GD with a noisy oracle, our work investigates mini-batch SGD and provides theoretical guarantees on the exponential reduction in sample complexity.

% We defer other related works to~\Cref{sec:app:related}.

\textbf{Transformer Optimization Dynamics.} Our works fall in the line of works that studies the optimization dynamics of Transformers on synthetic datasets~\citep{pmlr-v202-li23p,chen2024trainingdynamicsmultiheadsoftmax,kim2024transformerslearnnonlinearfeatures,wibisono2024context,chan2022data,wibisono2024how,cole2024provableincontextlearninglinear,sheen2024implicitregularizationgradientflow,chen2024trainingdynamicsmultiheadsoftmax,NEURIPS2023_e359ebe5,nichani2024transformerslearncausalstructure,10636559}. Similar to our works, \cite{wang2024transformersprovablylearnsparse} highlights that Transformers can learn to select sparse critical tokens from the context on linear data. We differ from their work in studying the mini-batch optimization dynamics on nonlinear data and hence establishing sample complexity bounds. 

% \vspace{-0.2in}
\section{Conclusion and Future Works.} 
% \vspace{-0.1in}
Our work demonstrates that Transformers trained on Chain-of-Thought (CoT) data with sparse sequential dependencies can efficiently learn sparse attention mechanisms, accurately capturing these dependencies while requiring exponentially fewer samples than models trained without CoT. Our current analysis of CoT training assumes that each subsequent token depends on exactly one token in the context and focuses solely on the parity function. A promising future direction is to explore more general scenarios where each token depends on multiple previous tokens and extends to function classes beyond parity.
\section*{Acknowledgements}
Jingzhao Zhang is supported by National Key R\&D Program of China 2024YFA1015800 and Shanghai Qi Zhi Institute Innovation Program.

\bibliographystyle{plainnat}
\bibliography{iclr2025_conference}

\appendix
\section{Additional Related Work}
\label{sec:app:related}

\textbf{Transformer Optimization Dynamics.} Our works fall in the line of works that studies the optimization dynamics of Transformers on synthetic datasets~\citep{pmlr-v202-li23p,chen2024trainingdynamicsmultiheadsoftmax,kim2024transformerslearnnonlinearfeatures,wibisono2024context,chan2022data,wibisono2024how,cole2024provableincontextlearninglinear,sheen2024implicitregularizationgradientflow,chen2024trainingdynamicsmultiheadsoftmax,NEURIPS2023_e359ebe5,nichani2024transformerslearncausalstructure,10636559}. Similar to our works, \cite{wang2024transformersprovablylearnsparse} highlights that Transformers can learn to select sparse critical tokens from the context on linear data. We differ from their work in studying the mini-batch optimization dynamics on nonlinear data and hence establishing sample complexity bounds. 

\section{Omitted Proof}

\subsection{Notation and Assumptions.}

We will denote a sequence of binary variables as $\bsequence[1], \ldots, \bsequence[T]$ with $\bsequence[i] \in \{0, 1\}$. We will use $[j]$ to denote variables corresponding to the $j-$th token dimension without otherwise specified.

\subsubsection{Training Specification}

\paragraph{Training Update} We will consider the following 1-pass SGD with batch size $B$:

\begin{align*}
    L^{(t)} &=  \sum_{s = 1}^B \sum_{i = n + 1}^{n + k} \ell( \transformer[w^{(t)}](\bsequence^{(t,s)})[i], \bsequence^{(t,s)}[i + 1]).\\
    w^{(t + 1)} &= w^{(t)} - \eta_t \frac{1}{B}  \nabla_{w} L^{(t)}.
\end{align*}

Here $w$ includes $\qk$ and $\weight$. We will leave $\head$ unchanged for the simplicity of analysis.

\subsection{Representation Theory}
\label{app:thm:rep}

\begin{proof}[Proof of~\Cref{thm:rep}]

We will set $d = \Theta(k \log(n / \delta))$ and $m = k+1$. We will use $p(x, y)$ as shorthand for $2 (x - 1) + y + 1$.

Define embedding matrix $M$ as follows
\begin{align*}
    M_{p(i, b)} = e_{S[i], b}^T.
\end{align*}
for $\{S[1], \ldots, S[p]\} = \mathcal{S}$. Then $M \in \R^{2k \times d}$.

\begin{lemma}
\label{lem:embedmatrix}
For large enough $n$, with probability $1 - \delta$, it holds that
\begin{align*}
    \| M^T \|_{\mathrm{op}} = O(1). \quad
    \lambda_{\mathrm{min}}(MM^T) > 1/12.
\end{align*}
\end{lemma}

\begin{proof}

The first result follows from~\Cref{lem:matrixbound}, noted that $K = \Theta(\sqrt{1/d})$. The second result follows from~\Cref{lem:mineigen}.
\end{proof}

\begin{lemma}
\label{lem:indicatorvec}
 With probability $1 - \delta/4$ over the randomness of embedding $e_{j,b}$ for $j \in \mathcal{S}, b \in \{0,1\}$, there exists $u_{j,b}$, satisfying that $\forall j \in \mathcal{S}, b \in \{0, 1\}$,
    \begin{align*}
        \forall j' \in \mathcal{S}, b' \in \{0, 1\}, \langle u_{j,b} , e_{j',b'} \rangle  &= \1( (j, b) = (j',b')). 
    \end{align*}
    Further, 
    \begin{align*}
        \| \sum_{j \in \mathcal{S}, b \in \{0,1\}}{u_{j,b}} \|_2 = O(\sqrt{k}).
    \end{align*}
\end{lemma}

\begin{proof}
    We can simply choose 
    \begin{align*}
        u_{S[i], b} &= M^T(MM^T)^{-1} o_{2(i - 1) + b}.
    \end{align*}
    with $o_{2(i - 1) + b}$ being the one-hot vector in $\R^{2k}$.
\end{proof}

\begin{lemma}
\label{lem:indicatorvec2}
    When the event in~\Cref{lem:indicatorvec} happens, there exists $v_j$,
    \begin{align*}
        \langle v_j, e_{j,0} \rangle = \langle v_j, e_{j,1} \rangle &= 1. \\
        \forall j' \neq j,  j \in S, \langle v_{j'}, e_{j,b} \rangle  &= 0. 
    \end{align*}
    Further, 
    \begin{align*}
        \| \sum_{j \in \mathcal{S}}{v_j} \|_2 \le O(\sqrt{k}).
    \end{align*}
\end{lemma}

\begin{proof}
We can choose $v_j = u_{j,0} + u_{j,1}$.
\end{proof}

\begin{lemma}
\label{lem:noother}
    When $v_i$ defined in~\Cref{lem:indicatorvec} exists, for all $j \not \in \mathcal{S}, b \in \{0, 1\}$, with probability $1 - \delta$, it holds that
    \begin{align*}
        |\langle \sum_{j'\in \mathcal{S}} v_{j'}, e_{j,b} \rangle| \le 1/2.
    \end{align*}
\end{lemma}

\begin{proof}
    Noted that $e_{j,b}$ is independent from  $v_{j'}$, assuming $V = \sum_{j'\in \mathcal{S}} v_{j'}$, then $\|V\|_2 = O(\sqrt{k})$, using Azuma-Hoeffding Bound, with probability $1 - \delta/2$
    \begin{align*}
        |\langle \sum_{j'\in \mathcal{S}} v_{j'}, e_{j,b} \rangle| = O(\frac{\| \sum_{j'\in \mathcal{S}} v_{j'} \|_2}{\sqrt{d}} \sqrt{\log(2k/\delta)} ) = O(\sqrt{\frac{k \log(2k/\delta)}{d}}) \le \frac{1}{2}.
    \end{align*}
    The proof is then completed.
\end{proof}

Assuming the event in~\Cref{lem:indicatorvec,lem:noother} happens, defined $V =\sum_{j'\in \mathcal{S}} v_{j'}$ and $U = \sum_{j'\in \mathcal{S}} u_{j',1}$
\begin{align*}
    \qk = 40\log(n) V e_{n+1,0}^T \quad
    W_{r, 1:d} = a_r e_{n+1,0} \quad 
    W_{r, d+1:2d} = b_r  U.
\end{align*}
for some scalar $a_r, b_r$ to be determined.

For an arbitrary $\bsequence$, it holds that, 
\begin{enumerate}
    \item If $j \in \mathcal{S}$, $e_{j, \bsequence[j]}^T Ae_{n+1,0} = 40 \log n$.
    \item If $j \not \in \mathcal{S}$, $e_{j, \bsequence[j]}^T Ae_{n+1,0} \le 20 \log n$.
\end{enumerate}

If we define
\[
z_j = e_{j, \bsequence[j]}^T A e_{n+1,0}.
\]

The softmax value is given by:
\[
\softmax(z)_{j, n+1} = \frac{e^{z_j}}{\sum_{j'} e^{z_{j'}}}.
\]

\begin{enumerate}[leftmargin=*]
    \item For $j \in \mathcal{S}$:
\[
e^{z_j} = e^{40 \log n} = n^{40}.
\]
\item For $j \not \in \mathcal{S}$:
\[
e^{z_j} \le e^{20 \log n} = n^{20}.
\]
\end{enumerate}

With $|\mathcal{S}| = k$, the denominator of the softmax becomes:
\[
\sum_{j'} e^{z_{j'}} = \sum_{j' \in \mathcal{S}} e^{40 \log n} + \sum_{j' \notin \mathcal{S}} e^{z_{j'}} \in [kn^{40}, k n^{40} + n^{21}].
\]

For $j \in \mathcal{S}$:
\[
\softmax(z)_{j, n+1} = \frac{n^{40}}{k n^{40} + O(n^{20})} = \frac{1}{k} \cdot \frac{1}{1 + O(n^{-19})} = \frac{1}{k} + O(n^{-19}).
\]

For $j \notin \mathcal{S}$:
\[
\softmax(z)_{j, n+1} = \frac{e^{z_j}}{k n^{40} + O(n^{20})} \leq \frac{n^{20}}{k n^{40}} = O(n^{-20}).
\]

We then have
\begin{align*}
    \softmax(\embed(\bsequence) A \embed(\bsequence))[j, n+1] &= \frac{\1(j \in \mathcal{S})}{k} + O(\frac{1}{n^7}).
\end{align*}

As 
\begin{align*}
    \langle U, e_{j,b} \rangle = \1(j \in \mathcal{S}, b = 1).
\end{align*}

We have that
\begin{align*}
    \langle U, \attention(\embed(\bsequence))[n + 1] \rangle = \frac{1}{k} \sum_{j \in \mathcal{S}} \bsequence[j] + O(\frac{1}{n^6}).
\end{align*}

% capital U?

Now we only need to map from the summation $\sum_{j \in \mathcal{S}} \bsequence[j]$ to the parity of the summation.%This can be done by choosing $a_r$ and $b_r$ iteratively. The proof is then complete.
We set 
\begin{align*}
b_r&=2k,\\
a_r&=  \left\{ 
\begin{array}{ll}
-4\lceil \frac {r} 2 \rceil + 4 & \text{if } 1\leq r\leq m ,\\
1 & \text{if } r = m+1, \\
-4\lfloor \frac {r-m} 2  \rfloor+2& \text{if } m+2\leq r\leq 2m.
\end{array}
\right.\\
h_r&=  \left\{ 
\begin{array}{ll}
1 & \text{if } 1\leq r\leq m ,\\
-1& \text{if } m+1\leq r\leq 2m,\\
\end{array}
\right.
\end{align*}

Then 
\begin{align*}
\mathcal T(\mathbb \embed(\bsequence))[n+1]&=\sum_{r=1}^{k}h_r \relu\left(a_r- b_r\langle U,\attention(\embed(\bsequence))[n+1]\rangle \right)\\
&=\sum_{r=1}^{m}\relu(a_r+2\sum_{j\in S}\bsequence[j])-\sum_{r=m}^{2m}\relu(a_r+2\sum_{j\in S}\bsequence[j])+O(\frac 1 {n^4})\\
&=(-1)^{\sum_{j\in S}\bsequence[j]+1}+O(\frac 1 {n^4})\\
&=(-1)^{\bsequence[n+2]+1}+O(\frac 1 {n^4}).
\end{align*}
Thus $\sgn(\transformer(\mathbb \embed(\bsequence))[n+1])=(-1)^{\bsequence[n+2]}$ for large enough $n$.

Note that the range of parameters is polynomial with $n$, thus could be represented with $\Theta(\log n )$ precision. With $\Theta(\log n)$ precision, the error of each activation can be bounded by a small inverse polynomial of $n$ and hence we can show the same result.
% Here $a_r$ and $b_r$ will grow exponentially with $k$, but there are only $2k$ such parameters and hence this part of the weight can be encoded in $o(nk)$ memory when $k = o(n)$. The rest parameters $u, v$ can be encoded in $O(\log n)$ precision and hence the configuration only requires $o(nk)$ bits memory to describe.
% The proof is then completed.

\end{proof}

\subsection{Dynamics without CoT}
\label{app:thm:nocot}

\begin{proof}
    We will utilize the Theorem 6 in~\citet{lyu2023tighttimespacelowerbounds}, which shows that any branching programs with $o(nk)$ memory will require exponential samples to learn sparse parities with constant passes. Here the frozen embedding matrix $e$, which will utilize naively $O(nd)$ memory, can't be saved in the memory. However, we can take the alternative approach to regenerate $e$ using a random number generator with a fixed seed on each step. This allow us to simulate standard SGD optimization with $o(nk)$ memory, which is a special case of branching programs.  
\end{proof}

\subsection{Dynamics with CoT}
\label{app:thm:cot}

\begin{assumption}
    \label{assum:order}
    Consider the following conditions for sufficiently large \( n \):
    \begin{enumerate}[leftmargin=*]
        \item The secret Hamming weight satisfies \( k \in \left[ \frac{n}{\log^{5}(n/\delta)}, \frac{n}{\log^{4}(n/\delta)} \right] \).
        \item Set \( d = k \log^{1.1}\left(\frac{n}{\delta}\right) \) and \( m = 2k \). This implies \( md \log n = o(nk) \).
        \item Define the batch size as \( B = C_2 n \log^{20}\left(\frac{n}{\delta}\right) \) for a sufficiently large constant \( C_2 = \frac{1.28 \times 10^7}{\epsilon^2} \).
        \item Set the learning rates to be
        \[
            \eta_0 = \eta_1 = \frac{m \epsilon \sqrt{B}}{100 \log\left(\frac{n}{\delta}\right)}, \quad \eta_2 = \frac{4k\epsilon}{3}.
        \]
    \end{enumerate}
\end{assumption}
   \begin{lemma}
    \label{lem:order}
    Under Assumption \ref{assum:order}, the following conditions hold as \( n \to \infty \):
    \begin{enumerate}[leftmargin=*]
        \item \( \lim_{n \to \infty} \frac{d}{n} = 0 \) and \( \lim_{n \to \infty} \frac{B}{nk} = 0 \).
        \item \( \lim_{n \to \infty} \sqrt{\frac{2k}{d} \log\left(\frac{Bmn}{\delta}\right)} = 0 \).
        \item Each of the following expressions tends to zero:
        \begin{enumerate}[label=(\alph*)]
            \item \( \frac{\sqrt{k \log\left(\frac{300 m n B}{\delta}\right)}}{\sqrt{B}} \).
            \item \( \frac{\sqrt{2k \log\left(\frac{400 m n}{\delta}\right)}}{\sqrt{d}} \).
            \item \( \frac{\sqrt{2nk \log\left(\frac{300 m n B}{\delta}\right)}}{\sqrt{B d}} \).
        \end{enumerate}
        \item \( \eta_0 = \eta_1 \leq \min\left\{ \frac{3m \epsilon \sqrt{B}}{80 \sqrt{\log\left(\frac{300nB}{\delta}\right)}}, \frac{mn\epsilon}{120} \right\} \).
        \item \( \lim_{n \to \infty} \log^{2}\left(\frac{300 m n B}{\delta}\right) \frac{n \sqrt{nk}}{d \sqrt{B d}} = 0 \).
        \item \(  \frac{\eta_0 \eta_1}{256 m n^2} > 40 \log n \).
        \item $\lim_{n \to \infty} \frac{\eta_2}{\eta_0} = 0.$
        \item \( \lim_{n \to \infty} (\eta_0 + \eta_1) < 5n \eta_2 \).
    \end{enumerate}
\end{lemma}

\begin{proof}
    The conditions in Lemma \ref{lem:order} are satisfied based on the definitions provided in Assumption \ref{assum:order}. Below, we outline the verification for each condition:

    \begin{enumerate}[leftmargin=*]
        \item \textbf{Limit of \( \frac{d}{n} \) and \( \frac{B}{nk} \)}:
        \[
            \frac{d}{n} = \frac{k \log^{1.1}\left(\frac{n}{\delta}\right)}{n} \in \left[ \frac{1}{\log^{3.9}(n/\delta)}, \frac{1}{\log^{2.9}(n/\delta)} \right] \to 0 \text{ as } n \to \infty.
        \]
        Similarly,
        \[
            \frac{B}{nk} = \frac{C_2 n \log^{20}\left(\frac{n}{\delta}\right)}{n k} = \frac{C_2 \log^{20}\left(\frac{n}{\delta}\right)}{k} \leq \frac{C_2 \log^{20}\left(\frac{n}{\delta}\right)}{\frac{n}{\log^{5}(n/\delta)}} = C_2 \frac{\log^{25}(n/\delta)}{n} \to 0.
        \]
        
        \item \textbf{Limit of \( \sqrt{\frac{2k}{d} \log\left(\frac{Bmn}{\delta}\right)} \)}:
        \[
            \sqrt{\frac{2k}{d} \log\left(\frac{Bmn}{\delta}\right)} = \sqrt{\frac{2k}{k \log^{1.1}\left(\frac{n}{\delta}\right)} \log\left(C_2 n \log^{20}\left(\frac{n}{\delta}\right) \cdot m n / \delta\right)}.
        \]
        Simplifying,
        \[
            \sqrt{\frac{2k}{d} \log\left(\frac{Bmn}{\delta}\right)} = \sqrt{\frac{2}{\log^{1.1}\left(\frac{n}{\delta}\right)} \cdot O(\log n)} = O\left( \frac{\sqrt{\log n}}{\log^{0.55}(n/\delta)} \right) \to 0.
        \]
        
        \item \textbf{Limits of Sub-Inequalities (a), (b), and (c)}:
        \begin{enumerate}[label=(\alph*)]
            \item 
            \[
                \frac{\sqrt{k \log\left(\frac{300 m n B}{\delta}\right)}}{\sqrt{B}} = \frac{\sqrt{k \cdot O(\log n)}}{\sqrt{C_2 n \log^{20}\left(\frac{n}{\delta}\right)}} = O\left( \frac{\sqrt{k} \cdot \sqrt{\log n}}{\sqrt{n} \cdot \log^{10}(n/\delta)} \right).
            \]
            Given \( k \leq \frac{n}{\log^{2}(n/\delta)} \),
            \[
                \frac{\sqrt{k} \cdot \sqrt{\log n}}{\sqrt{n} \cdot \log^{10}(n/\delta)} \leq \frac{\sqrt{\frac{n}{\log^{2}(n/\delta)}}  \sqrt{\log n}}{\sqrt{n} \cdot \log^{10}(n/\delta)} = \frac{\log^{0.5}(n)}{\log^{11}(n/\delta)} \to 0.
            \]
            
            \item 
            \[
                \frac{\sqrt{2k \log\left(\frac{400 m n}{\delta}\right)}}{\sqrt{d}} = \frac{\sqrt{2k \cdot O(\log n)}}{\sqrt{k \log^{1.1}\left(\frac{n}{\delta}\right)}} = O\left( \frac{\sqrt{\log n}}{\log^{0.55}(n/\delta)} \right) \to 0.
            \]
            
            \item 
            \[
                \frac{\sqrt{2nk \log\left(\frac{300 m n B}{\delta}\right)}}{\sqrt{B d}} = \frac{\sqrt{2nk \cdot O(\log n)}}{\sqrt{C_2 n \log^{20}\left(\frac{n}{\delta}\right) \cdot k \log^{1.1}\left(\frac{n}{\delta}\right)}} = O\left( \frac{\sqrt{nk \log n}}{\sqrt{nk} \cdot \log^{10.55}(n/\delta)} \right) \]
            \[
                = O\left( \frac{\sqrt{\log n}}{\log^{10.55}(n/\delta)} \right) \to 0.
            \]
        \end{enumerate}
        
        \item \textbf{Bound on \( \eta_0 \) and \( \eta_1 \)}:
        \[
            \eta_0 = \eta_1 = \frac{m \epsilon \sqrt{B}}{100 \log\left(\frac{n}{\delta}\right)} = \frac{2k \epsilon \sqrt{C_2 n \log^{20}\left(\frac{n}{\delta}\right)}}{100 \log\left(\frac{n}{\delta}\right)} = O\left( {k \epsilon \sqrt{n} \log^{9}(n/\delta)} \right).
        \]
        Comparing to the minimum of the two terms:
        \[
            \frac{3m \epsilon \sqrt{B}}{80 \sqrt{\log\left(\frac{300nB}{\delta}\right)}} = O\left( k \epsilon \sqrt{n} \log^{10}(n/\delta) \right),
        \]
        and
        \[
            \frac{mn\epsilon}{120} = O(k n \epsilon).
        \]
        Since \( \eta_0 \) scales similarly to the first term and \( \sqrt{n} \log^{10}(n/\delta) \ll n \) for large \( n \), the inequality \( \eta_0 \leq \min\{ \cdot \} \) holds.
        
        \item \textbf{Limit of \( \log^{2}\left(\frac{300 m n B}{\delta}\right) \frac{n \sqrt{nk}}{d \sqrt{B d}} \)}:
        \[
            \log^{2}\left(\frac{300 m n B}{\delta}\right) = O(\log^{2}(n/\delta)),
        \]
        and
        \begin{align*}
            \frac{n \sqrt{nk}}{d \sqrt{B d}} &= \frac{n \sqrt{nk}}{k \log^{1.1}\left(\frac{n}{\delta}\right) \cdot \sqrt{C_2 n \log^{20}\left(\frac{n}{\delta}\right) \cdot k \log^{1.1}\left(\frac{n}{\delta}\right)}} \\ &= O\left( \frac{n \sqrt{nk}}{k \log^{1.1}(n/\delta) \sqrt{n k} \log^{10.55}(n/\delta)} \right) = O\left( \frac{n}{k \log^{11.65}(n/\delta)} \right).
        \end{align*}

        Given \( k \geq \frac{n}{\log^{5}(n/\delta)} \),
        \[
            \frac{n}{k \log^{11.65}(n/\delta)} \leq \frac{n}{\frac{n}{\log^{5}(n/\delta)} \log^{11.65}(n/\delta)} = \frac{1}{\log^{5.65}(n/\delta)}.
        \]
        Therefore,
        \[
            \log^{2}\left(\frac{300 m n B}{\delta}\right) \frac{n \sqrt{nk}}{d \sqrt{B d}} = O\left( \frac{\log^{2}(n/\delta)}{\log^{5.65}(n/\delta)} \right) = O\left( \log^{-4.65}(n/\delta) \right) \to 0 \text{ as } n \to \infty.
        \]
        
        \item \textbf{Limit of \( \frac{\eta_0 \eta_1}{256 m n^2} \)}:
As
\[
\eta_0 \eta_1 = \left(\frac{m \epsilon \sqrt{B}}{100 \log\left(\frac{n}{\delta}\right)}\right)^2 = \frac{m^2 \epsilon^2 B}{10^4 \log^2\left(\frac{n}{\delta}\right)}.
\]
Substituting into the left term:
\[
\frac{\eta_0 \eta_1}{256 \, m \, n^2} = \frac{m^2 \epsilon^2 B}{256 \times 10^4 \log^2\left(\frac{n}{\delta}\right) \, m \, n^2} = \frac{m \epsilon^2 B}{256 \times 10^4 \log^2\left(\frac{n}{\delta}\right) \, n^2}.
\]
Substitute \( m = 2k \) and \( B = C_2 n \log^{10}\left(\frac{n}{\delta}\right) \):
\[
\frac{\eta_0 \eta_1}{256 \, m \, n^2} = \frac{2k \epsilon^2 C_2 n \log^{10}\left(\frac{n}{\delta}\right)}{256 \times 10^4 \log^2\left(\frac{n}{\delta}\right) n^2} = \frac{2k C_2 \log^8\left(\frac{n}{\delta}\right)}{256 \times 10^4 n}.
\]
Setting \( C_2 = \frac{1.28 \times 10^7}{\epsilon^2} \):
\[
\frac{\eta_0 \eta_1}{256 \, m \, n^2} \ge \frac{2k \times 1.28 \times 10^7 \log^8\left(\frac{n}{\delta}\right)}{256 \times 10^4 \times n} = 10 \times \frac{k \log^8\left(\frac{n}{\delta}\right)}{n}.
\]

Using \( k \geq \frac{n}{\log^5\left(\frac{n}{\delta}\right)} \) from the assumption:
\[
\frac{\eta_0 \eta_1}{256 \, m \, n^2} \geq 10 \times \frac{\frac{n}{\log^5 \left(\frac{n}{\delta}\right)} \log^8\left(\frac{n}{\delta}\right)}{n} = 10 \log^3 \left(\frac{n}{\delta}\right).
\]

Since \( \log^3\left(\frac{n}{\delta}\right) > 4\log n \) for sufficiently large \( n \), we have:
\[
10 \log^3\left(\frac{n}{\delta}\right) > 40 \log n.
\]

Thus,
\[
\frac{\eta_0 \eta_1}{256 \, m \, n^2} > 40 \log n.
\]
        
        \item We aim to show that:
\[
\frac{\eta_2}{\eta_0} \to 0 \quad \text{as} \quad n \to \infty.
\]
Here
\[
\eta_0 = \frac{2k \epsilon \sqrt{C_2 \, n \, \log^{10}\left(\frac{n}{\delta}\right)}}{100 \log\left(\frac{n}{\delta}\right)} = \frac{2k \epsilon \sqrt{C_2} \sqrt{n} \log^5\left(\frac{n}{\delta}\right)}{100 \log\left(\frac{n}{\delta}\right)} = \frac{k \epsilon \sqrt{C_2} \sqrt{n} \log^4\left(\frac{n}{\delta}\right)}{50}.
\]

Thus, the ratio \(\frac{\eta_2}{\eta_0}\) is:
\[
\frac{\eta_2}{\eta_0} = \frac{\frac{4k\epsilon}{3}}{\frac{k \epsilon \sqrt{C_2} \sqrt{n} \log^4\left(\frac{n}{\delta}\right)}{50}} = \frac{4k\epsilon \times 50}{3k \epsilon \sqrt{C_2} \sqrt{n} \log^4\left(\frac{n}{\delta}\right)} = \frac{200}{3 \sqrt{C_2} \sqrt{n} \log^4\left(\frac{n}{\delta}\right)}.
\]

Substituting \(C_2 = \frac{1.28 \times 10^7}{\epsilon^2}\):
\[
\frac{\eta_2}{\eta_0} = \frac{200}{3 \sqrt{\frac{1.28 \times 10^7}{\epsilon^2}} \sqrt{n} \log^4\left(\frac{n}{\delta}\right)} = \frac{200 \epsilon}{3 \times 3580 \sqrt{n} \log^4\left(\frac{n}{\delta}\right)}.
\]

As \(n \to \infty\), \(\sqrt{n} \log^4\left(\frac{n}{\delta}\right) \to \infty\), hence:
\[
\frac{\eta_2}{\eta_0} \to 0.
\]
        
        \item \textbf{Limit of \( (\eta_0 + \eta_1) < 5n \eta_2 \)}:
        \[
            \eta_0 + \eta_1 = 2 \cdot \frac{m \epsilon \sqrt{B}}{100 \log\left(\frac{n}{\delta}\right)} = O\left( \frac{m \epsilon \sqrt{B}}{\log\left(\frac{n}{\delta}\right)} \right).
        \]
        \[
            5n \eta_2 = 5n \cdot \frac{4k\epsilon}{3} = O(n k \epsilon).
        \]
        Given \( \sqrt{B} = O\left( \sqrt{n} \log^{10}(n/\delta) \right) \),
        \[
            \frac{m \epsilon \sqrt{B}}{\log\left(\frac{n}{\delta}\right)} = O\left( \frac{k \epsilon \sqrt{n} \log^{10}(n/\delta)}{\log(n/\delta)} \right) = O\left( k \epsilon \sqrt{n} \log^{9}(n/\delta) \right).
        \]
        Since \( n k \epsilon \) grows faster than \( k \epsilon \sqrt{n} \log^{9}(n/\delta) \) for large \( n \), the inequality \( (\eta_0 + \eta_1) < 5n \eta_2 \) holds.
    \end{enumerate}
This concludes the proof.
\end{proof}

\begin{proof}
We will set our hyperparameters according to~\Cref{assum:order}.
We will first warmup our analysis on linear loss $\ell(\hat y, y) = (-1)^{y}\hat y$.
We can rewrite the training dynamics in this case as 
\begin{align}
    w^{(t + 1)} &= w^{(t)} - \eta_t \frac{1}{B} \sum_{s = 1}^B \sum_{i = n + 1}^{n + k} \nabla_{w} \ell( \transformer[w^{(t)}](\bsequence^{(t,s)})[i], \bsequence^{(t,s)}[i + 1])  \notag \\
   &= w^{(t)} - \eta_t \frac{1}{B} \sum_{s = 1}^B \sum_{i = n + 1}^{n + k} (-1)^{ \bsequence^{(t,s)}[i + 1]} \nabla_{w}\transformer[w^{(t)}](\bsequence^{(t,s)})[i].\label{eq:gd}
\end{align}

We will train the model end to end in three steps and analyze the evolution of all the weight changes. The final results can be shown combining~\Cref{lem:linearFFNoutput2,lem:onehotattn}. \Cref{lem:extensionhinge} generalizes the analysis to hinge loss. It is also easy to verify that all the results hold if the parameters has only $O(\log n)$ precision.    
\end{proof}

\subsubsection{Auxiliary Statistics}
\label{sec:aux}

To simplify our calculations, we define several auxiliary statistics on the data. Table~\ref{tab:stats} provides the definitions and rough orders of these statistics, including logarithmic terms.

\begin{table}[h]
\centering
\begin{tabular}{|l|l|p{3.5cm}|}
\hline
\textbf{Statistic} & \textbf{Definition} & \textbf{Rough Order} \\
\hline
$\delta_{t,i,b}$ & $\displaystyle \sum_{s=1}^B (-1)^{ \bsequence^{(t,s)}[i+1] } \1\left( \bsequence^{(t,s)}[i] = b \right)$ & $\mathcal{O}\left( \sqrt{ B \log\left( \tfrac{ n B }{ \delta } \right) } \right)$ \\[2ex]
\hline
$\alpha_{t,i,j,b_1,b_2}$ & $\displaystyle \sum_{s=1}^B (-1)^{ \bsequence^{(t,s)}[i+1] } \1\left( \bsequence^{(t,s)}[j] = b_1, \bsequence^{(t,s)}[i] = b_2 \right)$ & $\bar{\alpha}_{i,j,b_1,b_2}$+ $\mathcal{O}\left( \sqrt{ B \log\left( \tfrac{ n B }{ \delta } \right) } \right)$ \\[2ex]
\hline
$\bar{\alpha}_{i,j,b_1,b_2}$ & $\displaystyle \begin{cases}
\frac{ (-1)^{b_1 + b_2} \1\left( j = S[i] \right) B }{4}, & i \geq n+2 \\
\frac{ (-1)^{b_1} \1\left( j = S[i],\ b_2 = 0 \right) B }{2}, & i = n+1
\end{cases}$ & $\mathcal{O}\left( B \right)$ \text{when }$j = S[i]$,\quad 0 \text{ otherwise}\\[2ex]
\hline
$\beta_{t,j,b_1}$ & $\displaystyle \sum_{ \substack{ i = \max\{ n+1, j \} \\ S[i] \neq j } }^{n+k} \sum_{b_2=0}^1 \frac{1}{i} \alpha_{t,i,j,b_1,b_2}$ & $\mathcal{O}\left( \frac{ \sqrt{ B k \log\left( \tfrac{ n B }{ \delta } \right) } }{ n } \right)$ \\[4ex]
\hline
$\psi_{t,r,j,b_1}$ & $\displaystyle \sum_{ \substack{ i = \max\{ n+1, j \} \\ S[i] \neq j } }^{n+k} \sum_{b_2=0}^1 \frac{1}{i} \1\left( \nu_{r,i,b_2} > 0 \right) \alpha_{t,i,j,b_1,b_2}$ & $\mathcal{O}\left( \frac{ \sqrt{ B k \log\left( \tfrac{ m n B }{ \delta } \right) } }{ n } \right)$ \\[4ex]
\hline
$\gamma_{i,b,i',b_2}$ & $\displaystyle \begin{aligned}
&\frac{1}{i'} \left( \sum_{r=1}^{2m} \1\left( \nu_{r,i,b} > 0 \right) \1\left( \nu_{r,i',b_2} > 0 \right) \right) \\
&- \frac{1}{i'} \frac{ 1 + \1\left( i = i',\ b = b_2 \right) }{2} m
\end{aligned}$ & $\mathcal{O}\left( \frac{ \sqrt{ m \log\left( \tfrac{ n }{ \delta } \right) } }{ n } \right)$ \\[6ex]
\hline
$\zeta_{t,i,b,j,b_1}$ & $\displaystyle \sum_{r=1}^{2m} \1\left( \nu_{r,i,b} > 0 \right) \psi_{t,r,j,b_1}$ & $\mathcal{O}\left( \frac{ m\sqrt{  k B  \log\left( \tfrac{ mn B }{ \delta } \right)} }{ n } \right)$ \\
\hline
\end{tabular}
\caption{Definitions and rough orders of auxiliary statistics, including logarithmic terms. Here, $t \in \{0,1,2\}$, $i' \in [n +1,n + k + 1]$,$i,j \in [n+k]$, $b,b_1,b_2 \in \{0,1\}$, $s \in [B]$, and $r \in [2m]$.}
\label{tab:stats}
\end{table}

\begin{lemma}
\label{lem:alphamartingale}
Fixing $j, b_1, b_2$, let $\mathcal{F}_{(i, t, s)}$ be the filtration generated by the random variables $\{\bsequence^{(t,s)}[i] \mid \max\{ n+ 1, j\} \leq i, S[i] \neq j , t' \leq t, s \leq B\}$ ordered in lexicographic order, the  process 
    \[
    (-1)^{\bsequence^{(t,s)}[i + 1]} \1(\bsequence^{(t,s)}[j] = b_1, \bsequence^{(t,s)}[i] = b_2)
    \]
    is a martingale with respect to $\{\mathcal{F}_t\}_t$.
\end{lemma}

\begin{proof}
 $\bsequence[i + 1] = \bsequence[i] \oplus \bsequence[S[i]]$, and since $S[i] \neq j$, $\bsequence[S[i]]$ is independent from $\bsequence[j], \ldots, \bsequence[i]$. Thus,
    \[
    \mathbb{E}[(-1)^{\bsequence^{(t,s)}[i + 1]} \mid \mathcal{F}_{(i,t,s)}] = 0,
    \]
    provided $\mathbf{1}(\bsequence^{(t+1,s)}[j] = b_1)$ is constant. Hence, this process is a martingale.
\end{proof}

This version maintains the essential details while focusing on the core aspects of the martingale property and independence conditions.

\begin{lemma}
\label{lem:approxbalance}
For each batch, the data inside the batch is approximately balanced. With probability at least $1 - \delta/10$, the following bounds hold for all $t \in [3]$, $i,j \in [n+k]$, $b,b_1,b_2 \in \{0,1\}$, and $r \in [m]$ for statistics defined in~\Cref{tab:stats}:
\begin{align*}
|\delta_{t,i,b}| &\leq \sqrt{ B \log\left( \tfrac{300 n B}{\delta} \right) } \\
|\alpha_{t,i,j,b_1,b_2} - \bar{\alpha}_{i,j,b_1,b_2}| &\leq \sqrt{ 2 B \log\left( \tfrac{300 n B}{\delta} \right) } \1\left( j = S[i] \right) \\
|\beta_{t,j,b_1}| &\leq \frac{ \sqrt{ 2 B k \log\left( \tfrac{300 n B}{\delta} \right) } }{ n } \\
|\psi_{t,r,j,b_1}| &\leq \frac{ 2 \sqrt{ B k \log\left( \tfrac{300 m n B}{\delta} \right) } }{ n }
\\
\left| \gamma_{i,b,i',b_2} \right| &\leq \frac{ \sqrt{ 2 m \log\left( \tfrac{200 n}{\delta} \right) } }{ n } \\
\left| \zeta_{t,i,b,j,b_1} \right| &\leq \frac{ 4m \sqrt{ k B  \log\left( \tfrac{300 m n B}{\delta} \right)} }{ n }
\end{align*}
\end{lemma}

\begin{proof}
Fix any $t$, $i$, and $b$. For each $s \in [B]$, define $X_s = (-1)^{ \bsequence^{(t,s)}[i+1] } \1\left( \bsequence^{(t,s)}[i] = b \right)$. The sequence $\{ X_s \}_{s=1}^B$ consists of independent random variables with $|X_s| \leq 1$ and zero mean. By the Azuma-Hoeffding inequality,
\[
\Pr\left( \left| \delta_{t,i,b} \right| > R \right) \leq 2 \exp\left( - \frac{ R^2 }{ 2 B } \right).
\]
Setting $R = \sqrt{ 2 B \log\left( \tfrac{300 n B}{\delta} \right) }$, we get
\[
\Pr\left( \left| \delta_{t,i,b} \right| > R \right) \leq \frac{\delta}{150 n B}.
\]
Applying a union bound over all $t$, $i$, and $b$, we ensure that with probability at least $1 - \delta / 50$, the bound on $\delta_{t,i,b}$ holds for all choices.

Similarly, for $\alpha_{t,i,j,b_1,b_2}$, we fix $t$, $i$, $j$, $b_1$, and $b_2$, and define
\[
Y_s = (-1)^{ \bsequence^{(t,s)}[i+1] } \1\left( \bsequence^{(t,s)}[j] = b_1, \bsequence^{(t,s)}[i] = b_2 \right).
\]
Again, $\{ Y_s \}_{s=1}^B$ are independent with $|Y_s| \leq 1$ and mean $\bar{\alpha}_{i,j,b_1,b_2}/B$. Applying the Azuma-Hoeffding inequality as before, we obtain the stated bound for $\alpha_{t,i,j,b_1,b_2}$.

For $\beta_{t,j,b_1}$ and $\psi_{t,r,j,b_1}$,  the bound can be derived similarly combining~\Cref{lem:alphamartingale} and Azuma-Hoeffding bound.

For $\gamma_{i,b,i',b_2}$, consider the sum $\sum_{r=1}^{2m} \1\left( \nu_{r,i,b} > 0 \right) \1\left( \nu_{r,i',b_2} > 0 \right)$. Each term is an independent Bernoulli random variable with mean $\frac{1}{2} \left( 1 + \1\left( i = i',\ b = b_2 \right) \right)$. The variance of the sum is bounded by $m$. Applying the Azuma-Hoeffding inequality, we get
\[
\Pr\left( \left| \gamma_{i,b,i',b_2} \right| > R \right) \leq 2 \exp\left( - \frac{ R^2 n^2 }{ 2 m } \right).
\]
Setting $R = \frac{ \sqrt{ 2 m \log\left( \tfrac{200 n}{\delta} \right) } }{ n }$, we obtain the desired bound.

For $\zeta_{t,i,b,j,b_1}$, we note that it is a sum over $2m$ terms, each involving $\psi_{t,r,j,b_1}$, which are bounded as in Lemma~\ref{lem:approxbalance}. The total number of terms is $2m$, and each term is bounded by $\frac{ 2 \sqrt{ B k \log\left( \tfrac{300 m n B}{\delta} \right) } }{ n }$. 
\end{proof}

\subsubsection{First Step: Configuring the MLPs}

\begin{lemma}
\label{lem:stableattgd}
For the dynamics following~\Cref{eq:gd}, attention weight stays constant in the first step, i.e., $A^{(1)} = A^{(0)} = 0$.
\end{lemma}

\begin{proof}
   Because $W_{:, d+1:2d} = 0$, $\nabla_A \transformer[w^{(0)}](\bsequence)[i] = 0$.
\end{proof}

\begin{lemma}
\label{lem:stableattoutputgd}
For the dynamics following~\Cref{eq:gd}, attention output stays constant in the first step, and 
\begin{align*}
     \forall t \in \{0, 1\}, i \in [n + k + 1], \attention[\qk^{(t)}](\embed(\bsequence))[i] = \frac{1}{i}  \sum_{j = 1}^i  e_{i, \bsequence[i]}.
\end{align*}
\end{lemma}

\begin{proof}
    This follows from the definition of the attention module and~\Cref{lem:stableattgd}.
\end{proof}

\begin{lemma}
\label{lem:neuronspecializesingle}
For an input $\bsequence$, With probability $1 - \frac{\delta}{50Bm}$ , at initialization, whether a neuron $r$ outputs nonzero value at position $i$ is determined by $\nu_{r, i, \bsequence[i]}$. 
\begin{align*}
\1\left( \left \langle W^{0}_{r, 1:d}, \embed(\bsequence)[i] \right \rangle > 0 \right) = \1\left( \nu_{r, i, \bsequence[i]} > 0 \right).
\end{align*}
Further, 
\begin{align*}
\left | \left \langle W^{0}_{r, 1:d}, \embed(\bsequence)[i] \right \rangle  - \nu_{r, i, \bsequence[i]} \right| < \epsilon / 100.
\end{align*}
\end{lemma}

\begin{proof}

Under~\Cref{assum:init}, 
\begin{align*}
W^{0}_{r, 1:d}
% &= \sum_{i= n + 1}^{n + k}  \nu_{r, i, \bsequence[i]}  e_{i, \bsequence[i]}  
&=\sum_{i= n + 1}^{n + k} \sum_{b=0}^1 \nu_{r, i, b}  e_{i, b}  
\end{align*}

Because,
\begin{align*}
    \sum_{i= n + 1}^{n + k} \sum_{b=0}^1\nu_{r, i, b}^2  = 2k \epsilon^2.
\end{align*}

By~\Cref{lem:linearemb}, with probability $1 - \frac{\delta}{50Bm}$,

\begin{align*}
    \left | \left \langle W^{0}_{r, 1:d}, \embed(\bsequence)[i] \right \rangle  - \nu_{r, i, \bsequence[i]} \right| \le \sqrt{\frac{4 \log(\frac{100Bm}{\delta})k}{d}} \epsilon
\end{align*}

By~\Cref{lem:order}.2, $\sqrt{\frac{4k}{d} \log(100Bm/\delta)}<0.01$, we can conclude that
\begin{align*}
     \Pr\left( \left |\sum_{i'= n + 1}^{n + k} \sum_{l = 1}^d  \nu_{r, i', b'} \1( i' \neq i )  e_{i', \bsequence[i']}[l] e_{i, \bsequence[i]}[l] \right| \ge \epsilon/100 \right) \le \frac{\delta}{50Bm}.
\end{align*}
The proof is then complete.
\end{proof}

\begin{lemma}
\label{lem:neuronspecializestep0}
With probability $1 - \frac{\delta}{50}$ , the following $\event_1$ happens: for any input $\bsequence^{(0, s)}$ in the first batch, at initialization, for any $r$, whether a neuron $r$ outputs nonzero value at position $i$ is determined by $\nu_{r, i, \bsequence[i]}$. 
    \begin{align*}
\1\left( \left \langle W^{0}_{r, 1:d}, \embed(\bsequence)[i] \right \rangle > 0 \right) = \1\left( \nu_{r, i, \bsequence[i]} > 0 \right).
\end{align*}
Further, 
\begin{align*}
\left | \left \langle W^{0}_{r, 1:d}, \embed(\bsequence)[i] \right \rangle - \nu_{r, i, \bsequence[i]}  \right| < \epsilon / 100.
\end{align*}

\end{lemma}
\begin{proof}
Apply union bound to~\Cref{lem:neuronspecializesingle} over neuron dimension and data in the first batch.
\end{proof}

\begin{lemma}
\label{lem:gdonembeddingform}
     When $\event_1$ defined in~\Cref{lem:neuronspecializestep0} happens, for all $i' \in [n + 1, n+k+1], b'
\in \{0, 1\}$, the gradient on $W_{r, 1:d}$ satisfies that,
\begin{align*}
\frac{ d \loss^{(0)}}{d W_{r, 1:d}} =  \frac{h_r}{B} \sum_{i = n + 1}^{n + k} \sum_{b = 0}^1  \1( \nu_{r, i, b} > 0) \delta_{0, i, b}  e_{i, b}.
\end{align*}
\end{lemma}

\begin{proof}
\begin{align*}
\frac{ d \loss^{(0)} }{d W_{r, 1:d}}  =&
\frac{1}{B} \sum_{s = 1}^B \sum_{i = n + 1}^{n + k} (-1)^{ \bsequence^{(0,s)}[i + 1]} \nabla_{W_{r, 1:d}}\transformer[w^{(t)}](\bsequence^{(0,s)})[i] \\
=& \frac{h_r}{B} \sum_{s = 1}^B \sum_{i = n + 1}^{n + k} (-1)^{ \bsequence^{(0,s)}[i + 1]} e_{i, \bsequence^{(t,s)}[i]} \1( \nu_{r, i, \bsequence^{(0,s)}[i]} > 0).
\end{align*}
We can then expand this formula to have
\begin{align*}
    &\sum_{s = 1}^B \sum_{i = n + 1}^{n + k} (-1)^{ \bsequence^{(0,s)}[i + 1]}  e_{i, \bsequence^{(0,s)}[i]} \1( \nu_{r, i, \bsequence^{(0,s)}[i]} > 0) \\
    =& \sum_{s = 1}^B \sum_{i = n + 1}^{n + k} \sum_{b = 0}^1  (-1)^{ \bsequence^{(0,s)}[i + 1]} \1( \bsequence^{(0,s)}[i] = b)   e_{i, b}  \1( \nu_{r, i, b} > 0) \\
    =& \sum_{i = n + 1}^{n + k} \sum_{b = 0}^1  e_{i, b}   \1( \nu_{r, i, b} > 0) \sum_{s = 1}^B  (-1)^{ \bsequence^{(0,s)}[i + 1]} \1( \bsequence^{(0,s)}[i] = b) \\
    =& \sum_{i = n + 1}^{n + k} \sum_{b = 0}^1  e_{i, b} \1( \nu_{r, i, b} > 0) \delta_{0, i, b}.
\end{align*}
The proof is then complete.
\end{proof}

\begin{lemma}
\label{lem:gdonembedding}
With probability $1 - 0.13\delta$, for all $i' \in [n + 1, n+k+1], b'
\in \{0, 1\}$, the gradient on $W_{r, 1:d}$ satisfies that,
\begin{align*}
\left |\left \langle  \frac{ d \loss^{(0)} }{d W_{r, 1:d}}, e_{i,b} \right \rangle \right |  \le \frac{\sqrt{\log(300nB/\delta)}}{m \sqrt{B}}.
\end{align*}
\end{lemma}

\begin{proof}

Combining~\Cref{lem:linearemb,lem:gdonembeddingform}, we have that for any $r$, with probability $1 - \delta/100m$,
\begin{align*}
    \left |\left \langle  \frac{h_r}{B} \sum_{i = n + 1}^{n + k} \sum_{b = 0}^1  \1( \nu_{r, i, b} > 0) \delta_{0, i, b}  e_{i, b}, e_{i,b} \right \rangle \right | \le \frac{1}{2mB} \sup \delta_{0, i,b}\left( 1+ \sqrt{\frac{2k \log\left( \frac{200m}{\delta} \right)}{d}} \right)
\end{align*}

By~\Cref{lem:approxbalance,lem:order}, with probability $1 - 11\delta/100$, for all $r$, it holds that for all $r, i', b'$
\begin{align*}
    \left |\left \langle \frac{h_r}{B} \sum_{i = n + 1}^{n + k} \sum_{b = 0}^1  \1( \nu_{r, i, b} > 0) \delta_{0, i, b}  e_{i, b}, e_{i,b}  \right \rangle \right | < \frac{1}{mB} \sup \delta_{0, i, b} < \frac{\sqrt{\log(300nB/\delta)}}{m \sqrt{B}}.
\end{align*}
Combining with~\Cref{lem:gdonembeddingform} and applying union bound concludes the proof.    
\end{proof}

\begin{lemma}
\label{lem:gdonattentionform}
For $\bar \alpha$ and $\psi$ defined in~\Cref{sec:aux},     When $\event_1$ defined in~\Cref{lem:neuronspecializestep0} happens, for all $i' \in [n + 1, n+k+1], b'
\in \{0, 1\}$, the gradient on $W_{r, d+1:2d}$ satisfies that,
\begin{align*}
    \frac{ d \loss^{(0)} }{d W_{r, d + 1:2d}}  =& \frac{h_r}{B} \sum_{i = n + 1}^{n + k} \sum_{b_1 = 0}^1 \sum_{b_2 = 0}^1  \frac{1}{i}   \1( \nu_{r, i, b_2} > 0)  \bar \alpha_{i, S[i], b_1, b_2}  e_{S[i], b_1}  
    +\frac{h_r}{B} \sum_{j = 1}^{n + k} \sum_{b_1 = 0}^1 \psi_{0, r, j, b_1}  e_{j, b_1}. 
\end{align*}
\end{lemma}

\begin{proof}
With probability $1-\delta/50$, $\event_1$ defined in~\Cref{lem:neuronspecializestep0} happens, and
\begin{align*}
\frac{ d \loss^{(0)} }{d W_{r, d + 1:2d}}  =&
\frac{1}{B} \sum_{s = 1}^B \sum_{i = n + 1}^{n + k} (-1)^{ \bsequence^{(0,s)}[i + 1]} \nabla_{W_{r, 1:d}}\transformer[w^{(t)}](\bsequence^{(0,s)})[i] \\
=& \frac{h_r}{B} \sum_{s = 1}^B \sum_{i = n + 1}^{n + k} (-1)^{ \bsequence^{(0,s)}[i + 1]}    \attention[\qk^{(t)}](\embed(\bsequence^{(0,s)} ))[i] \1( \nu_{r, i, \bsequence^{(0,s)}[i]} > 0) \\
=& \frac{h_r}{B} \sum_{s = 1}^B \sum_{i = n + 1}^{n + k}  (-1)^{ \bsequence^{(0,s)}[i + 1]} \1( \nu_{r, i, \bsequence^{(0,s)}[i]} > 0)  \frac{1}{i}  \sum_{j = 1}^i \sum_{b_1 = 0}^1 \1 \left(\bsequence[j] = b \right) e_{j,b_1} 
\end{align*}

Rearranging the summation, and use $\alpha$ defined in~\Cref{lem:approxbalance},,
\begin{align*}
\frac{ d \loss^{(0)} }{d W_{r, d + 1:2d}}  =& \frac{h_r}{B} \sum_{j = 1}^{n + k} \sum_{b_1 = 0}^1  e_{j, b_1}  \sum_{i = \max\{n + 1, j\}}^{n + k} \sum_{b_2 = 0}^1  \frac{1}{i}   \1( \nu_{r, i, b_2} > 0) \alpha_{0, i,j,b_1,b_2}. 
\end{align*}

Breaking the summation,
\begin{align*}
    \frac{ d \loss^{(0)} }{d W_{r, d + 1:2d}}  =& \frac{h_r}{B} \sum_{i = n + 1}^{n + k} \sum_{b_1 = 0}^1 \sum_{b_2 = 0}^1  \frac{1}{i}   \1( \nu_{r, i, b_2} > 0)  \alpha_{0, i, S[i], b_1, b_2}e_{S[i], b_1}  \\
    &+\frac{h_r}{B} \sum_{j = 1}^{n + k} \sum_{b_1 = 0}^1  e_{j, b_1}  \sum_{i \in [\max\{n + 1, j\}. n + k], S[i] \neq j} \sum_{b_2 = 0}^1  \frac{1}{i}   \1( \nu_{r, i, b_2} > 0) \alpha_{0, i,j,b_1,b_2}
\end{align*}

This simplifies to
\begin{align*}
    \frac{ d \loss^{(0)} }{d W_{r, d + 1:2d}}  =& \frac{h_r}{B} \sum_{i = n + 1}^{n + k} \sum_{b_1 = 0}^1 \sum_{b_2 = 0}^1  \frac{1}{i}   \1( \nu_{r, i, b_2} > 0)  \bar \alpha_{i, S[i], b_1, b_2}  e_{S[i], b_1}  
    &+\frac{h_r}{B} \sum_{j = 1}^{n + k} \sum_{b_1 = 0}^1 \psi_{0, r, j, b_1}  e_{j, b_1}. 
\end{align*}
The proof is then complete.
\end{proof}

This leads to the following bound on the gradient,
\begin{lemma}
\label{lem:gdonattention}
With probability $1 - 0.16\delta$, for all $i \in [n+k+1], b
\in \{0, 1\}$, 
the gradient on $W_{r, d + 1: 2d}$ satisfies that,
\begin{align*}
\left |\left \langle  \frac{ d \loss^{(0)} }{d W_{r, d + 1:2d}}, e_{i, b} \right \rangle \right |  \le \frac{3}{2mn}.
\end{align*}    
\end{lemma}

\begin{proof}
    
Denote 
\begin{align*}
 G_r =& \frac{h_r}{B} \sum_{i = n + 1}^{n + k} \sum_{b_1 = 0}^1 \sum_{b_2 = 0}^1  \frac{1}{i}   \1( \nu_{r, i, b_2} > 0)  \bar \alpha_{i, S[i], b_1, b_2}  e_{S[i], b_1}  
    +\frac{h_r}{B} \sum_{j = 1}^{n + k} \sum_{b_1 = 0}^1 \psi_{0, r, j, b_1}  e_{j, b_1}. 
\end{align*}

By~\Cref{lem:linearemb}, for all $r \in [2m], i \in [n+k+1], b
\in \{0, 1\}$, with probability $1 - \delta/25$,

\begin{align*}
   &\left|  \left \langle G_r, e_{i, b} \right \rangle  \right| \le \frac{1}{2mB}\left | \sum_{i' = n + 1}^{n + k} \sum_{b_2 = 0}^1 \frac{1}{i}   \1( \nu_{r, i, b_2} > 0)  \bar \alpha_{i, S[i], b, b_2} \1(S[i'] = i)  + \psi_{0, r, i, b} \right| \\
   &+ \frac{\sqrt{2\log(\frac{400mn}{\delta})}}{2mB \sqrt{d}} \left( \sqrt{\sum_{i = n + 1}^{n + k} \sum_{b_1 = 0}^1 \sum_{b_2 = 0}^1  \frac{1}{i^2}   \1( \nu_{r, i, b_2} > 0)  \bar \alpha_{i, S[i], b_1, b_2}^2} 
   + \sqrt{\sum_{j = 1}^{n + k} \sum_{b_1 = 0}^1 \psi_{0, r, j, b_1}^2} \right)
\end{align*}

By~\Cref{lem:approxbalance}, we have with probability $1 - \delta/10$,
\begin{align*}
\sup \psi_{0, r, i, b} \le  \frac{ 2 \sqrt{ B k \log\left( \tfrac{300 m n B}{\delta} \right) } }{ n },\quad
\sup \bar \alpha_{i, S[i], b_1, b_2} \le B/2.
\end{align*}

Combining with~\Cref{lem:order},
\begin{align*}
    &\frac{1}{2mB}\left | \sum_{i' = n + 1}^{n + k} \sum_{b_2 = 0}^1 \frac{1}{i}   \1( \nu_{r, i, b_2} > 0)  \bar \alpha_{i, S[i], b, b_2} \1(S[i'] = i)  \right| \le \frac{1}{2mn} \\
    & \frac{1}{2mB} \left| \psi_{0, r, i, b} \right| \le \frac{2\sqrt{ k \log\left( \tfrac{300 m n B}{\delta} \right) }}{mn\sqrt{B}} \le \frac{1}{200mn} \\
    & \frac{\sqrt{2\log(\frac{400mn}{\delta})}}{2mB \sqrt{d}} \sqrt{\sum_{i = n + 1}^{n + k} \sum_{b_1 = 0}^1 \sum_{b_2 = 0}^1  \frac{1}{i^2}   \1( \nu_{r, i, b_2} > 0)  \bar \alpha_{i, S[i], b_1, b_2}^2}  \le \frac{\sqrt{2k\log(\frac{400mn}{\delta})}}{2mn \sqrt{d}} \le \frac{1}{200mn},\\
    &\frac{\sqrt{2\log(\frac{400mn}{\delta})}}{2mB \sqrt{d}}  \sqrt{\sum_{j = 1}^{n + k} \sum_{b_1 = 0}^1 \psi_{0, r, j, b_1}^2} \le  \frac{2\sqrt{2nk}{\log(\frac{300mnB}{\delta})}}{mn\sqrt{Bd}} \le \frac{1}{200mn}.
\end{align*}
Further by~\Cref{lem:gdonattentionform}, we have that $\frac{ d \loss^{(0)} }{d W_{r, d + 1:2d}} = G_r$ with probability $1 - \delta/50$. This concludes the proof.
\end{proof}

\subsubsection{Second Step: Configuring the Attention}

\begin{lemma}
\label{lem:neuronspecializestep1}
    With probability $1 - 0.29{\delta}$ , the following $\event_2$ happens: the change in the neuron outcome is small after the first step for any input $\bsequence$,
    \begin{align*}
      \sup_r \left | \left \langle W^{1}_{r} -  W^{0}_{r} , [\embed(\bsequence)[i], \attention[\qk^1](\embed(\bsequence))[i]] \right \rangle  \right|  < 3\epsilon/80.
\end{align*}
    This leads to the case that for any input $\bsequence^{(1, s)}$ in the second batch, for any $r$, whether a neuron $r$ outputs nonzero value at position $i$ is determined by $\nu_{r, i, \bsequence[i]}$. 
    \begin{align*}
\1\left( \left \langle W^{1}_{r, 1:d}, \embed(\bsequence)[i] \right \rangle > 0 \right) = \1\left( \nu_{r, i, \bsequence[i]} > 0 \right).
\end{align*}
\end{lemma}

\begin{proof}
By~\Cref{lem:gdonembedding}, with probability $1 - 0.13\delta$,
    \begin{align*}
     \sup_r | \left \langle W^{1}_{r, 1:d} - W^{0}_{r, 1:d}, \embed(\bsequence)[i] \right \rangle | = \eta_0 \left| \left \langle \frac{d \loss^{(0)} }{d W_{r,1:d}}, \embed(\bsequence)[i] \right \rangle \right| \le \eta_0 \frac{\sqrt{\log(300nB/\delta)}}{m \sqrt{B}}.
    \end{align*}

By~\Cref{lem:gdonattention},with probability $1 - 0.16\delta$,
    \begin{align*}
     \sup_r | \left \langle W^{1}_{r, d+1:2d} - W^{0}_{r, d+1:2d}, \attention[\qk^1](\embed(\bsequence))[i] \right \rangle | = \eta_0 \left | \left \langle \frac{d \loss^{(0)} }{d W_{r, d + 1:2d}},\attention[\qk^1](\embed(\bsequence))[i]  \right \rangle \right| \le \eta_0 \frac{3}{2mn}.
    \end{align*}

Hence, as $\eta_0 \le \min\{ \frac{3m \epsilon \sqrt{B} }{80\sqrt{\log(300nB/\delta)}}, \frac{mn\epsilon}{120} \}$ (\Cref{lem:order}.3), we conclude that
\begin{align*}
      \sup_r \left | \left \langle W^{1}_{r} -  W^{0}_{r} , [\embed(\bsequence)[i], \attention[\qk^1](\embed(\bsequence))[i] \right \rangle  \right|  < 3\epsilon/80.
\end{align*}
The proof is then complete.
\end{proof}

This shows that the gradient of MLP on the second and first steps is the same in distribution. 
\begin{lemma}
\label{lem:gdform}
With probability $1 - 0.58 \delta$, inequalities in~\Cref{lem:approxbalance} hold and for $t \in \{0,1\}, \forall i \in [n + k], \forall b \in \{0,1\}$,
\begin{align*}
\frac{ d \loss^{(t)}}{d W_{r, 1:d}} =&  \frac{h_r}{B} \sum_{i = n + 1}^{n + k} \sum_{b = 0}^1  \1( \nu_{r, i, b} > 0) \delta_{t, i, b}  e_{i, b}. \\
    \frac{ d \loss^{(t)} }{d W_{r, d + 1:2d}}  =& \frac{h_r}{B} \sum_{i = n + 1}^{n + k} \sum_{b_1 = 0}^1 \sum_{b_2 = 0}^1  \frac{1}{i}   \1( \nu_{r, i, b_2} > 0)  \bar \alpha_{i, S[i], b_1, b_2}  e_{S[i], b_1}  
    +\frac{h_r}{B} \sum_{j = 1}^{n + k} \sum_{b_1 = 0}^1 \psi_{t, r, j, b_1}  e_{j, b_1}.  \\
\left |\left \langle  \frac{ d \loss^{(t)} }{d W_{r, d + 1:2d}}, e_{i, b} \right \rangle \right |  \le& \frac{3}{2mn}.
\end{align*}
Moreover for any input $\bsequence$,
    \begin{align*}
      \sup_r \left | \left \langle W^{t}_{r} -  W^{0}_{r} , [\embed(\bsequence)[i], \attention[\qk^1](\embed(\bsequence))[i] \right \rangle  \right|  < 3\epsilon/40.
\end{align*}
Further,
    \begin{align*}
\1\left( \left \langle W^{t}_{r, 1:d}, \embed(\bsequence^{(t,s)})[i] \right \rangle > 0 \right) = \1\left( \nu_{r, i, \bsequence^{(t,s)}[i]} > 0 \right).
\end{align*}
\end{lemma}

\begin{proof}
    Consider the backward calculation of MLP, it is decided by the input and the corresponding activation of neurons, which jointly follows the same distribution in the first and second steps. The proof is similar to~\Cref{lem:gdonattentionform,lem:gdonembeddingform,lem:neuronspecializestep1}.
\end{proof}

We will now show that while MLP changes little, it is enough to provide high-quality gradient information to inform the attention layer. To this end, we will define the following terms,
\begin{align*}
    \kappa_{t, i, b} &= -\sum_{r = 1}^{2m}  \1 \left( \nu_{r, i, b}  > 0 \right) h_r \weight^{(t)}_{r, d+1:2d} \\
    \Delta_{t, i, b, j,b_1} &= \langle \kappa_{t, i,b}, e_{j,b_1} \rangle
\end{align*}
Intuitively, the FFN with weight $W^{(1)}$ maps attention output $x$ to $-\kappa_{i,\bsequence[i]} x$ at position $i$.

\begin{lemma}
\label{lem:attnsignal}
For $\alpha, \beta,\gamma, \zeta$ defined in~\Cref{sec:aux},
when the event in~\Cref{lem:gdform} happens, for any $i$ and $b$,
\begin{align*}
    \kappa_{1, i,b} =& \frac{\eta_0}{8Bmi}  \sum_{b_1 = 0}^1  \bar \alpha_{i, S[i], b_1, b}  e_{S[i], b_1}   
    + \frac{\eta_0}{4Bm^2}\left(\sum_{i' = n + 1}^{n + k} \sum_{b_1 = 0}^1 \sum_{b_2 = 0}^1 \gamma_{i,b,i',b_2} \bar \alpha_{i', S[i'], b_1, b_2} e_{S[i'], b_1}\right)  \\
    & + \frac{\eta_0}{4Bm^2}   \left( \sum_{j = 1}^{n + k} \sum_{b_1 = 0}^1  \zeta_{0, i,b,j,b_1} e_{j, b_1} \right).
\end{align*}
\end{lemma}

\begin{proof}

When the event in~\Cref{lem:gdform} happens, we have that
\begin{align*}
    \frac{ d \loss^{(0)} }{d W_{r, d + 1:2d}}  = \frac{h_r}{B} \sum_{i' = n + 1}^{n + k} \sum_{b_1 = 0}^1 \sum_{b_2 = 0}^1  \frac{1}{i'}   \1( \nu_{r, i', b_2} > 0)  \bar \alpha_{i', S[i'], b_1, b_2}  e_{S[i'], b_1}  
    &+\frac{h_r}{B} \sum_{j = 1}^{n + k} \sum_{b_1 = 0}^1 \psi_{0, r, j, b_1}  e_{j, b_1}. 
\end{align*}

Summing over the axis of $r$ as in $\kappa$,
\begin{align*}
    \kappa_{1,i, b} =& \frac{\eta_0}{4Bm^2}  \left( \sum_{i' = n + 1}^{n + k} \sum_{b_1 = 0}^1 \sum_{b_2 = 0}^1 \left(\sum_{r = 1}^{2m}  \1 \left( \nu_{r, i, b}  > 0 \right)     \1( \nu_{r, i', b_2} > 0) \right) \frac{1}{i'} \bar \alpha_{i', S[i'], b_1, b_2}  e_{S[i'], b_1}   \right)\\
    &+ \frac{\eta_0}{4Bm^2}  \left( \sum_{j = 1}^{n + k} \sum_{b_1 = 0}^1  e_{j, b_1}  \left( \sum_{r = 1}^{2m} \1 \left( \nu_{r, i, b}  > 0 \right) \psi_{0, r, j, b_1}  \right) \right).
\end{align*}

Recall that 
$\gamma_{ i, b,i', b_2} = \frac{1}{i'} \left( \sum_{r=1}^{2m} \1\left( \nu_{r,i,b} > 0 \right) \1\left( \nu_{r,i',b_2} > 0 \right) \right) - \frac{1}{i'} \frac{ 1 + \1\left( i = i',\ b = b_2 \right) }{2} m$. We then have

\begin{align*}
    \kappa_{1,i, b} =& \frac{\eta_0}{4Bm}  \left( \sum_{i' = n + 1}^{n + k} \sum_{b_1 = 0}^1 \sum_{b_2 = 0}^1 \frac{1}{i'} \frac{ 1 + \1\left( i = i',\ b = b_2 \right) }{2}  \bar \alpha_{i', S[i'], b_1, b_2}  e_{S[i'], b_1}   \right)\\
    &+ \frac{\eta_0}{4Bm^2}  \left( \sum_{i' = n + 1}^{n + k} \sum_{b_1 = 0}^1 \sum_{b_2 = 0}^1 \gamma_{ i, b,i', b_2} \bar \alpha_{i', S[i'], b_1, b_2}  e_{S[i'], b_1}   \right)\\
    &+ \frac{\eta_0}{4Bm^2}  \left( \sum_{j = 1}^{n + k} \sum_{b_1 = 0}^1  e_{j, b_1}  \left( \sum_{r = 1}^{2m} \1 \left( \nu_{r, i, b}  > 0 \right) \psi_{0, r, j, b_1}  \right) \right).
\end{align*}
The first term can be greatly simplified, as 
\begin{align*}
    \alpha_{i', S[i'], b_1, 0} + \alpha_{i', S[i'], b_1, 1} = 0.
\end{align*}
This concludes the proof.
\end{proof}

\begin{lemma}
\label{lem:signalcorrelation}

With probability $1 - 0.6\delta$, inequalities and equalities in~\Cref{lem:approxbalance,lem:gdform,lem:attnsignal} hold, and for all $i, b_2, j, b_1$,
\begin{align*}
    \left| \Delta_{1, i, b_2, j,b_1} - \frac{\eta_0}{8Bmi} \bar \alpha_{i, S[i], b_1, b_2} \1(j = S[i])  \right| \le \frac{10\eta_0}{mn}  \frac{\sqrt{nk}}{\sqrt{Bd}}  \log(300mnB/\delta).
\end{align*}
\end{lemma}

\begin{proof}

    Define 
    \begin{align*}
        H_{i,b} =& \frac{\eta_0}{8Bmi}  \sum_{b_1 = 0}^1  \bar \alpha_{i, S[i], b_1, b}  e_{S[i], b_1}   \\
        &+ \frac{\eta_0}{4Bm^2}\left(\sum_{i' = n + 1}^{n + k} \sum_{b_1 = 0}^1 \sum_{b_2 = 0}^1 \gamma_{i,b,i',b_2} \bar \alpha_{i', S[i'], b_1, b_2}e_{S[i'], b_1}\right)  \\
        & + \frac{\eta_0}{4Bm^2}   \left( \sum_{j = 1}^{n + k} \sum_{b_1 = 0}^1  \zeta_{i,b,j,b_1} e_{j, b_1} \right)
    \end{align*}

As the inequalities in~\Cref{lem:approxbalance} hold, by~\Cref{lem:linearemb,lem:order}, with probability $1 - \delta/50$, it holds that,
\begin{align*}
   &\left| \frac{\eta_0}{4Bm^2}  \left \langle \left(\sum_{i' = n + 1}^{n + k} \sum_{b_1 = 0}^1 \sum_{b_2 = 0}^1 \gamma_{i,b,i',b_2} \bar \alpha_{i', S[i'], b_1, b_2}e_{S[i'], b_1}\right) , e_{j, b_1} \right \rangle \right | \\\le& \frac{\eta_0 \sup \bar \alpha \sup \gamma }{4Bm^2} \left(1 + {4\sqrt{\frac{k \log(200n/\delta)}{d}}} \right) \\
   \le& \frac{\eta_0}{4mn} \sqrt{\frac{2k\log\left(\frac{200n}{\delta}\right)}{{md}}}.
\end{align*}
and,
\begin{align*}
   &\left| \frac{\eta_0}{4Bm^2}  \left \langle   \left( \sum_{j = 1}^{n + k} \sum_{b_1 = 0}^1  \zeta_{i,b,j,b_1} e_{j, b_1} \right) , e_{j, b_1} \right \rangle \right | \\\le& \frac{\eta_0 \sup \zeta }{4Bm^2} \left(1 + {4\sqrt{\frac{n \log(200n/\delta)}{d}}} \right) \\
   \le& \frac{\eta_0}{mn}  \left(\sqrt{\frac{k}{B}\log(200n/\delta)} + {4\sqrt{\frac{nk}{dB}}}\log(200n/\delta)\right).
\end{align*}
and,
\begin{align*}
    \left| \frac{\eta_0}{8Bmi} \left \langle   \left( \sum_{b_1 = 0}^1  \bar \alpha_{i, S[i], b_1, b}  e_{S[i], b_1}\right),e_{j, b_1} \right \rangle - \frac{\eta_0}{8Bmi} \bar \alpha_{i, S[i], b_1, b} \1(j = S[i]) \right| \le& \frac{2\sqrt{\log(200n/\delta)}}{8mn \sqrt{d}}.
\end{align*}
With $d < n$ and $B < nk$ (\Cref{lem:order}), we conclude that,
\begin{align*}
      \left| \langle H_{i,b}, e_{j, b_1} \rangle - \frac{\eta_0}{8Bmi} \bar \alpha_{i, S[i], b_1, b} \1(j = S[i])  \right| \le \frac{10\eta_0}{mn}  \frac{\sqrt{nk}}{\sqrt{Bd}}  \log(300mnB/\delta).
\end{align*}
By~\Cref{lem:attnsignal}, we have the desired result.
\end{proof}

\begin{lemma}
\label{lem:gradientonA}
Under the setting of~\Cref{lem:signalcorrelation}, for every $\bsequence \in \{ \bsequence^{(1,s)} \mid s \in [B] \}$,
$$\frac{d \ell( \transformer(\bsequence)[i], \bsequence[i + 1])  }{d A} = (-1)^{\bsequence[i + 1] + 1}  \sum_{j = 1}^i \langle \kappa_{1, i, \bsequence[i]}, e_{j, \bsequence[j]} \rangle \frac{  \partial \softmax \left(\embed(\bsequence)^T A \embed(\bsequence) \right)[j, i]}{\partial A},$$
with
\[
\frac{\partial \softmax \left(\embed(\bsequence)^T A \embed(\bsequence) \right)[j, i]}{\partial A} = \softmax(Z)[j, i] \cdot \left( \embed(\bsequence)[j] - \sum_{p=1}^i \softmax(Z)[p, i] \embed(\bsequence)[p] \right) \embed(\bsequence)[i]^T
\]
\end{lemma}

\begin{proof}
    
Our goal is to calculate the gradient of the loss with respect to the attention layer. We will first calculate the gradient of the loss with respect to the output of the attention layer, and then use the chain rule. When the event in~\Cref{lem:gdform} happens, for any binary sequence $\bsequence$ in the second batch, we have that for any $i$,
\begin{align*}
    \frac{ d \transformer(\bsequence)[i]  }{d \attention(\embed(\bsequence))[i]} &=
    \sum_{r = 1}^{2m} h_r \weight^{(1)}_{r} \1 \left( \left( W^{(1)} \begin{bmatrix} 
    \embed\left( \bsequence \right) \\
    \attention\left(\embed\left( \bsequence \right) \right) \end{bmatrix} \right)\left[i\right]_r  > 0 \right) \\
    &= \sum_{r = 1}^{2m} h_r \weight^{(1)}_{r,d+1:2d} \1 \left( \nu_{r, i, \bsequence[i]} > 0 \right) 
\end{align*}

This then implies that,
\begin{align*}
    \frac{ d \ell( \transformer(\bsequence)[i], \bsequence[i + 1])  }{d \attention(\embed(\bsequence))[i]} 
    &= (-1)^{\bsequence[i + 1]} \left(\sum_{r = 1}^{2m} h_r  \weight^{(1)}_{r,d+1:2d}  \1 \left( \nu_{r, i, \bsequence[i]}  > 0 \right) \right) 
\end{align*}

We now calculate the gradient of $A$ (the attention matrix) to the attention output.
\begin{align*}
    \frac{\partial \attention(\embed(\bsequence))[i]}{\partial A}  &=  \frac{\partial \left( \embed(\bsequence)  \softmax \left(\embed(\bsequence)^T A \embed(\bsequence) \right) \right)[i]}{\partial A} \\
     &= \frac{\sum_{j = 1}^i \partial \embed(\bsequence)[j] \softmax \left(\embed(\bsequence)^T A \embed(\bsequence) \right)[j,i] }{\partial A} \\
     &= \sum_{j = 1}^i  \embed(\bsequence)[j] \frac{  \partial \softmax \left(\embed(\bsequence)^T A \embed(\bsequence) \right)[j, i]}{\partial A} 
\end{align*}

Hence, the gradient of the loss with respect to the attention matrix is
\begin{align*}
    &\frac{d \ell( \transformer(\bsequence)[i], \bsequence[i + 1])  }{d A} \\=& \sum_{j = 1}^i \frac{d \ell( \transformer(\bsequence)[i], \bsequence[i + 1])  }{d \attention(\embed(\bsequence))[i]} \embed(\bsequence)[j] \frac{  \partial \softmax \left(\embed(\bsequence)^T A \embed(\bsequence) \right)[j, i]}{\partial A} \\
=&  (-1)^{\bsequence[i + 1]}  \sum_{j = 1}^i \sum_{r = 1}^{2m} h_r    \1 \left( \nu_{r, i, \bsequence[i]}  > 0 \right) \langle \weight^{(1)}_{r,d+1:2d}, \embed(\bsequence)[j] \rangle \frac{  \partial \softmax \left(\embed(\bsequence)^T A \embed(\bsequence) \right)[j, i]}{\partial A} \\
=& (-1)^{\bsequence[i + 1] + 1}  \sum_{j = 1}^i \langle \kappa_{1, i, \bsequence[i]}, e_{j, \bsequence[j]} \rangle \frac{  \partial \softmax \left(\embed(\bsequence)^T A \embed(\bsequence) \right)[j, i]}{\partial A}
\end{align*}

We will use $Z$ to denote $\embed(\bsequence)^T A \embed(\bsequence)$, and calculate the derivative of th+e softmax function applied to $Z$ with respect to $A$.

\begin{align*}
    \frac{\partial \softmax(Z)[j, i]}{\partial A} &= \frac{\partial}{\partial A} \left( \frac{e^{Z[j, i]}}{\sum_{p=1}^i e^{Z[p, i]}} \right) \\
    &= \frac{e^{Z[j, i]} \cdot \frac{\partial Z[j, i]}{\partial A} \cdot \sum_{p=1}^i e^{Z[p, i]} - e^{Z[j, i]} \cdot \sum_{p=1}^i e^{Z[p, i]} \cdot \frac{\partial Z[p, i]}{\partial A}}{\left( \sum_{p=1}^i e^{Z[p, i]} \right)^2}
\end{align*}

Given \( Z[p, i] = \embed(\bsequence)[p]^T A \embed(\bsequence)[i] \), the derivative with respect to \( A \) is:

\[
\frac{\partial Z[p, i]}{\partial A} = \embed(\bsequence)[p] \cdot \embed(\bsequence)[i]^T
\]

Substituting the derivatives into the quotient rule expression:

\[
\frac{\partial \softmax(Z)[j, i]}{\partial A} = \frac{e^{Z[j, i]} \cdot \embed(\bsequence)[j] \embed(\bsequence)[i]^T \cdot \sum_{p=1}^i e^{Z[p, i]} - e^{Z[j, i]} \cdot \sum_{p=1}^i e^{Z[p, i]} \cdot \embed(\bsequence)[p] \embed(\bsequence)[i]^T}{\left( \sum_{p=1}^i e^{Z[p, i]} \right)^2}
\]

Thus, the derivative is:

\[
\frac{\partial \softmax \left(\embed(\bsequence)^T A \embed(\bsequence) \right)[j, i]}{\partial A} = \softmax(Z)[j, i] \cdot \left( \embed(\bsequence)[j] - \sum_{p=1}^i \softmax(Z)[p, i] \embed(\bsequence)[p] \right) \embed(\bsequence)[i]^T
\]
The proof is then complete.
\end{proof}

\begin{lemma}
\label{lem:onehotattn}
After two training steps, the attention layer will show the following structure.
With probability $1 - 0.7\delta$, for all $i, b, i', b_2$,
\begin{enumerate}
    \item If $j' = S[i']$, then $\left \langle e_{j', b_1'}, A^{(2)} e_{i', b_2'} \right \rangle \ge \frac{\eta_0\eta_1}{256mn^2} $.
    \item If $j' \neq S[i']$, then $|\left \langle e_{j', b_1'}, A^{(2)} e_{i', b_2'} \right \rangle | \le \frac{\eta_0\eta_1}{512mn^2}$.
\end{enumerate}
Further, events in~\Cref{lem:attnsignal,lem:gdform} hold.
\end{lemma}

\begin{proof}

By~\Cref{lem:gradientonA},
\begin{align*}
    \frac{\partial \softmax \left(\embed(\bsequence)^T A \embed(\bsequence) \right)[j, i]}{\partial A} = \frac{1}{i} \left( \embed(\bsequence)[j] - \sum_{p=1}^i \frac{1}{i} \embed(\bsequence)[p] \right) \embed(\bsequence)[i]^T
\end{align*}

We can rewrite the gradient as
\begin{align*}
    \frac{d \ell( \transformer(\bsequence)[i], \bsequence[i + 1])  }{d A} =& -\frac{1}{i} (-1)^{\bsequence[i + 1]}  \left( \sum_{j = 1}^i \langle \kappa_{i,  \bsequence[i]}, \embed(\bsequence)[j] \rangle \embed(\bsequence)[j] \right) \embed(\bsequence)[i]^T \\
    &- \frac{1}{i^2} \sum_{j = 1}^i \sum_{p=1}^i (-1)^{\bsequence[i + 1]}   \langle \kappa_{i,  \bsequence[i]}, \embed(\bsequence)[j] \rangle \embed(\bsequence)[p] \embed(\bsequence)[i]^T
\end{align*}

If we define
\begin{align*}
\mu_{t, j, i, p, b_1, b_2,b_3} = \left(  \sum_{s = 1}^B (-1)^{\bsequence^{(t,s)}[i + 1]} \1(\bsequence^{(t,s)}[j] = b_1, \bsequence^{(t,s)}[i] = b_2, \bsequence^{(t,s)}[p]=b_3) \right).
\end{align*}

Summing over the second batch,  we have that
\begin{align*}
    -\frac{d L^{(1)}  }{d A} =&  \frac{1}{B} \sum_{i = n+1}^{n+k} \sum_{j = 1}^i  \sum_{b_1 = 0}^1 \sum_{b_2 = 0}^1 \frac{1}{i} \alpha_{1, i, j, b_1, b_2} \Delta_{1,i,j,b_1,b_2} e_{j,b_1} e_{i,b_2}^T \\
    &-\frac{1}{B}\sum_{i = n+1}^{n+k} \sum_{j = 1}^i  \sum_{p=1}^i  \sum_{b_1 = 0}^1 \sum_{b_2 = 1}^1 \sum_{b_3 = 0}^1 \frac{1}{i^2} \mu_{1, j, i, p, b_1, b_2,b_3} \Delta_{1,i,j,b_1,b_2} e_{p,b_3} e_{i,b_2}^T.
\end{align*}

This implies $A^{(2)}$ is updated as,
\begin{align*}
    A^{(2)} &= -\eta_1\frac{d L^{(1)}  }{d A} \\
    &= \frac{\eta_1}{B}\sum_{i = n + 1}^{n + k}\sum_{j = 1}^i  \sum_{b_1 = 0}^1 \sum_{b_2 = 0}^1 \frac{1}{i} \alpha_{1, i, j, b_1, b_2} \Delta_{1,i,j,b_1,b_2} e_{j,b_1} e_{i,b_2}^T \\
    &-\frac{\eta_1}{B}\sum_{i = n + 1}^{n + k} \sum_{j = 1}^i  \sum_{p=1}^i  \sum_{b_1 = 0}^1 \sum_{b_2 = 1}^1 \sum_{b_3 = 0}^1 \frac{1}{i^2} \mu_{1, j, i, p, b_1, b_2,b_3} \Delta_{1,i,j,b_1,b_2} e_{p,b_3} e_{i,b_2}^T.
\end{align*}

Recall that our goal is to calculate $e_{j', b_1'}^T A^{(2)} e_{i', b_2'}$. We can then calculate the contribution from each term by separating the calculation as follows:

\begin{align*}
    T_j &= \frac{\eta_1}{B}\sum_{i \in [\max\{n + 1,j \}, n + k], b_2 \in \{0, 1\}, S[i] \neq j} \sum_{b_1 = 0}^1 \frac{1}{i} \alpha_{1, i, j, b_1, b_2} \Delta_{1,i,j,b_1,b_2} \langle e_{j', b_1'}, e_{j,b_1} \rangle \langle e_{i,b_2}, e_{i', b_2'} \rangle. \\
    R &=  \frac{\eta_1}{B} \sum_{i = n + 1}^{n + k} \sum_{b_1 = 0}^1 \sum_{b_2 = 0}^1 \frac{1}{i} \alpha_{1, i, S[i], b_1, b_2} \Delta_{1,i,S[i],b_1,b_2} \langle e_{j', b_1'}, e_{S[i],b_1} \rangle\langle e_{i,b_2}, e_{i', b_2'} \rangle.  \\
    U_{j, p} &= \frac{\eta_1}{B}\sum_{i = \max\{n +1, j, p \}, S[i] \not \in \{j, p\} }^{n + k} \sum_{b_1 = 0}^1 \sum_{b_2 = 0}^1  \sum_{b_3 = 0}^1 \frac{1}{i^2} \mu_{1, j, i, p, b_1, b_2,b_3} \Delta_{1,i,j,b_1,b_2} \langle e_{j', b_1'}, e_{j,b_1} \rangle \langle e_{p,b_3}, e_{i',b_2'} \rangle \\
    V_{p} &= \frac{\eta_1}{B} \sum_{i = \max\{n + 1, p\}}^{n + k} \sum_{b_1 = 0}^1 \sum_{b_2 = 0}^1 \frac{1}{i^2} \mu_{1, S[i], i, p, b_1, b_2,b_3} \Delta_{1,i,S[i],b_1,b_2} \langle e_{j', b_1'}, e_{S[i],b_1} \rangle \langle e_{p,b_3}, e_{i',b_2'} \rangle \\
    W_{j} &=  \frac{\eta_1}{B}\sum_{i =\max\{n + 1,j\}}^{n + k} \sum_{b_1 = 0}^1 \sum_{b_2 = 0}^1 \sum_{b_3 = 0}^1 \frac{1}{i^2} \mu_{1, j, i, S[i], b_1, b_2,b_3} \Delta_{1,i,j,b_1,b_2} \langle e_{j', b_1'}, e_{j,b_1} \rangle \langle e_{p,b_3}, e_{i',b_2'} \rangle . \\
    Y &= \frac{\eta_1}{B}\sum_{i = n + 1}^{n + k} \sum_{b_1 = 0}^1 \sum_{b_2 = 0}^1 \sum_{b_3 = 0}^1 \frac{1}{i^2} \mu_{1, S[i], i, S[i], b_1, b_2,b_3} \Delta_{1,i,S[i],b_1,b_2} \langle e_{j', b_1'}, e_{S[i],b_1} \rangle \langle e_{S[i],b_3}, e_{i',b_2'} \rangle.
\end{align*}

Then we have that
\begin{align*}
    e_{j', b_1'}^T A^{(2)} e_{i', b_2'} &= \sum_{j = 1}^{n + k } T_j + R  - \sum_{j = 1}^{n + k}\sum_{p = 1}^{n + k}  U_{j, p} - \sum_{j = 1}^{n + k} W_j - \sum_{p = 1}^{n +k} V_p  + Y.
\end{align*}

.By~\Cref{lem:arrow}, with probability $1 - \delta/100$,
\begin{align*}
    \left| \langle e_{i,b_2}, e_{i', b_2'} \rangle - \1(i = i' \& b_2 = b_2') \right| \le 2\sqrt{\frac{\log(100nd/\delta)}{d}}.
\end{align*}
We will also assume the event in~\Cref{lem:signalcorrelation} holds.

We will now discuss each term separately
\begin{enumerate}[leftmargin=*]
\item For $T_j$, by~\Cref{lem:alphamartingale} and Azuma-Hoeffding bound, with probability $1 - \delta/50$,
\begin{align*}
     &\left|\sum_{i \in [\max\{n + 1,j \}, n + k], b_2 \in \{0, 1\}, S[i] \neq j}  \sum_{b_2 = 0}^1 \frac{1}{i} \alpha_{1, i, j, b_1, b_2} \Delta_{1,i,j,b_1,b_2} \langle e_{j', b_1'}, e_{j,b_1} \rangle \langle e_{i,b_2}, e_{i', b_2'} \rangle \right|  \\
    \le& 2 \sum_{b_1 = 0}^1 \sqrt{B \log(\frac{100n}{\delta}) \sum_{i \in [\max\{n + 1,j \}, n + k], b_2 \in \{0, 1\}, S[i] \neq j}  \frac{1}{i^2}  \Delta_{1,1,i,j,b_1,b_2}^2 \left(\langle e_{j', b_1'}, e_{j,b_1} \rangle\right)^2 \left( \langle e_{i,b_2}, e_{i', b_2'} \rangle \right)^2 } \\
    \le& \frac{2\sqrt{B\log(\frac{100n}{\delta})}}{n} \sup_{j \neq S[i]}  |\Delta_{1, i, j, b_1, b_2}| \sum_{b_1 = 0}^1 \left|\langle e_{j', b_1'}, e_{j,b_1} \rangle \right| \sqrt{ \sum_{i \in [\max\{n + 1,j \}, n + k], b_2 \in \{0, 1\}, S[i] \neq j}  \left( \langle e_{i,b_2}, e_{i', b_2'} \rangle \right)^2 } 
\end{align*}

Now by~\Cref{lem:signalcorrelation}, we have that
\begin{align*}
    \sup_{j \neq S[i]}  |\Delta_{1, i, j, b_1, b_2}| \le \frac{10\eta_0}{mn}  \frac{\sqrt{nk}}{\sqrt{Bd}}  \log(300mnB/\delta).
\end{align*}

Further by~\Cref{lem:arrow}, 
\begin{align*}
    \sum_{i \in [\max\{n + 1,j \}, n + k], b_2 \in \{0, 1\}, S[i] \neq j}  \left( \langle e_{i,b_2}, e_{i', b_2'} \rangle \right)^2  \le& \frac{8k \log(100nd / \delta)}{d} + 1 \le 4. \\
    \left|\langle e_{j', b_1'}, e_{j,b_1} \rangle \right| \le& \1(j = j' \& b_1' = b_1) + 2\sqrt{\frac{\log(100nd / \delta)}{d}}.
\end{align*}

\begin{align*}
    |T_j| \le \frac{40\eta_1\eta_0}{mn^2} \frac{\sqrt{nk}}{B\sqrt{d}} \log^{1.5}(300mnB/\delta) \left( 4\sqrt{\frac{\log(100nd / \delta)}{d}}  + \1(j = j')\right).
\end{align*}

Summing over $j$, by~\Cref{lem:order}, we have that the contribution is bounded by
\begin{align}
\label{eq:T}
    \sum_{j} |T_j| &\le \frac{160\eta_1 \eta_0}{mn^2} \log^2(300mnB/\delta) \frac{n\sqrt{nk}}{d\sqrt{Bd}}\le \frac{\eta_1 \eta_0}{2000mn^2}.
\end{align}

\item For $R$, we will directly put everything to an upper bound. We will discuss three cases,

\begin{itemize}
    \item $\not \exists i'', S[i''] = j$, by~\Cref{lem:signalcorrelation},
    \begin{align*}
        |R| &\le \frac{16k\eta_1 \log(100nd/\delta)}{Bnd} \sup \alpha \Delta + \frac{\eta_1}{B} \frac{1}{i'}  \alpha_{1, i', S[i'], b_1, b_2'} \Delta_{1,i',S[i'],b_1,b_2'} | \langle e_{j', b_1'}, e_{S[i],b_1} \rangle| \\
        &\le \frac{16k\eta_1 \log(100nd/\delta)}{Bnd}\frac{B}{2} \frac{\eta_0}{4mn} + \frac{\eta_1}{Bn}B \frac{\eta_0}{4mn} \sqrt{\frac{\log(100nd/\delta)}{d}} \\
        &=\frac{ \eta_0\eta_1 k\log(100nd/\delta)}{mn^2d} + \frac{\eta_0 \eta_1}{4mn^2} \sqrt{\frac{\log(100nd/\delta)}{d}}  \le \frac{\eta_0 \eta_1}{2000mn^2}.
    \end{align*}
    \item $\exists i'', S[i''] = j, i'' \neq i'$, we can get a similar bound,
    \begin{align*}
    |R| &\le \frac{\eta_0 \eta_1}{2000mn^2}.
    \end{align*}
    \item $S[i'] = j'$, in this case, we can show that,
    \begin{align*}
        |R - \frac{\eta_0 \eta_1}{8B^2mi'} \bar \alpha_{i, S[i'], b_1, b_2}^2 | \le \frac{\eta_0 \eta_1}{2000mn^2}.
    \end{align*}
\end{itemize}
This concludes that
\begin{align}
\label{eq:r}
     \left|R - \frac{\eta_0 \eta_1}{8B^2mi'} \bar \alpha_{i, S[i'], b_1, b_2}^2 \1( S[i'] = j') \right| \le \frac{\eta_0 \eta_1}{2000mn^2}.
\end{align}

\item $U_{j,p}$ can be bounded in the same way as $T_j$ and we have that,
\begin{align}
\label{eq:U}
    \sum_{j} \sum_{p} U_{j,p} \le \frac{\eta_1\eta_0}{1000mn^2}.
\end{align}

\item For $V_{p}, W_{j}, Y$, we will directly put everything to an upper bound similar to the bound of $R$, we have that
\begin{align}
\label{eq:V}
    \sum_{p=1}^{n + k} |V_{p}| &\le \frac{\eta_1 \eta_0}{1000mn^2}.\\
    \label{eq:W}
     \sum_{j=1}^{n + k} |W_{j}| &\le \frac{\eta_1 \eta_0}{1000mn^2} \\
     \label{eq:Y}
    |Y| &\le \frac{\eta_1 \eta_0}{1000mn^2}.
\end{align}
\end{enumerate}

Summing over~\Cref{eq:T,eq:r,eq:U,eq:V,eq:W,eq:Y}, we have that
\begin{align*}
    \left | \left \langle e_{j', b_1'}, A^{(2)} e_{i', b_2'} \right \rangle -  \frac{\eta_0 \eta_1}{8Bmi'} \bar \alpha_{i, S[i'], b_1, b_2}^2 \1( S[i'] = j') \right | \le \frac{\eta_1 \eta_0}{512mn^2}.
\end{align*}

Further 
\begin{align*}
    \frac{\eta_1}{8B^2 mi'} \bar \alpha_{i, S[i'], b_1, b_2}^2\ge \frac{\eta_0\eta_1}{128mn^2}.
\end{align*}
Hence, we can show that
\begin{enumerate}
    \item If $j' = S[i']$, then $\left \langle e_{j', b_1'}, A^{(2)} e_{i', b_2'} \right \rangle \ge \frac{\eta_0\eta_1}{256mn^2}$.
    \item If $j' \neq S[i']$, then $|\left \langle e_{j', b_1'}, A^{(2)} e_{i', b_2'} \right| \rangle \le \frac{\eta_0\eta_1}{512mn^2}$.
\end{enumerate}
The proof is complete.
\end{proof}

\begin{lemma}
\label{lem:onehotoutputattn}
Under the setting of~\Cref{lem:onehotattn}, the attention output is approximately one-hot after the second step, with
\begin{align*}
    \left | \attention[\qk^2](\embed(\bsequence))[i] - e_{S[i], \bsequence[S[i]]} \right| < 1 / n^{10}.
\end{align*}
\end{lemma}

\begin{proof}[Proof of Lemma \ref{lem:onehotoutputattn}]

This is a direct combination with~\Cref{lem:onehotattn,lem:onehotoutputattntech,lem:order}.
\end{proof}

\subsubsection{Third Step: Moving MLP in Linear Regime}

\begin{lemma}
\label{lem:gdformlast}
With probability $1 - 0.8 \delta$, the gradient of the FFN layer on the third step can be written as, 
\begin{align*}
    \left |\left \langle  \frac{ d \loss^{(2)} }{d W_{r, 1:d}} ,e_{i,b}\right \rangle \right |  \le& \frac{\sqrt{\log(300nB/\delta)}}{m \sqrt{B}}. \\
    \frac{ d \loss^{(2)} }{d W_{r, d + 1:2d}}  =& h_r \sum_{i = n + 1}^{n + k} \sum_{b_1 = 0}^1 \sum_{b_2 = 0}^1  \frac 1 i  \1( \nu_{r, i, b_2} > 0) (-1)^{b_1 + b_2}  e_{S[i], b_1} + O(\frac{1}{n^8}).
\end{align*}
Further the event in~\Cref{lem:onehotattn,lem:onehotoutputattn} hold.
\end{lemma}

\begin{proof}
The proof is similar to~\Cref{lem:gdonattentionform,lem:gdonembeddingform}, switching the original output with the near one-hot output of the attention layer.
\end{proof}

\begin{lemma}
\label{lem:attnalmostunchanged}

With probability $1 - 0.9 \delta$, the attention output is almost unchanged after the third step, with
\begin{align*}
    \left | \attention[\qk^3](\embed(\bsequence))[i] - e_{S[i], \bsequence[S[i]]} \right| < 1 / n^9.
\end{align*}
Further, events in  \Cref{lem:gdformlast} hold.
\end{lemma}

\begin{proof}

By~\Cref{lem:gradientonA,lem:onehotoutputattn,lem:order}, $\qk^{(3)} - \qk^{(2)}$ is of order $1/n^8$, this implies that the attention weight and output is almost unchanged. 
\end{proof}

\begin{lemma}
\label{lem:finalFFNoutputattn}
With probability $1 - 0.91 \delta$, 
\begin{align*}
    \left|\left \langle  W_{r, d+1:2d}^{(2)} , \attention(\qk^3)[\embed(\bsequence)][i] \right \rangle +  \mathrm{sign}(\nu_{r, i, \bsequence[S[i]]}h_r) \frac{2\epsilon}{3} \1( \nu_{r, i, 0} \nu_{r, i, 1}<0) \right| < \frac{\epsilon}{300}.
\end{align*}
Further, events in  \Cref{lem:attnalmostunchanged} hold.
\end{lemma}

\begin{proof}
We will first calculate the projection of the gradient on $e_{S[i'],b'}$. When the event in~\Cref{lem:gdformlast} happens,
\begin{align*}
    \left \langle \frac{ d \loss^{(2)} }{d W_{r, d+1:2d}}, e_{S[i'],b'} \right \rangle &= h_r \sum_{i = n + 1}^{n + k} \sum_{b_1 = 0}^1 \sum_{b_2 = 0}^1    \1( \nu_{r, i, b_2} > 0) (-1)^{b_1 + b_2} \langle e_{S[i'],b'}, e_{S[i], b_1} \rangle + O(\frac{1}{n^8}).
\end{align*}

By~\Cref{lem:linearemb,lem:order}, with probability $1 - 0.01\delta$,
\begin{align*}
    \left|\left \langle \frac{ d \loss^{(2)} }{d W_{r, d+1:2d}}, e_{S[i'],b'} \right \rangle - \mathrm{sign}(\nu_{r, i', b'})h_r \1( \nu_{r, i', 0} \nu_{r, i', 1}<0) \right| \le \frac{1}{200m}
\end{align*}

By~\Cref{lem:gdform,lem:order}, we have that,
\begin{align*}
    \left|\left \langle  W_{r, d+1:2d}^{(2)} , e_{S[i'],b'} \right \rangle +  \mathrm{sign}(\nu_{r, i', b'})h_r \eta_2 \1( \nu_{r, i', 0} \nu_{r, i', 1}<0) \right| < \frac{(\eta_0 + \eta_1)}{20mn}  + \frac{\eta_2}{200m} < \frac{\eta_2}{100m}.
\end{align*}

Combining with~\Cref{lem:attnalmostunchanged} and $\eta_2 = \frac{4\epsilon m}{3}$, we have the result.
\end{proof}

\begin{lemma}
\label{lem:linearFFNoutput2}
With probability $1 - \delta$, for all $\bsequence \in \{0, 1\}^{n + k}$, $i \in [n + 1, n + k]$, we have that
\begin{align*}
    &\left|\transformer(\embed(\bsequence))[i] - \frac{\epsilon (-1)^{\bsequence[i + 1] + 1} }{3} \right| < \frac{\epsilon}{3}.
\end{align*}
\end{lemma}

\begin{proof}
By~\Cref{lem:gdformlast}, the following holds,
\begin{align*}
        \left \langle  W_{r, 1:d}^{(3)} ,\embed(\bsequence)[i] \right \rangle  \mathrm{sign}(\nu_{r, i, \bsequence[i]})  \in [5\epsilon/6, 7\epsilon/6].
\end{align*}

Combining with~\Cref{lem:finalFFNoutputattn}, 
\begin{align*}
    \left|\left \langle  W_{r, d+1:2d}^{(2)} , \attention(\qk^3)[\embed(\bsequence)][i]  \right \rangle +  \mathrm{sign}(\nu_{r, i, \bsequence[S[i]]}h_r) \frac{2\epsilon}{3} \1( \nu_{r, i, 0} \nu_{r, i, 1}<0) \right| < \frac{\epsilon}{300}.
\end{align*}

Hence, we still have that 
\begin{align*}
    |\left \langle  W_{r, d+1:2d}^{(3)} , \attention(\qk^3)[\embed(\bsequence)][i] \right \rangle| < \frac{5\epsilon}{6}.
\end{align*}
which implies that,
\begin{align*}
\1(\left \langle W_{r, 1:d}^{(3)} , \embed(\bsequence)[i] \right \rangle > 0) = \1(\nu_{r, i, \bsequence[i]} > 0).
\end{align*}

Consider the output contribution of the attention part,
\begin{align*}
    \Big|& \sum_{r = 1}^{2m} h_r \left \langle W_{r, d+1:2d}^{(3)} , \attention(\qk^3)[\embed(\bsequence)][i] \right \rangle \1(\nu_{r, i, \bsequence[i]} > 0) \\
    &+ \sum_{r = 1}^{2m} \frac{2h_r\epsilon}{3} \mathrm{sign}(\nu_{r, i, \bsequence[S[i]]}h_r)  \1( \nu_{r, i, 0} \nu_{r, i, 1}<0, \nu_{r, i, \bsequence[i]} > 0) \Big| < \frac{\epsilon}{300}.
\end{align*}

The later term with $1 - \delta/100$ satisfies,
\begin{align*}
    &\sum_{r = 1}^{2m} \frac{2h_r\epsilon}{3} \mathrm{sign}(\nu_{r, i, \bsequence[S[i]]}h_r)  \1( \nu_{r, i, 0} \nu_{r, i, 1}<0, \nu_{r, i, \bsequence[i]} > 0) \\
    =& \sum_{r = 1}^{2m} \frac{\epsilon}{3m} \mathrm{sign}(\nu_{r, i, \bsequence[S[i]]}) \1( \nu_{r, i, 0} \nu_{r, i, 1}<0, \nu_{r, i, \bsequence[i]} > 0) \\
    =& \sum_{r = 1}^{2m} \frac{\epsilon}{3m} (-1)^{\bsequence[S[i]] + \bsequence[i]}\1( \nu_{r, i, 0} \nu_{r, i, 1}<0) \\
    =& \frac{\epsilon}{3} (-1)^{\bsequence[S[i]] + \bsequence[i]} + O(\frac{\epsilon\log(100n/\delta)}{\sqrt{m}})
\end{align*}

This shows that,
\begin{align*}
    \left|\sum_{r = 1}^{2m} h_r \left \langle W_{r, d+1:2d}^{(3)} , \attention(\qk^3)[\embed(\bsequence)][i] \right \rangle \1(\nu_{r, i, \bsequence[i]} > 0) - \frac{\epsilon}{3} (-1)^{1+ \bsequence[S[i]] + \bsequence[i]} \right| \le \frac{\epsilon}{150}.
\end{align*}

On the other hand, we have that
\begin{align*}
    &\left|\sum_{r = 1}^{2m} h_r \left \langle W_{r, 1:d}^{(3)} , \embed(\bsequence)[i] \right \rangle \1(\nu_{r, i, \bsequence[i]} > 0) \right| \\
    \le& \left|\sum_{r = 1}^{2m} h_r \left \langle W_{r, 1:d}^{(0)} , \embed(\bsequence)[i] \right \rangle \1(\nu_{r, i, \bsequence[i]} > 0)\right| + \left|\sum_{r = 1}^{2m} h_r \left \langle W_{r, 1:d}^{(3)} -  W_{r, 1:d}^{(0)} , \embed(\bsequence)[i] \right \rangle \1(\nu_{r, i, \bsequence[i]} > 0) \right| \\
    \le& \frac{\epsilon}{4}.
\end{align*}

The first term is bounded due to standard concentration inequality over the axis of $r$. The second term is bounded by~\Cref{lem:gdform,lem:gdformlast}. Combining the terms, we have that
\begin{align*}
    &\left|\transformer(\embed(\bsequence))[i] - \frac{\epsilon (-1)^{\bsequence[i] + \bsequence[S[i]] + 1} }{3} \right| =  \left|\transformer(\embed(\bsequence))[i] - \frac{\epsilon (-1)^{\bsequence[i + 1] + 1} }{3} \right| < \frac{\epsilon}{3}.
\end{align*}
This concludes that the model is able to predict the correct output.
\end{proof}

\subsection{Final Proof}
\label{sec:proofofthm3}

\begin{lemma}
\label{lem:extensionhinge}
The results in~\Cref{lem:linearFFNoutput2,lem:attnsignal} can be extended to hinge loss $ \ell(\hat y, y) = \max \{(-1)^y \hat y + 1, 0\} $ with the same probability.
\end{lemma}

\begin{proof}
Our~\Cref{lem:gdform,lem:gdformlast} shows that the output of the FFN layer is bounded by $\epsilon/3$ throughout training. The hinge loss is linear when the output is in $[-1, 1]$. Hence, the results can be directly extended to hinge loss.
\end{proof}

\begin{proof}[Proof of~\Cref{thm:main}]
    The original theorem is a combination of~\Cref{lem:extensionhinge,lem:linearFFNoutput2,lem:attnsignal}.
\end{proof}

\subsection{Technical Lemma}

\begin{lemma}[Theorem 4.4.5 of~\cite{Vershynin2018HDP}]
\label{lem:matrixbound}
There exists universal constant $C$,let $A$ be an $m \times n$ random matrix whose entries $A_{ij}$ are independent, mean zero, sub-Gaussian random variables. Then, for any $t > 0$, we have
    \[
    \|A\| \leq CK \left(\sqrt{m} + \sqrt{n} + t\right)
    \]
    with probability at least $1 - 2\exp(-t^2)$. Here $K = \max_{i,j} \|A_{ij}\|_{\psi_2}$ is the maximum sub-Gaussian norm of $A_{ij}$.
\end{lemma}

\begin{lemma}[Theorem 4.6.1 of~\cite{Vershynin2018HDP}]
\label{lem:mineigen}
    Let $A$ be an $m \times n$ matrix whose rows $A_i$ are independent, mean zero, sub-Gaussian isotropic random vectors in $\mathbb{R}^n$. Then, for any $t \geq 0$, we have
    \[
    \sqrt{m} - CK^2 (\sqrt{n} + t) \leq s_n(A) \leq s_1(A) \leq \sqrt{m} + CK^2 (\sqrt{n} + t)
    \]
    with probability at least $1 - 2\exp(-t^2)$. Here, $K = \max_i \|A_i\|_{\psi_2}$.

    Furthermore, a slightly stronger conclusion holds:
    \[
    \left\| \frac{1}{m} A^\top A - I_n \right\| \leq K^2 \max(\delta, \delta^2),
    \]
    where $\delta = C \left( \sqrt{\frac{n}{m}} + \frac{t}{\sqrt{m}} \right).$
\end{lemma}

\begin{lemma}
\label{lem:linearemb}

Define $v = \sum_{i = n+1}^{n + k} \sum_{b = 0}^1 \lambda_{i, b} e_{i,b}$, then with probability $1 - \delta$,
\[
\left| \langle e_{i',b'}, v \rangle - \lambda_{i',b'} \right| \le \sqrt{\frac{2 \log\left( \tfrac{2}{\delta} \right)}{d} \sum_{(i,b) \ne (i',b')} \lambda_{i,b}^2}.
\]
\end{lemma}

\begin{proof}
We aim to bound \( |\langle e_{i',b'}, v \rangle| \), where
\[
v = \sum_{i = n + 1}^{n + k} \sum_{b = 0}^1 \lambda_{i,b} e_{i,b}.
\]
Note that each \( e_{i,b} \in \mathbb{R}^d \) has entries that are independent random variables from \( \left\{ -\tfrac{1}{\sqrt{d}}, \tfrac{1}{\sqrt{d}} \right\} \).

First, observe that
\[
\langle e_{i',b'}, v \rangle = \sum_{(i,b)} \lambda_{i,b} \langle e_{i',b'}, e_{i,b} \rangle = \lambda_{i',b'} + \sum_{(i,b) \ne (i',b')} \lambda_{i,b} \langle e_{i',b'}, e_{i,b} \rangle.
\]
Since \( \langle e_{i',b'}, e_{i',b'} \rangle = 1 \) and for \( (i,b) \ne (i',b') \), the inner products \( \langle e_{i',b'}, e_{i,b} \rangle \) are sums of independent random variables with mean zero.

Define
\[
X_{i,b} = \lambda_{i,b} \langle e_{i',b'}, e_{i,b} \rangle, \quad \text{for } (i,b) \ne (i',b').
\]
Each \( X_{i,b} \) is a sum of \( d \) independent random variables bounded in \( \left[ -\tfrac{|\lambda_{i,b}|}{d}, \tfrac{|\lambda_{i,b}|}{d} \right] \), because each component \( e_{i',b',j} e_{i,b,j} \) is \( \pm \tfrac{1}{d} \).

By the Azuma-Hoeffding inequality, for any \( t > 0 \),
\[
\Pr\left( \left| \sum_{(i,b) \ne (i',b')} X_{i,b} \right| \ge t \right) \le 2 \exp\left( -\frac{2 dt^2}{\sum_{(i,b) \ne (i',b')} 4 \lambda_{i,b}^2} \right) = 2 \exp\left( -\frac{dt^2}{2 \sum_{(i,b) \ne (i',b')} \lambda_{i,b}^2} \right).
\]
Setting the right-hand side equal to \( \delta \) and solving for \( t \), we get
\[
t = \sqrt{\frac{2 \log\left( \tfrac{2}{\delta} \right)}{d} \sum_{(i,b) \ne (i',b')} \lambda_{i,b}^2}.
\]
Therefore, with probability at least \( 1 - \delta \),
\[
\left| \langle e_{i',b'}, v \rangle - \lambda_{i',b'} \right| \le \sqrt{\frac{2 \log\left( \tfrac{2}{\delta} \right)}{d} \sum_{(i,b) \ne (i',b')} \lambda_{i,b}^2}.
\]
The proof is then complete.
\end{proof}

\begin{lemma}
\label{lem:arrow}
With probability $1 - \delta$, it holds that
\begin{align*}
    \forall i, i' \in [n + k], b, b' \in \{0, 1\}, \| e_{i,b}^T e_{i',b'} - \1(i = i' \& b = b') \|_2 \le \sqrt{2\frac{\log(8nd/\delta)}{d}}.
\end{align*}
\end{lemma}
\begin{proof}
    This is a direct consequence combining~\Cref{lem:linearemb} and union bound.
\end{proof}

\begin{lemma}
\label{lem:onehotoutputattntech}
If for constant $C$, the attention score before softmax has the following property:
\begin{itemize}
    \item for each position $i$, the target position $j = S[i]$ satisfies:
\[
\left\langle e_{j, b_1'}, A^{(2)} e_{i, b_2'} \right\rangle \geq 2C \log n .
\]
\item For all other positions $j' \neq j$, the attention weights satisfy:
\[
\left| \left\langle e_{j', b_1'}, A^{(2)} e_{i, b_2'} \right\rangle \right| \leq C \log n.
\]
\end{itemize}
The attention output is approximately one-hot, with
\begin{align*}
    \left | \attention[\qk^2](\embed(\bsequence))[i] - e_{S[i], \bsequence[S[i]]} \right| < 4 / n^{C - 1}.
\end{align*}
\end{lemma}

\begin{proof}[Proof of Lemma \ref{lem:onehotoutputattn}]

The attention output for position $i$ is given by:
\[
\attention[\qk^2](\embed(\bsequence))[i] = \hiddenstate \cdot \softmax\left(\hiddenstate^\top \qk \hiddenstate\right)[i],
\]
where the softmax is applied column-wise.

Given the condition from Lemma \ref{lem:order}, we set 
\[
\Delta = 2C \log n
\]

This implies that:
\[
e^{-\Delta / 2} < e^{-C \log n} = \frac{1}{n^{C}}.
\]

Define $z_{j'} = \left\langle e_{j', b_1'}, A^{(2)} e_{i, b_2'} \right\rangle$.

The softmax for the target position $j = S[i]$ is:
\[
\softmax(z)_j = \frac{e^{z_j}}{e^{z_j} + \sum_{j' \neq j} e^{z_{j'}}}.
\]

Given that $z_j \geq \Delta$ and $|z_{j'}| \leq \frac{\Delta}{2}$ (since $\frac{\eta_0 \eta_1}{512 m n^2} = \frac{\Delta}{2}$), we have:
\[
z_j - z_{j'} \geq \Delta - \frac{\Delta}{2} = \frac{\Delta}{2}.
\]

Thus:
\[
\softmax(z)_j \geq \frac{e^{\Delta}}{e^{\Delta} + (T-1) e^{\Delta/2}} = \frac{1}{1 + (T-1) e^{-\Delta/2}}.
\]

Therefore with $T \le 2n$:
\[
\softmax(z)_j \geq \frac{1}{1 + (T-1) \cdot \frac{1}{n^{C}}} \geq 1 - \frac{2}{n^{C - 1}}.
\]

Therefore
\[
\sum_{j' \neq j} \softmax(z)_{j'} \leq  \frac{2}{n^{C - 1}}.
\]

The attention output for position $i$ is:
\[
\attention[\qk^2](\embed(\bsequence))[i] = \sum_{j} e_{j,\bsequence[j]} \softmax(z)_j.
\]

Substituting the bounds:
\begin{align*}
\left| \attention[\qk^2](\embed(\bsequence))[i] - e_{S[i],\bsequence[S[i]]} \right| &= \left| \sum_{j' \neq S[i]} e_{j',\bsequence[j']} \softmax(z)_{j'} \right| + \left|  e_{S[i],\bsequence[S[i]]} (\softmax(z)_{S[i]} - 1) \right|\\
&\leq \sum_{j' \neq S[i]}  \softmax(z)_{j'} + \left|  (\softmax(z)_{S[i]} - 1) \right| \le \frac{4}{n^{C - 1}}    
\end{align*}
We conclude our proof.
\end{proof}

\newpage

\section{Additional Experiment Results}

In this section, we provide details of the experiment setup and present additional results. All training was conducted using PyTorch \cite{paszke2019pytorch} on NVIDIA RTX A10 GPUs.

\subsection{Experiment Details}\label{sec: experiment detail}

The transformer architecture adopted in the experiment section is based on the GPT-2 model \citep{radford2019language} with a hidden size of $720$, an intermediate size of $3072$, and trainable position embeddings. For all experiments, we use \texttt{Adam} \citep{kingma2014adam} optimizer with random initialization, using hyperparameters $\beta_1 = 0.9, \beta_2= 0.999$, a weight decay of $0$ and a linear decay learning rate schedule. The batch size is set to $512$ throughout. The validation set contains $2048$ samples which are nonintersecting with the training data. In all experiments regarding sample complexity, a test set of size $2048$ which is non-intersected with the training data is used. 
% \paragraph{Details in \Cref{fig: motivating example}}

% % Same as in Section \Cref{sec: exp multi-layer}, the sample complexity is refers to the minimum number of samples seen by the model before reaching an evaluation accuracy of $1$.  

% To compute the normalized attention entropy on CoT data, we take the 

\paragraph{Details of the experiments in \Cref{sec: GSM8K}}
This section examines the normalized attention entropy of Qwen2-7B \citep{yang2024qwen2} and Qwen2-Math-7B models \citep{qwen2024} on the GSM8K  dataset \citep{cobbe2021training} with and without CoT prompting respectively. Both models consist of $28$ layers, each with $28$ attention heads. The normalized attention entropy is computed as the average across the GSM8K test set. In \Cref{fig: GSM8K data}, the entropy values of each attention head in a layer are sorted separately under the three setups. While \Cref{fig: GSM8K data} presents the normalized attention entropy for specific attention heads in the first, last, and two intermediate layers, \Cref{fig: GSM8K all layer} shows the entropy across all layers.

For \texttt{With CoT} data, the input is concatenated with the ground truth answer from the GSM8K dataset. For  \texttt{With CoT} data, we extract the final answer from the ground truth, and concatenate the input with the string ``The answer is [Final Answer].''. 

\begin{figure}
    \centering
    \includegraphics[width=\linewidth]{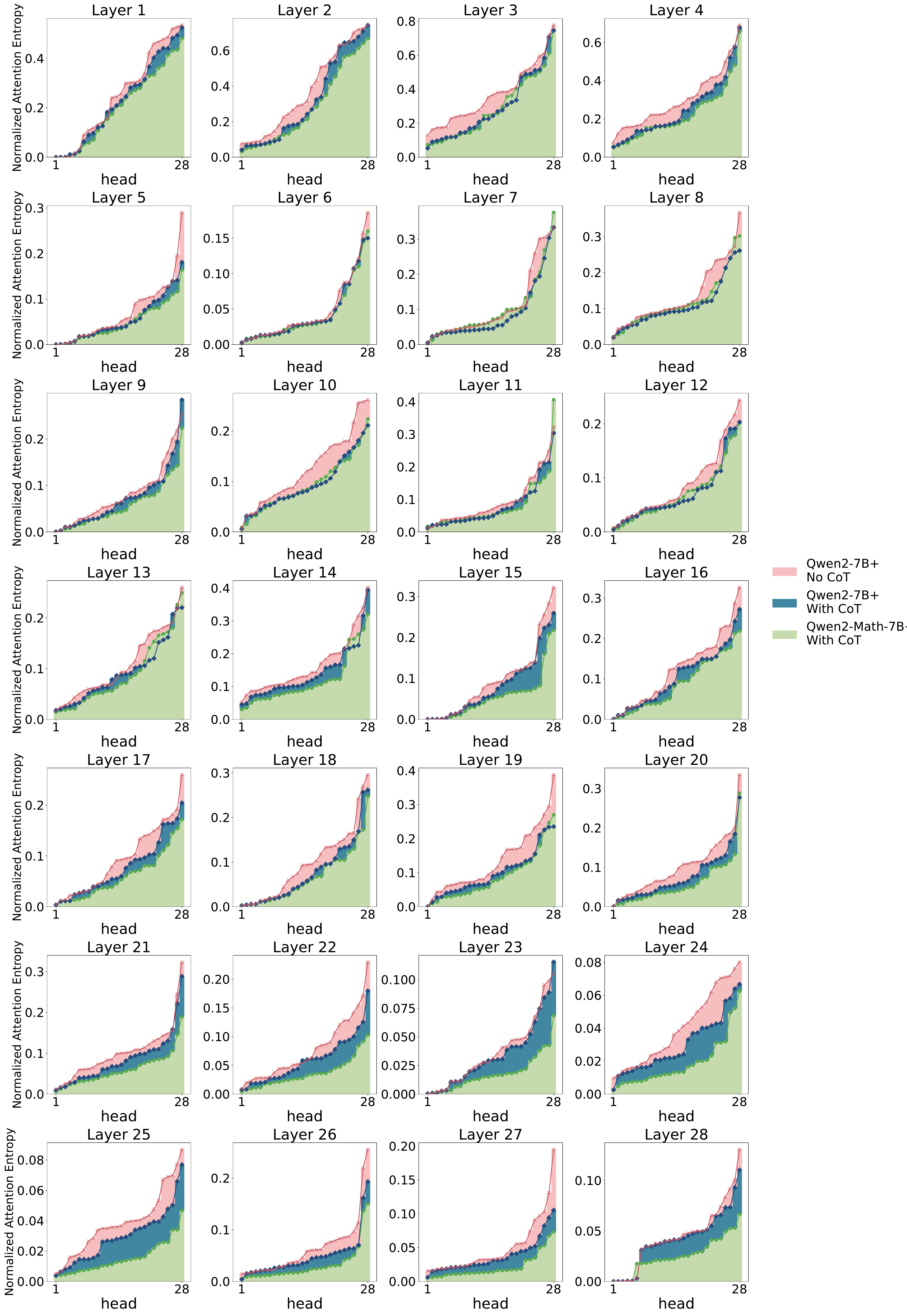}
    \caption{The normalized attention entropy of the pre-trained Qwen2-7B and math-specialized Qwen2-Math-7B models on the GSM8K dataset with and without CoT prompting (Section \ref{sec: GSM8K}). Each bar represents the entropy of an attention head.}
    \label{fig: GSM8K all layer}
\end{figure}

\subsection{Additional Results}

In \Cref{fig: sample complexity with CoT and attention} (Right), we present the attention pattern of a single-layer, single-head transformer trained on the $(n=20, k=40)$ parity problem with CoT data. In \Cref{fig: attention 2layers_2heads}, we show the attention patterns of a multi-layer, multi-head transformer trained on the same problem. We observe that in the first layer, the attention pattern is sparse and interpretable, with each secret variable attended to by at least one attention head. In contrast, the second layer exhibits an almost uniform attention pattern. A possible explanation is that the first layer captures sufficient information, which is then transferred to subsequent layers via the residual connections.

\begin{figure}
    \centering
    \includegraphics[width=\linewidth]{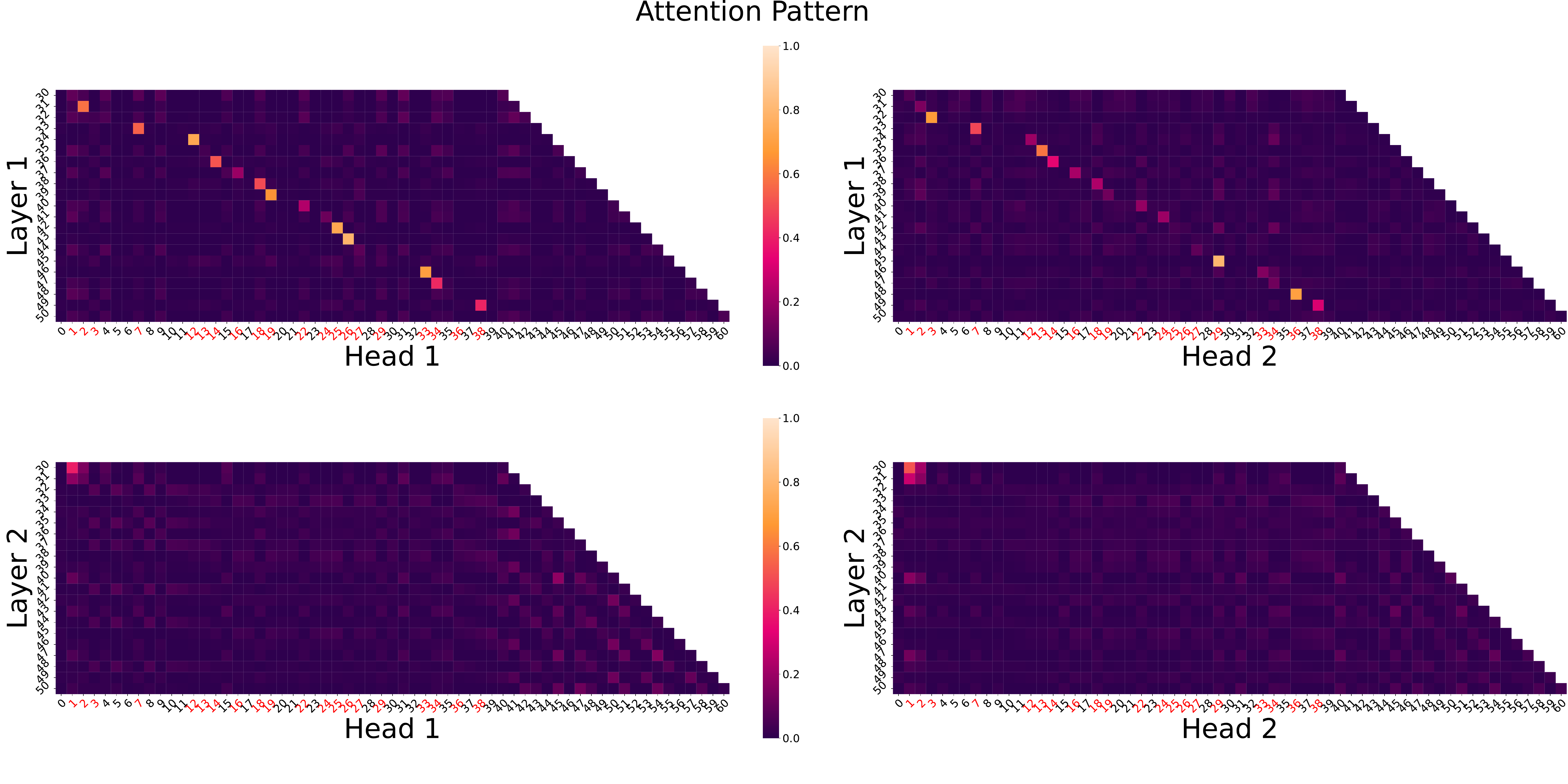}
    \caption{The attention pattern learned by a $2$-layer $2$-head transformer on $(n=20,k=6)$ parity problem with CoT.}
    \label{fig: attention 2layers_2heads}
\end{figure}

In \Cref{sec:multipass}, we observe that training with repeated data can help Transformers learn the parity function, but it still requires significantly more computation compared to trained on CoT data (\Cref{fig: multi-pass helps}). To further substantiate this observation, we explore the training of models with and without CoT on the $(n=20, k=12)$ parity problem. Compared with the $(n=20,k=6)$ parity problem considered in \Cref{fig: multi-pass helps}, this problem is harder for the models to learn as the number of secret variables $k$ is larger. We examine various configurations: the number of layers and heads ranges from $1, 2, 3, 4, 6$, and $8$; the learning rate varies from $6 \times 10^{-6}$, $8 \times 10^{-8}$, to $1 \times 10^{-4}$; and the training dataset size varies from $10^4$, $10^5$ to $10^6$, with corresponding epochs ranging from $1000$, $100$ to $10$.  Across all configurations we examined, training without CoT fails to achieve non-trivial accuracy (\Cref{fig: mutli-pass fail}). In contrast, models trained with CoT achieve perfect evaluation accuracy when trained on a dataset with $10,000$ samples for only $6$ epochs, or with approximately $60,000$ fresh samples in one-pass training setting.

\begin{figure}
    \centering
    \includegraphics[width=\linewidth]{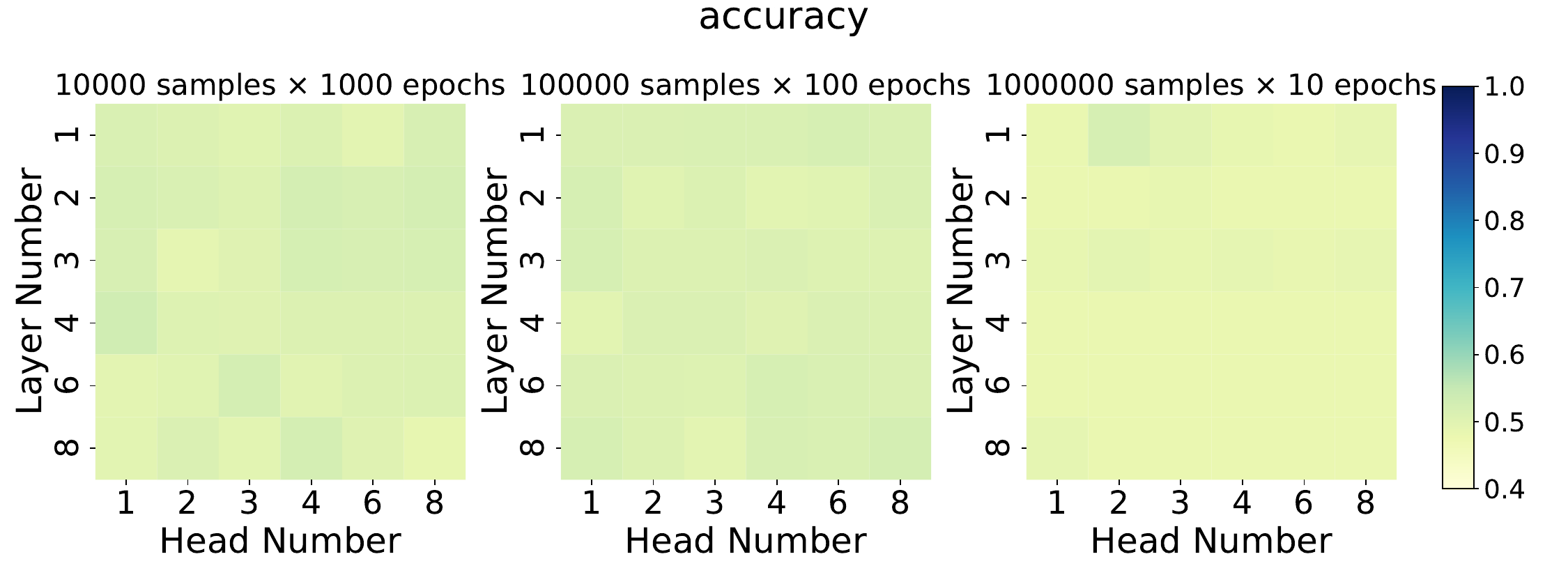}
    \caption{Training with CoT on $(n=20,k=12)$ parity problem fails to achieve non-trivial accuracy under different configurations.}
    \label{fig: mutli-pass fail}
\end{figure}

In \Cref{fig: repeated-data better}, a $4$-layer $4$-head transformer achieves perfect evaluation accuracy on $(n=20,k=6)$ problem when trained on a small dataset of $50000$ samples, but fails to achieve non-trivial accuracy when trained on a larger dataset. Furthermore, successful learning coincides with a significant decrease in attention entropy, indicating the development of sparse attention, while entropy remains high when trained on a larger dataset. In \Cref{fig: mutli-pass helps more arch}, we present more results across different architectures ($1$-layer $1$-head, $2$-layer $3$-head and $4$-layer $4$-head transformers) with dataset size of $5000,10000,50000,100000,1000000$, and observe the same pattern. 

\begin{figure}
    \centering
    \includegraphics[width=\linewidth]{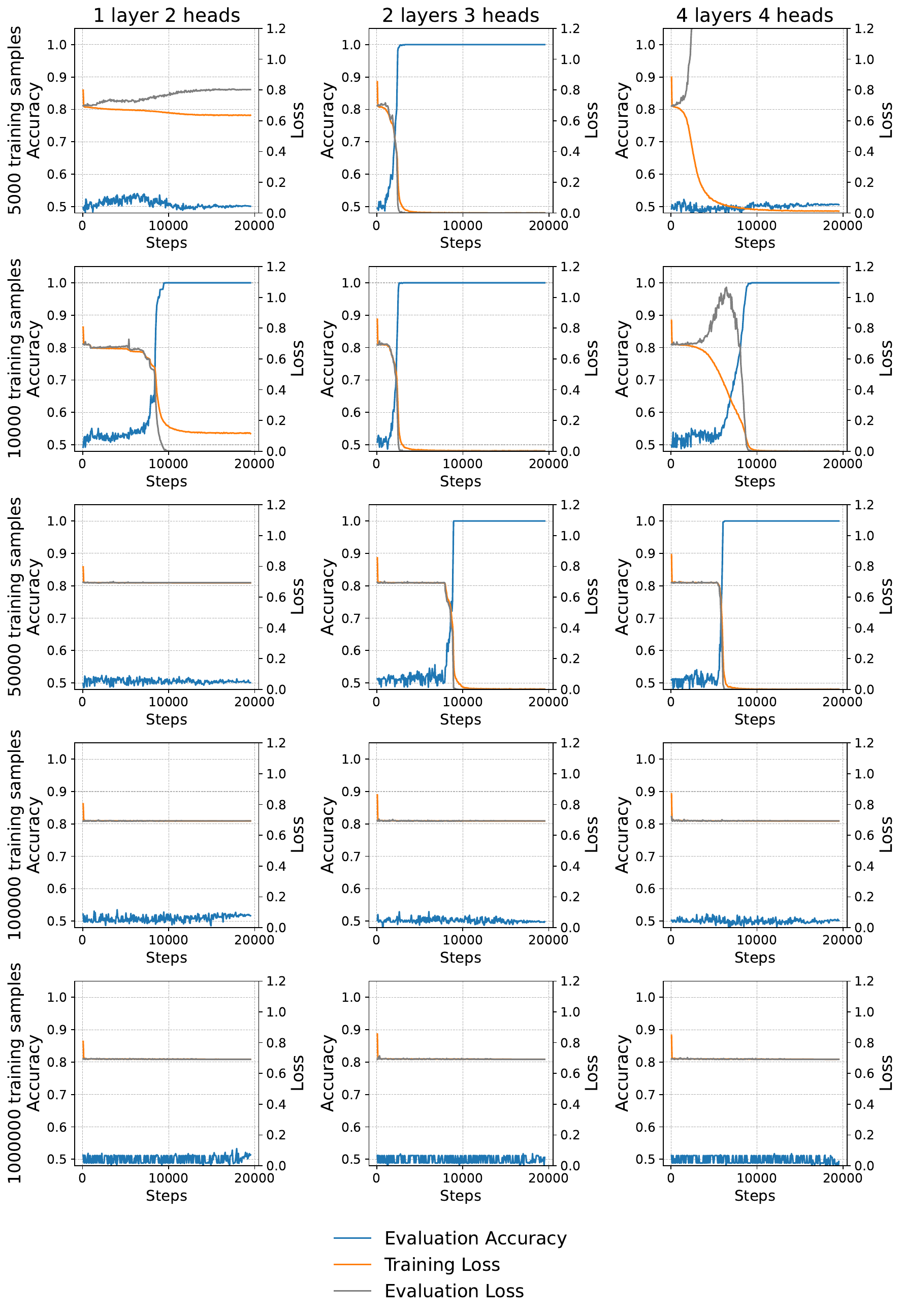}
    \caption{Transformer trained on the $(n=20, k=6)$ parity problem without CoT with different sizes of training dataset and a fixed number of iterations.}
    \label{fig: mutli-pass helps more arch}
\end{figure}

\begin{figure}[t]
    % 第一行 - 第一张图片
    \begin{minipage}[t]{0.67\linewidth}  % 设置为整行宽度
        \centering
        \begin{subfigure}[t]{\linewidth}
            \centering
            \includegraphics[height=8.3cm, keepaspectratio]{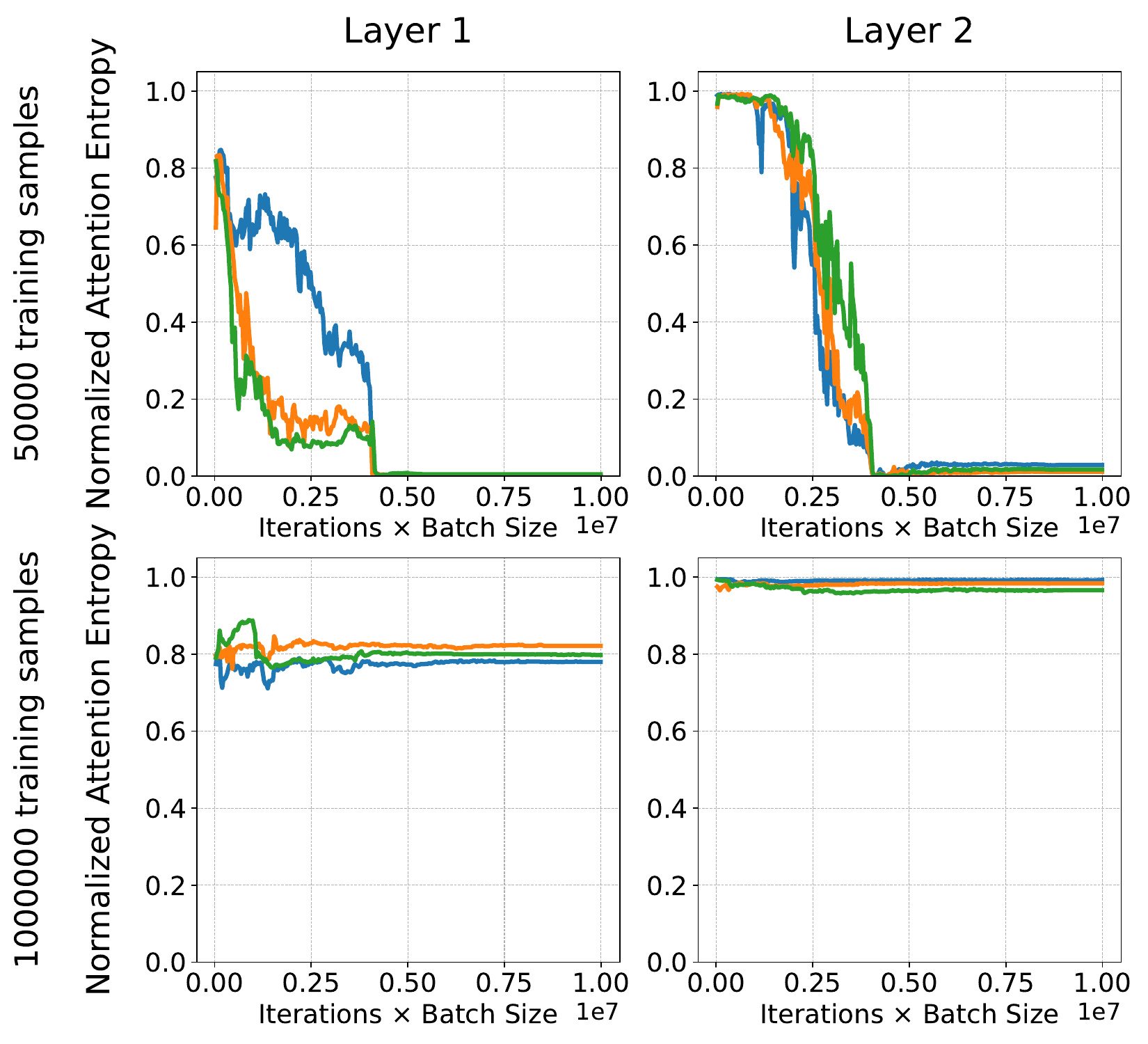}
            \caption{$2$-layer $3$-head transformer}
        \end{subfigure}
    \end{minipage}
    % 第一行 - 第一张图片
    \begin{minipage}[t]{0.33\linewidth}  % 设置为整行宽度
        \centering
        \begin{subfigure}[t]{\linewidth}
            \centering
            \includegraphics[height=8.3cm, keepaspectratio]{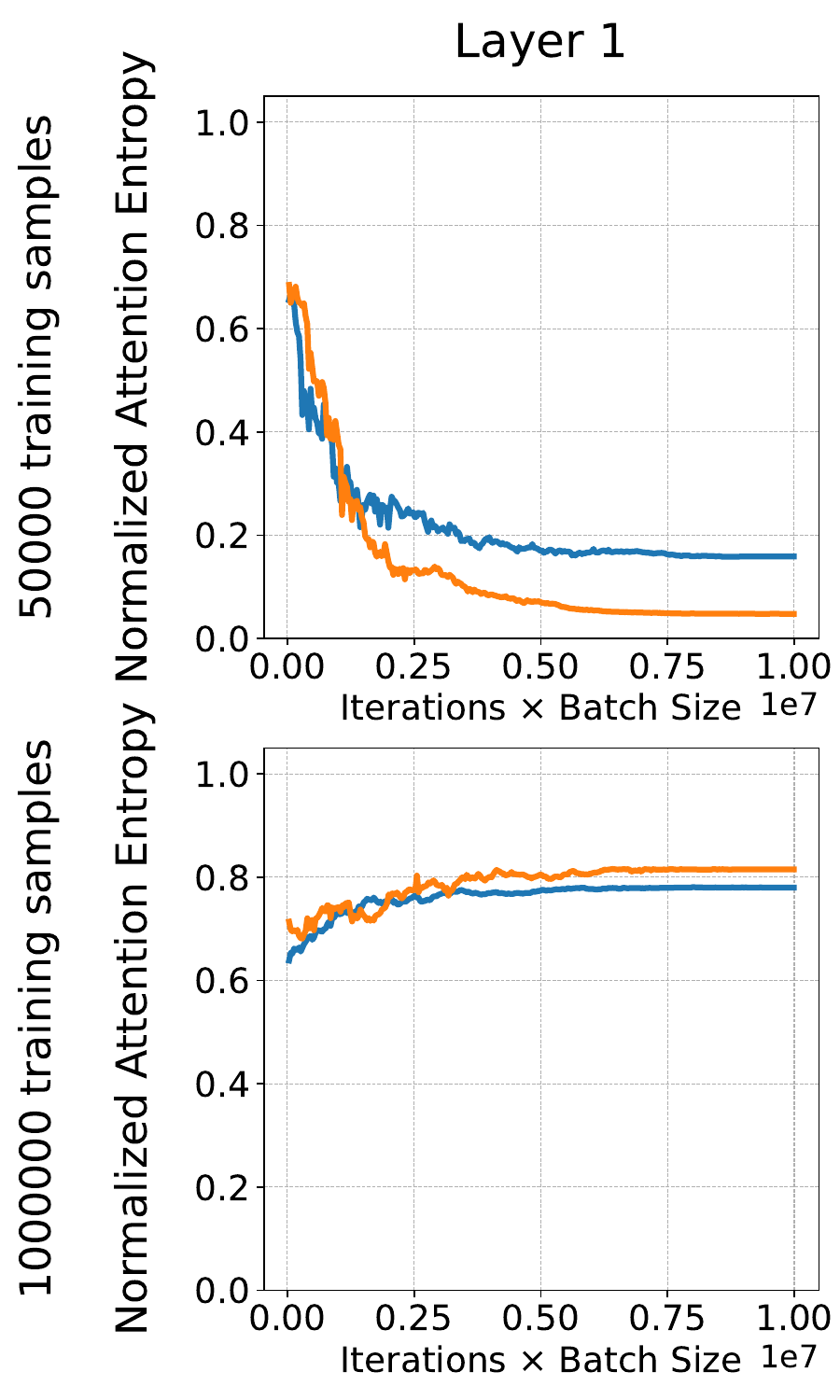}
            \caption{$1$-layer $2$-head transformer}
        \end{subfigure}
    \end{minipage}

    % 添加整体的图标题
    % \captionsetup{justification=centering}  % 设置大标题的样式
    \caption{Normalized attention entropy curve of transformers training on the $(n=20, k=6)$ parity problem without CoT. Each line in the graph represents the attention entropy for a head of a certain layer.}
    \label{fig: mutli-pass helps more arch entropy}
\end{figure}

\end{document}